\documentclass{article}

\PassOptionsToPackage{numbers, compress}{natbib}

\usepackage[preprint]{neurips_2021}

\usepackage{wrapfig}
\usepackage[utf8]{inputenc} 
\usepackage[T1]{fontenc}    
\usepackage{hyperref}       
\hypersetup{hidelinks}
\hypersetup{pdflang={en},pdfdisplaydoctitle}
\usepackage{url}            
\usepackage{booktabs}       
\usepackage{amsfonts}       
\usepackage{nicefrac}       
\usepackage{microtype}      
\usepackage{xcolor}         
\usepackage{multirow}

\usepackage{amsmath}
\usepackage{amsthm}
\usepackage{amssymb} 

\usepackage{mdframed} 
\usepackage{thmtools}
\newcommand{\myNum}[1]{(\emph{#1})}

\newcommand{\smartparagraph}[1]{\vspace{2pt} \noindent {\bf #1}}

\usepackage{graphicx}

\newcommand{\R}{\mathbb{R}} 
\newcommand{\N}{\mathbb{N}} 
\newcommand{\mbE}{\mathbb{E}} 
\newcommand{\fC}{\mathfrak{C}} 
\newcommand{\cA}{{\cal A}}

\newcommand{\cC}{{\cal C}}
\newcommand{\cD}{{\cal D}}

\newcommand{\cO}{{\cal O}}

\newcommand{\mS}{{\bf S}}

\usepackage[colorinlistoftodos,bordercolor=orange,backgroundcolor=orange!20,linecolor=orange,textsize=scriptsize]{todonotes}

\newcommand{\eqdef}{:=}

\newcommand{\dotprod}[1]{\left< #1\right>} 
\newcommand{\norm}[1]{\lVert#1\rVert}      

\usepackage{footnote}
\makesavenoteenv{tabular}
\makesavenoteenv{table}

 \newtheorem*{lemma*}{Lemma}
 \newtheorem*{proposition*}{Proposition}
 \newtheorem*{theorem*}{Theorem}
 \newtheorem*{corollary*}{Corollary}
 \newtheorem*{remark*}{Remark}

\newtheorem{assumption}{Assumption}
\newtheorem{lemma}{Lemma}
\newtheorem{theorem}{Theorem}

\newtheorem{corollary}{Corollary}

\theoremstyle{definition}
\newtheorem{definition}[theorem]{Definition}
\theoremstyle{remark}
\newtheorem{remark}{Remark} 

\usepackage{enumerate}
\usepackage[ruled,noend,vlined]{algorithm2e}

\usepackage{mathtools, nccmath}
\usepackage{caption}
\usepackage{paralist}

\usepackage[skip=0cm,list=true]{subcaption}

\makeatletter
\renewcommand{\@notice}{}
\makeatother

\SetKwInput{kwParameters}{Parameters}
\SetKwInput{kwInput}{Input}

\title{Rethinking gradient sparsification\\ as total error minimization}

\author{
Atal Narayan Sahu, Aritra Dutta, Ahmed M. Abdelmoniem, \\ \textbf{Trambak Banerjee$^\dagger$, Marco Canini, Panos Kalnis}  \\
  KAUST, $^\dagger$University of Kansas
}

\begin{document}

\maketitle

\begin{abstract}

Gradient compression is a widely-established remedy to tackle the communication bottleneck in distributed training of large deep neural networks (DNNs). Under the error-feedback framework, Top-$k$ sparsification, sometimes with $k$ as little as $0.1\%$ of the gradient size, enables training to the same model quality as the uncompressed case for a similar iteration count. From the optimization perspective, we find that Top-$k$ is the communication-optimal sparsifier given a per-iteration $k$ element budget.
We argue that to further the benefits of gradient sparsification, especially for DNNs, a different perspective is necessary --- one that moves from per-iteration optimality to consider optimality for the entire training.

We identify that the \emph{total error} --- the sum of the compression errors for all iterations --- encapsulates sparsification throughout training. Then, we propose a communication complexity model that minimizes the total error under a communication budget for the entire training. We find that the \emph{hard-threshold sparsifier}, a variant of the Top-$k$ sparsifier with $k$ determined by a constant hard-threshold, is the optimal sparsifier for this model. Motivated by this, we provide convex and non-convex convergence analyses for the hard-threshold sparsifier with error-feedback. Unlike with Top-$k$ sparsifier, we show that hard-threshold has the same asymptotic convergence and linear speedup property as SGD in the convex case and has no impact on the data-heterogeneity in the non-convex case. Our diverse experiments on various DNNs and a logistic regression model demonstrated that the hard-threshold sparsifier is more communication-efficient than Top-$k$.

\end{abstract}

\section{Introduction}

With the emergence of huge DNNs consisting of hundreds of millions to billions of parameters~\cite{megatronml, gpt3}, distributed data-parallel training \cite{parallel_SGD} is an increasingly important workload. As the training process typically spans several compute nodes (or workers) that periodically exchange the local gradient vectors at each iteration of the optimizer (e.g., SGD), communication among nodes remains in many cases the main performance bottleneck~\cite{switchml, byteps2020, pipedream}.

Lossy gradient compression techniques are becoming a common approach to rein in communication efficiency~\cite{grace}. In particular, sparsification, which sends only a subset of gradient coordinates (e.g., Top-$k$~\cite{aji_sparse, Alistarh18_sparse} sends the $k$ largest gradient coordinates by magnitude in each iteration), may significantly reduce data volumes and thus speed up training.
However, due to its lossy nature, compression raises a complex trade-off between training performance and accuracy.
For instance, Agarwal et al.~\cite{agarwal2020accordion} note that training ResNet-18 on CIFAR-100 using sparsification speeds up training significantly (3.6$\times$), but it also degrades final accuracy by 1.5\%. On the other hand, Lin et al.~\cite{lin2018deep} reports a $500\times$ data reduction via sparsification under deep gradient compression (DGC) for ResNet-50 on ImageNet while preserving the same final accuracy when adopting a carefully-tuned warmup phase.

The vast literature on gradient compression largely considers a fixed communication budget per iteration while leaving it up to practitioners to grapple with specifying an additional hyper-parameter that determines the degree of compression before training begins.
Meanwhile, recent adaptive Top-$k$ sparsifiers~\cite{agarwal2020accordion, zhong2021adaptive} empirically demonstrate that tuning the degree of compression in different phases of DNN training yields a more communication-efficient scheme than a fixed communication scheme (e.g., a static $k$ for Top-$k$).
However, these works lack a theoretical framework proving that adaptive compression enjoys better convergence guarantees than the fixed compression scheme.

This raises a fundamental question: \emph{Given a fixed communication budget, is there a provably better communication scheme than fixed per-iteration compressed communication?}
In this paper, we first observe that Top-$k$ is the communication-optimal sparsifier for a fixed per-iteration communication budget (\S\ref{sec:topk}).
Then, our insight is that by adopting a different perspective that accounts for \emph{the effect of sparsification throughout training}, a more efficient communication scheme is possible under a revised notion of optimality that considers an overall communication budget (instead of a per-iteration budget).

Consequently, we consider sparsification by using the \emph{error-feedback} (EF) mechanism \cite{stich2018sparsified, Alistarh18_sparse}, a delayed gradient component update strategy that is instrumental for the convergence of sparsifiers. Let $e_t$ denote the error arising due to sparsification at iteration $t$. In EF, this error is added to the gradient update at iteration $t+1$. We identify that the term affecting the non-convex convergence in EF-SGD is the \emph{total-error:} $\sum_t\norm{e_t}^2$ \cite{ef-sgd, stich2019error}.

\begin{figure*}[t]
\centering
\begin{subfigure}{0.32\textwidth}
		\centering
		\includegraphics[width=\textwidth]{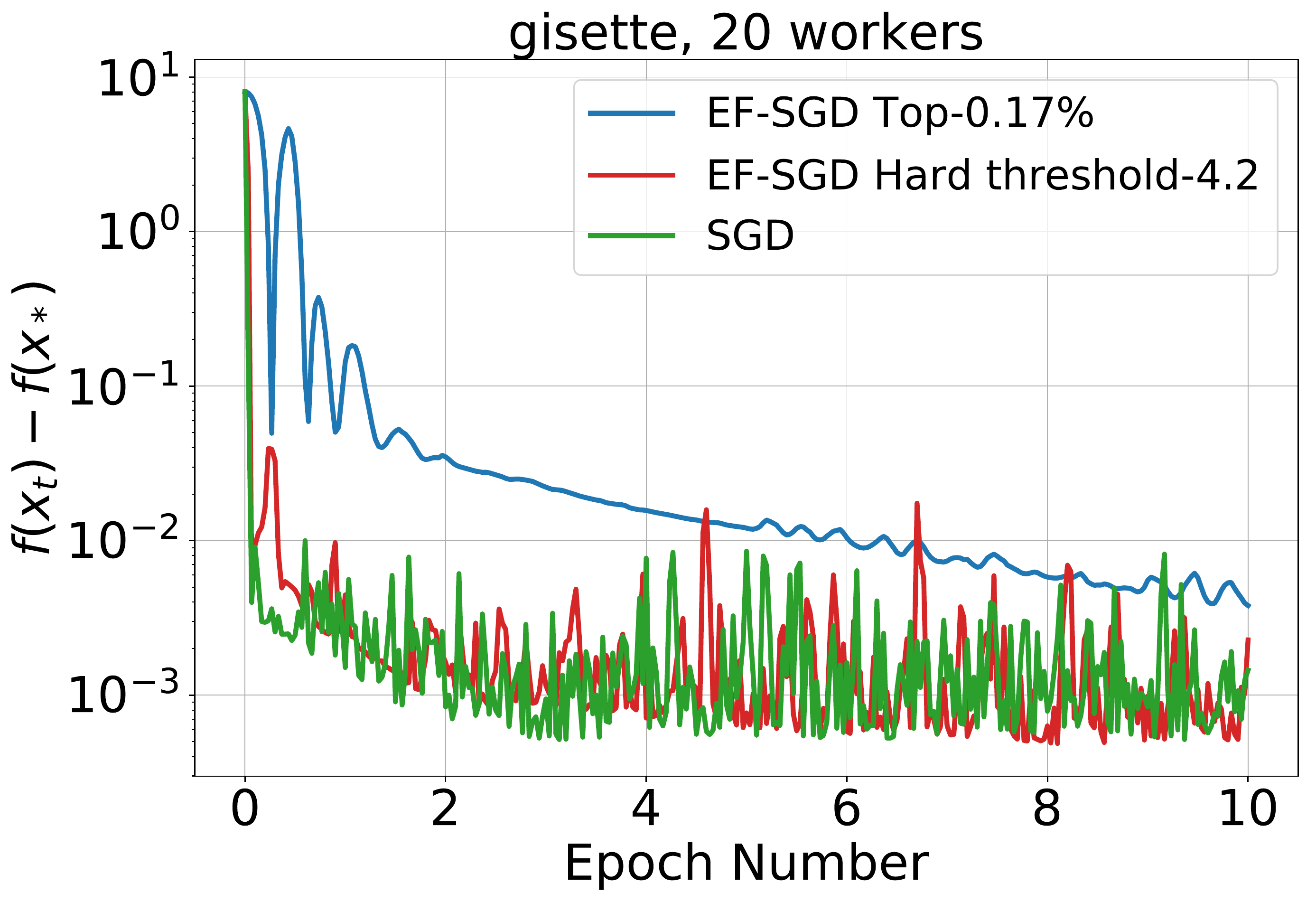}
				\caption{}
	\end{subfigure}
\begin{subfigure}{0.32\textwidth}
		\centering
		\includegraphics[width=\textwidth]{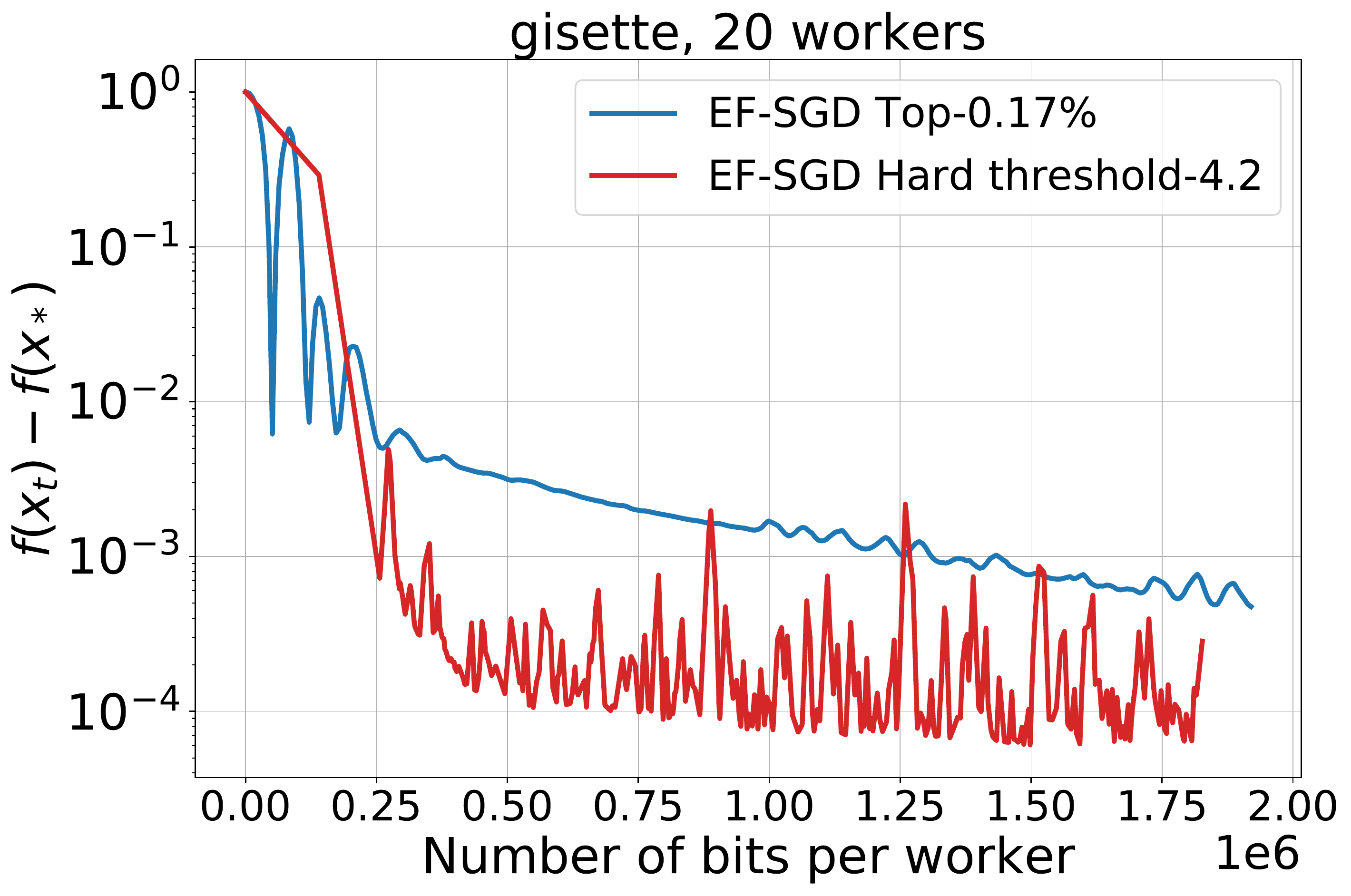}
		\caption{}
	\end{subfigure}
\begin{subfigure}{0.32\textwidth}
        \centering
		\includegraphics[width=\textwidth]{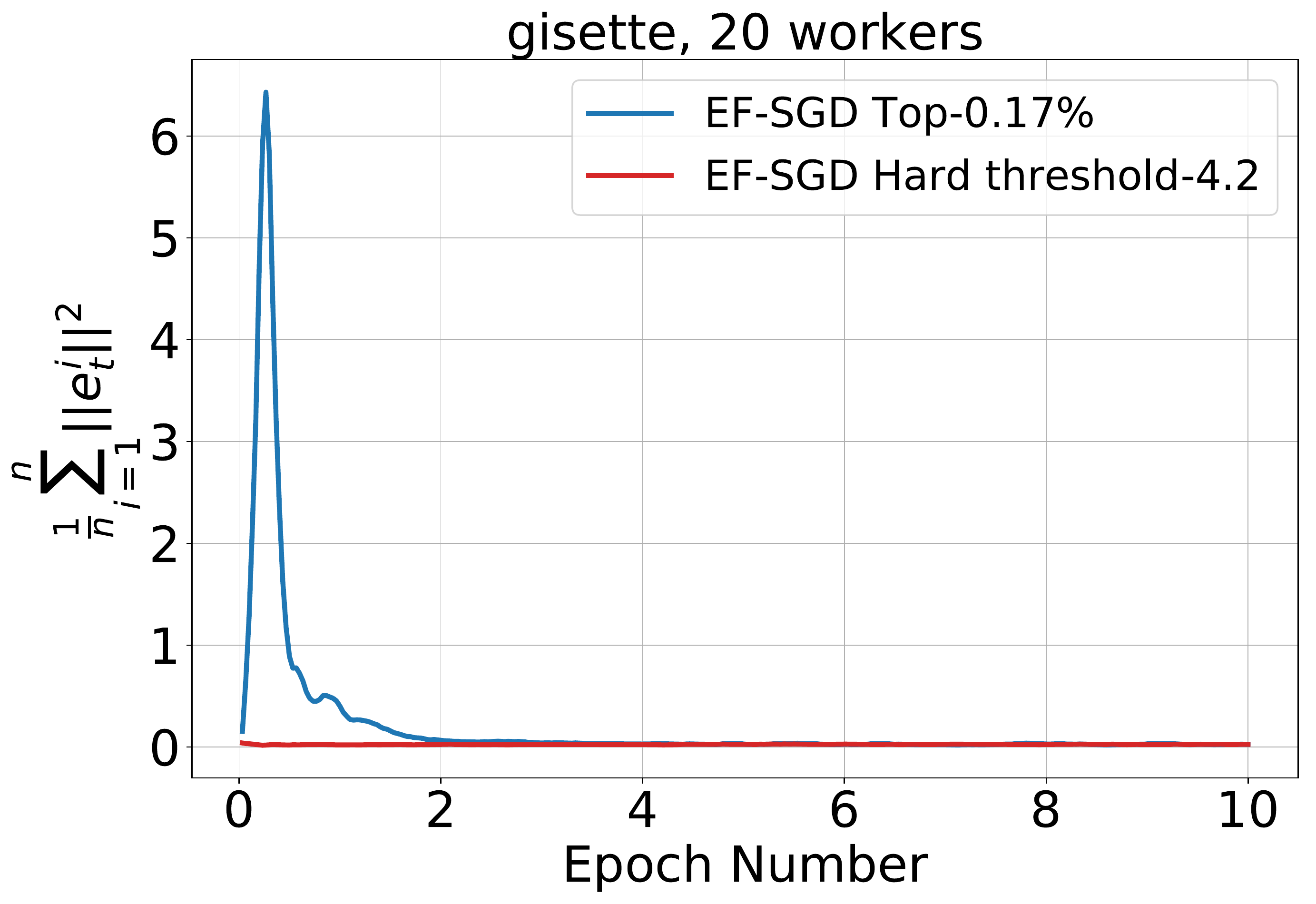}
		\caption{}
	\end{subfigure}
\caption{
\small{Convergence of Top-$k$ and Hard-threshold for a logistic regression model on \texttt{gisette} LIBSVM dataset with 20 workers: (a) Functional suboptimality vs. epochs; (b) functional suboptimality vs. bits communicated;  (c) error norm vs. epochs. Hard-threshold converges as fast as the baseline, no compression SGD and much faster than Top-$k$ because of a smaller total-error than Top-$k$.}}\label{fig:topk_error_accumulation}
\end{figure*}

Directly minimizing the total-error is not possible; thus, Top-$k$ minimizes $\norm{e_t}^2$ at each iteration. We argue that it is possible to focus on the \emph{sum of $\norm{e_t}^2$} and devise a communication scheme that achieves a smaller total-error than any fixed communication sparsifier.
We demonstrate that to achieve this change of perspective; it is sufficient to consider a practical yet straightforward mechanism that is a natural counterpart of Top-$k$: the \emph{hard-threshold sparsifier}, which 
communicates the gradient coordinates with magnitude greater than or equal to a fixed given threshold, $\lambda\ge0$, in each iteration.
Although the two sparsifiers are in an equivalence relation (a given $\lambda$ corresponds to a $k$), under the total-error minimization perspective, we adopt a \emph{fixed} threshold, $\lambda$, which implies a \emph{variable} $k$ at every iteration.

To illustrate intuitively why this change of perspective yields benefits, consider the following example.
Figure \ref{fig:topk_error_accumulation} shows an experiment where we train on 20 workers a 6,000-parameter logistic regression model on the \texttt{gisette} LIBSVM dataset \cite{libsvm} using the Top-$k$ and hard-threshold sparsifiers configured to send the \emph{same data volume}.\footnote{We train for $10$ epochs and set $k=0.17\%$ for Top-$k$, and $\lambda=4.2$ for hard-threshold.}
The loss function is strongly convex and has a unique minimizer, $x_\star$, therefore, a unique optimum, $f(x_\star)$.
We see that hard-threshold converges at the same speed as SGD while communicating $\sim 600 \times$ less data, whereas Top-$k$ has a significantly slower convergence speed.
We attribute this to the fact that Top-$k$ has a large error accumulation in the initial $500$ iterations, while the error magnitude for hard-threshold is less than $0.04$ throughout training (cf. Figure \ref{fig:topk_error_accumulation}c). Our results with DNN training also reflect this insight (\S\ref{sec:experiments}).

Moreover, the hard-threshold sparsifier has computational benefits over Top-$k$ sparsifier, as hard-threshold's underlying filtering operation requires $d$ comparisons in each iteration, where $d$ is the number of parameters. In contrast, Top-$k$ is a compute-intensive sparsifier (e.g., on CPU, the computational complexity is $\mathcal{O}(d\log_{2}k)$~\cite{ShanbhagMIT2018}). For GPUs, several optimized implementations are proposed but they rely on the data distribution and are efficient only for a small $k$~\cite{ShanbhagMIT2018}. For instance, PyTorch uses Radix select algorithm which has a computational complexity of $\mathcal{O}\left(\lceil{b}/{r}\rceil~d\right)$ where $b$ is the number of bits to represent gradient values and $r$ is the radix size~\cite{pytorch}.

Finally, while the hard-threshold sparsifier already exists in the literature \cite{Strom15,layer-wise}, we are the first to formally study it and theoretically demonstrate its benefits as an adaptive counterpart of Top-$k$. Moreover, our argument in favor of hard-threshold precisely falsifies the claim by Dryden et al. \cite{Dryden2016CommunicationQF} that stopped its widespread adoption --- \emph{a hard-threshold may lead to a degenerate situation when the EF in gradient compression builds up.}

This paper makes the following contributions:

\myNum{i}\smartparagraph{Communication complexity model (\S\ref{sec:communication_complexity_model}).} We propose a communication complexity model that captures the effects of compression in the entire optimization process. We allow for variable communication in each iteration by only imposing a total communication budget. We show that the hard-threshold sparsifier is the communication-optimal sparsifier in this model.\\ 
\myNum{ii}\smartparagraph{Absolute compressors (\S\ref{sec:absolute_compressors}).} We identify that the hard-threshold sparsifier, along with other existing compressors \cite{Dettmers20158, switchml}, belongs to the class of \emph{absolute compressors}, which have an absolute bound on the error. Absolute compressors have not been formally studied before with EF. We show that absolute compressors with EF converge for both strongly convex and non-convex loss functions. In both cases, absolute compressors enjoy the same asymptotic convergence with linear speedup (with respect to the number of workers) as no-compression SGD. Notably, $\delta$-contraction operators \cite{stich2018sparsified}, the class in which Top-$k$ belongs, do not have the same asymptotic convergence and linear speedup property in the strongly convex case. This was previously considered as a drawback of EF in general over unbiased compression \cite{horvath2021a}. Moreover, for the non-convex case, $\delta$-contraction operators have a worse dependence on $\delta$ in the distributed setting with heterogeneous data, while absolute compressors do not have such an anomaly.\\
\myNum{iii}\smartparagraph{Experiments (\S\ref{sec:experiments}).} We conduct diverse experiments on both strongly convex and non-convex (for DNNs) loss functions to substantiate our claims. Our DNN experiments include computer vision, language modeling, and recommendation tasks, and our strongly convex experiment is on logistic regression. We find that the hard-threshold sparsifier is consistently more communication-efficient than the Top-$k$ sparsifier given the same communication budget.

\section{Related work}
\smartparagraph{Gradient compression} techniques can be broadly classified into quantization \cite{alistarh2017qsgd, DBLP:conf/interspeech/SeideFDLY14,Dettmers20158, ef-sgd, DBLP:conf/nips/WenXYWWCL17}, sparsification \cite{aji_sparse,lin2018deep, WangniWLZ18}, hybrid compressors \cite{Qsparse-local-SGD, Strom15, Dryden2016CommunicationQF}, and low-rank methods \cite{PowerSGD, ATOMO}.~The state-of-the-art compressors are biased $\delta$-contraction operators~\cite{PowerSGD,lin2018deep}, see \S\ref{sec:discussion}.~We refer to \cite{grace} for a recent survey and quantitative evaluation of these techniques.

\smartparagraph{Error-feedback}~(EF) or memory was first empirically used in \cite{DBLP:conf/interspeech/SeideFDLY14, Strom15}.~However, \cite{stich2018sparsified, ef-sgd, stich2019error} were the first to give a convergence analysis of the EF framework, which was extended to the distributed setup in \cite{beznosikov2020biased,zheng2019communication}. 
Recently, \cite{cser2020} proposed {\em error-reset}, a different form of EF, while~\cite{horvath2021a} introduce another alternative by communicating a compressed version of the error. \cite{gorbunov2020linearly} propose and analyze multiple algorithms combining EF with variance reduction in a general frame-work. 

\smartparagraph{Communication-optimal compression.} \cite{gandikota2019vqsgd, safaryan2020uncertainty, wei2020trilemma, alyazeed2020optimal} devise a communication-optimal compressor by minimizing the worst-case compression factor\footnote{For a vector $x$ and a possibly randomized compression operator $\cC$, we denote the compression error as $\mbE_{\cC}[\norm{x-\cC(x)}^2]$, and compression factor as $\mbE_{\cC}[\norm{x-\cC(x)}^2]/\norm{x}^2$.} under a per-vector communication budget.

\smartparagraph{Adaptive compression.}
\cite{WangniWLZ18} designed an adaptive sparsifier that minimizes expected sparsity of the compressed vector under a given variance budget. While AdaQS \cite{guo2020adaqs} periodically doubles the quantization states in QSGD \cite{alistarh2017qsgd} to reduce the compression factor, DQSGD~\cite{yan2021dqsgd} sets the number of quantization states proportional to the gradient norm. ACCORDION \cite{agarwal2020accordion} chooses a low compression ratio if training is in the critical regime, and a high compression ratio otherwise.

\smartparagraph{Efficient Top-$k$ estimation} is a focus of many recent works.
While \cite{lin2018deep} estimates the Top-$k$ threshold on a randomly sampled subset, \cite{shi2019understanding, sayed2021sidco} estimate it by fitting parametric statistical distributions. \cite{dctgupta2020} estimates the threshold every $1,000$ iterations. These works determine a threshold to approximately determine a \emph{fixed} Top-$k$ set, while we use a hard-threshold to determine the $k$ in each iteration.

\section{Background on Error-Feedback (EF) SGD}

Consider the following distributed optimization problem with $n$ workers:\useshortskip
\begin{equation}\label{eq:opt}
\textstyle    \min_{x\in\R^d} f(x) \coloneqq \frac{1}{n}\sum_{i\in[n]}f_i(x) ,
\end{equation}
where $f_i(x)=\mbE_{z_i \sim \cD_i}{l(x;z_i)}$ denotes the loss function evaluated on input $z_i$ sampled from $\cD_i$, the data-distribution at the $i^{th}$ worker. Let $g_{i,t}$ denote the stochastic gradient computed at $i^{th}$ worker at iteration $t$ such that $\textstyle g_{i,t} = \nabla f_i(x_{t}) + \xi_{i,t}$ with $\mbE [\xi_{i,t}| x_{t}] = \mathbf{0}$.

Algorithm \ref{algo:ef-sgd} gives the pseudo-code of compressed EF SGD \cite{stich2018sparsified, ef-sgd, Alistarh18_sparse} to solve \eqref{eq:opt}. Let $e_{i,t}$ denote the locally accumulated error at $i^{th}$ worker due to compression from previous steps. Adding this error to the current gradient, $\gamma_{t} g_{i,t}$ provides the corrected update, $p_{i,t}$. This corrected update is further compressed and exchanged with other machines, and the local error is updated for the next step.
\begin{algorithm}
    \SetAlgoLined
    \SetKwFor{ForP}{for}{in parallel do}
    \DontPrintSemicolon
    \SetInd{0.5em}{0.75em}
    \kwInput{$\cC$-compressor}
    Initialize $x_{0}\in \R^d$\;
    \ForP{worker $i \in [n]$}{
        Initialize error $e_{i,0}=0$\;
        \For{iteration $t = 0,\hdots, T-1$}{
        $p_{i,t}\leftarrow e_{i, t} + \gamma_{t} g_{i,t}$ \tcc*[f]{Incorporate EF into update}\;
        $\Delta_{i,t}\leftarrow \gamma_{t}\cC(\frac{p_{i,t}}{\gamma_{t}})$ \;
        $e_{i,t+1}\leftarrow p_{i,t}-\Delta_{i,t}$ \tcc*[f]{Update error}\;
        $\bar{\Delta}_t=\frac{1}{n}\sum_{i\in [n]} \Delta_{i,t}$ \tcc*[f]{Exchange and average $\Delta_{i,t}$ among workers}\;
        $x_{t+1} \leftarrow x_{t} - \bar{\Delta}_t$\;
     }
    }
\caption{Distributed EF SGD}\label{algo:ef-sgd}
\end{algorithm}
\begin{remark}
In Algorithm \ref{algo:ef-sgd}, we have $\Delta_{i,t}\leftarrow \gamma_{t}\cC(\frac{p_{i,t}}{\gamma_{t}})$, while \cite{stich2018sparsified, ef-sgd, Alistarh18_sparse} consider $\Delta_{i,t}\leftarrow \cC(p_{i,t})$. We do this to extend the EF framework for absolute compressors. We note that Algorithm \ref{algo:ef-sgd} is more general as $\gamma_{t}\cC(\frac{p_{i,t}}{\gamma_{t}})$ is equivalent to $\cC(p_{i,t})$ for all known $\delta$-contraction operators.
\end{remark}

\vspace{-7pt}
\subsection{Assumptions}\label{sec:assumptions}
We consider the following general assumptions on the loss function.
\begin{assumption}\label{ass:smoothness}
\textbf{(Smoothness)} The function, $f_i: \R^d\to \R$ at each worker, $i \in [n]$ is $L$-smooth, i.e., for every $x,y\in\R^d$ we have, $f_i(y)\leq f_i(x)+\dotprod{\nabla f_i(x), y-x}+\frac{L}{2}\norm{y-x}^2$. 
\end{assumption}
\begin{assumption}\label{ass:minimum}
\textbf{(Global minimum)} 
There exists $x_\star$ such that, $f(x_\star)=f^\star\leq f(x)$, for all $x\in\R^d.$  
\end{assumption}
\begin{assumption}\label{ass:m-sigma2-bounded_noise}
\textbf{(($M, \sigma^2$) bounded noise)} \cite{stich2019error} For every stochastic noise $\xi_{i,t}$, there exist $M, \sigma^2\geq0$, such that 
$\mbE [\norm{\xi_{i,t}}^2 \mid x_{t}] \leq M\norm{\nabla f_i(x_{t})}^2+\sigma^2$, for all $x_t\in\R^d$.
\end{assumption}
\begin{remark}\label{remark:noise model}
The above assumption implies $\mbE[{\norm{g_{i,t}}^2}\mid x_{t}]\leq (M+1)\norm{\nabla f_i(x_{t})}^2+\sigma^2$. This general noise model does not uniformly bound the second moment of stochastic gradients as in \cite{ef-sgd,zheng2019communication,Qsparse-local-SGD}.
\end{remark}
\begin{assumption}\label{ass:bounded_similarity}
\textbf{($(C, \zeta^2)$ bounded similarity)} The variance of gradients among workers is bounded, i.e., there exist constants, $C,\zeta\geq0$ such that,
$\frac{1}{n}\sum_{i \in [n]}\norm{\nabla f_i(x) - \nabla f(x)}^2 \leq C\norm{\nabla f(x)}^2 +\zeta^2,$ for all $x\in\R^d$.
\end{assumption}
\begin{remark}
If all the workers have the same training data, all $f_i$ are the same, resulting in $C,\zeta=0$. This assumption is an extension from \cite{lian2017can, jiang2018linear}, which consider $C=0$.
\end{remark}
For the \emph{convergence of strongly convex functions}, we require an additional assumption as follows.
\begin{assumption}\label{ass:mu-strong-convex}
\textbf{($\mu$-strong convexity)} The functions, $f_i:\R^d\to\R$ are 
$\mu$-strongly convex, i.e., there exists $\mu\geq0$, such that
$    f_i(y) \geq f_i(x)+ \dotprod{\nabla f_i(x), y-x}+ \frac{\mu}{2}\norm{x-y}^2,$ for all $x\in\R^d$.
\end{assumption}

\smartparagraph{Convergence of EF-SGD.} The following result shows the convergence of EF-SGD \cite{ef-sgd} in minimizing general smooth functions for a single worker $(n=1)$ case. 
\begin{theorem} \label{thm:ef-gen} \cite{ef-sgd, stich2019error} 
Let Assumption \ref{ass:smoothness}, \ref{ass:minimum} and \ref{ass:m-sigma2-bounded_noise} hold. Then Algorithm \ref{algo:ef-sgd} with a constant step-size, $\gamma$ where $\gamma\leq \frac{1}{2L(M+1)}$ and $n=1$ follows
\begin{eqnarray*}
\frac{1}{T}\sum_{t=0}^{T-1}\mbE\norm{\nabla f(x_t)}^2 \le \frac{4(f(x_0)-f^{\star})}{\gamma T} +2\gamma L \sigma^2+  2L^2\sum_{t=0}^{T-1}\frac{\gamma^2\mbE\norm{ g_t + \frac{e_t}{\gamma}-\cC(g_t + \frac{e_t}{\gamma})}^2}{T}.
\end{eqnarray*}
\end{theorem}
\begin{remark}\label{rem:total-error}
Theorem \ref{thm:ef-gen} is a simplified version of the distributed case for $n>1$ and quoted to emphasize the effect of compression between the error-corrected gradient, $g_t + \frac{e_t}{\gamma}$, and its compressed form, $\cC(g_t + \frac{e_t}{\gamma})$. The term that \emph{solely accounts for the effect of compression in the entire training process} is the \emph{total-error}: $\sum_{t=0}^{T-1}\mbE\norm{e_{t+1}}^2=\sum_{t=0}^{T-1}\gamma^2\mbE\norm{ g_t + \frac{e_t}{\gamma}-\cC(g_t + \frac{e_t}{\gamma})}^2$.
\end{remark}

\section{A communication complexity perspective to sparsification}\label{sec:communication_complexity_model}
We now propose a communication complexity model and contrast it with existing communication-optimal strategies. We start with a sparse approximation problem that we encounter in our subsequent discussions.
\subsection{A sparse approximation problem}
Let $p\in\R^m$ be a given vector.~We want to approximate $p$ with a sparse vector, $q$, that has at most $0<\tau\le m$ non-zero elements. Formally, we write the {\em constrained sparse approximation} problem as:
\begin{equation}\label{eq:sparse_approx_problem}
    \textstyle    q^\star=\arg\min_{q \in \R^{m}} \norm{p-q}^2
     \quad \text{subject to } \|q\|_0 \leq \tau ,
\end{equation}
where $\|\cdot\|_0$ denotes the number of non-zero elements in a vector.
~Problem \eqref{eq:sparse_approx_problem} and its variants are well studied and arise in signal processing \cite{donoho2006compressed,bryan2013making, dutta_thesis} and matrix approximation \cite{stewart1993early, boas2017shrinkage}. 

\begin{lemma}\label{sol:sparse_approx_problem}
The solution $q^\star$ to \eqref{eq:sparse_approx_problem} is obtained by keeping Top-$\tau$ magnitude entries from $p$ and setting the rest to zeros.
\end{lemma}

\subsection{Minimizing the total-error is not possible}\label{sec:cannot_min_total_error}
Let $\fC$ denote the class of all compressors. 
We constrain to the class of deterministic sparsifiers, denoted by $\mS \subset \fC$, but one can similarly consider other subclasses in $\fC$. For each $x\in\R^d$, a deterministic sparsifier, $\cC_p$ with sparsification parameter, $p$ determines a sparse support set, $S_p(x)\subseteq[d]$ and sparsifies as
\begin{equation*}\label{eq:def_sparse}
 \textstyle   \cC_p(x) = \sum_{i \in S_p(x)}x[i] e_i, 
\end{equation*}
where $e_i$ denotes the $i^{th}$ standard basis in $\R^d$, and $x[i]$ denotes the corresponding element in $x$. For example, for hard-threshold sparsifier, $\cC_{\lambda}$, we have $S_{{\lambda}}(x) = \{i \mid |x[i]|\geq\lambda\}.$ Motivated by Theorem \ref{thm:ef-gen} and Remark \ref{rem:total-error}, we now propose the following  communication complexity model:
\begin{eqnarray}\label{eq:opt_comm_total}
    \min_{\cC \in \mS} \sum_{t=0}^{T-1} && \mbE\norm{ g_t + \frac{e_t}{\gamma}-\cC(g_t + \frac{e_t}{\gamma})}^2 \quad \text{subject\;to\;} \sum_{t=0}^{T-1}\norm{\cC(g_t + \frac{e_t}{\gamma})}_0 \leq K,
\end{eqnarray}
where $K$ is the budget on the number of elements communicated in $T$ iterations. However, solving \eqref{eq:opt_comm_total} is intractable since the optimization process is dynamic with new gradient and error depending on the previous gradients and errors, and involves randomness due to stochastic gradients. 

\subsection{Top-$k$ is communication-optimal for a per-iteration $k$ element budget}\label{sec:topk}
To simplify \eqref{eq:opt_comm_total}, one can focus individually at the error at each iteration. Based on this, we show in Lemma \ref{lemma:top-k} that Top-$k$ has the best compression error among all sparsifiers under a per-iteration $k$-element communication budget. 
\begin{lemma}\label{lemma:top-k}
Given the gradient $g_t$ and error $e_t$ at iteration $t$, Top-$k$ sparsifier achieves the optimal objective for the optimization problem:
\begin{equation}\label{eq:per-iteration-optimal}
    \min_{\cC\in\mS} \norm{g_t + \frac{e_t}{\gamma}-\cC(g_t + \frac{e_t}{\gamma})}^2 \quad \text{subject\;to\;} \norm{\cC(g_t + \frac{e_t}{\gamma})}_0 \leq k.
\end{equation}
\end{lemma}

Similar to Lemma \ref{sol:sparse_approx_problem}, \eqref{eq:per-iteration-optimal} is solved when $\cC(g_t + \frac{e_t}{\gamma})$ contains the $k$ highest magnitude elements of $g_t + \frac{e_t}{\gamma}$. That is, when $\cC$ is the Top-$k$ sparsifier. Additionally, based on the above model, a per-iteration $k$-element communication budget (resulting in a total budget of $kT$ elements throughout training), implies that Top-$k$ is performed at each iteration. However, to have a more communication-efficient compression, we require a communication complexity model that \myNum{i} better captures total-error in Theorem \ref{thm:ef-gen}; and \myNum{ii} allows for adaptive communication, i.e., sends variable data in each iteration.

\subsection{A communication complexity model for adaptive sparsification}

Although the total-error cannot be minimized (\S\ref{sec:cannot_min_total_error}), Lemma \ref{lemma:top-k} motivates us to consider a {\em simplified model} that can capture the total-error. Instead of $(g_t + \frac{e_t}{\gamma})_{t=0}^{T-1}$, we consider a fixed sequence $(a_t)_{t=0}^{T-1}$ and examine the following communication complexity model:
\begin{equation}\label{eq:opt_comm_total_fixed}
 \textstyle   \min_{\cC \in \mS} \sum_{t = 0}^{T-1} \norm{a_t-\cC(a_t)}^2 
     \quad \text{subject\;to\;} \sum_{t = 0}^{T-1}\|\cC(a_t)\|_0 \leq K,
\end{equation}
where $K\in\N$ denotes the total communication budget. 

Let $\cA\in\R^{d T}$ be formed by stacking $(a_t)_{t=0}^{T-1}$ vertically and consider the following sparse approximation problem:
\begin{equation}\label{eq:matrix_opt_comm_total_fixed}
\textstyle    \min_{B \in \R^{dT}}\norm{\cA-B}^2
     \quad \text{subject\;to\;} \|B\|_0 \leq K,
\end{equation}
Note that \eqref{eq:matrix_opt_comm_total_fixed} allows for all $B$ that are formed by stacking $(\cC(a_t))_{t=0}^{T-1}$ vertically, for some sparsifier $\cC$ satisfying $\sum_{t=0}^{T-1}\|\cC(a_t)\|_0 \leq K$. Therefore, the optimal objective for \eqref{eq:matrix_opt_comm_total_fixed} is a lower bound to the optimal objective for \eqref{eq:opt_comm_total_fixed}. Let $\cA_{(i)}$ denote the element with $i^{th}$ largest magnitude in $\cA$, and assume no two elements have the same magnitude, i.e., $\cA_{(i+1)}\neq \cA_{(i)},$ for all $i\in[dT]$. Then, $B=\cC_{\lambda}(A)$ with $\lambda \in \left(A_{\left(K+1\right)}, A_{\left(K\right)}\right]$ contains the Top-$K$ magnitude entries from $A$, and therefore, by Lemma \ref{sol:sparse_approx_problem} is optimal for \eqref{eq:matrix_opt_comm_total_fixed}. Moreover, since hard-threshold is an element-wise sparsifier, $\cC_{\lambda}(A)$ is equivalent to stacking $(\cC_{\lambda}(a_t))_{t=0}^{T-1}$ vertically. Therefore, $\cC_{\lambda}$ with $\lambda=\left(\cA_{\left(K+1\right)}, \cA_{\left(K\right)}\right]$ achieves optimal objective in \eqref{eq:opt_comm_total_fixed}. The following lemma formalizes this.

\begin{lemma}\label{lemma:hard-threshold}
$\cC_{\lambda}$ is optimal for the communication complexity model \eqref{eq:opt_comm_total_fixed}. That is, for every budget $K$, there exists a $\lambda\ge0$ such that $\cC_{\lambda}$  minimizes \eqref{eq:opt_comm_total_fixed}.
\end{lemma}

\subsection{Discussion}\label{sec:discussion}
To capture the effect of compression, existing works \cite{alistarh2017qsgd, stich2018sparsified, ef-sgd} use a bound on the \emph{compression factor}, $\max_{x\in\R^d} \frac{\mbE_{\cC} \norm{\cC(x)-x}^2}{\norm{x}^2}$. We formally define them as relative compressors.

\begin{definition}\textbf{Relative Compressor} \cite{alistarh2017qsgd,stich2018sparsified}\label{defn:rel_comp}. An operator, $\cC:\R^d\to\R^d$ is a {\em relative compressor} if for all vector, $x\in\R^d$ it satisfies 
\useshortskip\begin{eqnarray}\label{eq:quantization}
 \textstyle \mathbb{E}_{\cC}\|x-\cC(x)\|^2\le \Omega \|x\|^2, 
\end{eqnarray}
where $\Omega >0$ is the compression factor and the expectation, $\mbE_{\cC}$, is taken with respect to the randomness of $\cC$. {\em $\delta$-contraction operators} \cite{stich2018sparsified,ef-sgd} with $\Omega=1-\delta$ and $\delta\in(0,1]$, are special cases of relative compressors. 
\end{definition}

Top-$k$ is a $\delta$-contraction operator with $\delta=\frac{k}{d}$ \cite{stich2018sparsified}. 
Therefore, by \eqref{eq:quantization}, Top-$k$ allows for larger compression error with larger inputs. Our communication complexity model demonstrates that this might not necessarily be a good idea. Moreover, with EF, a large error at any iteration has a cascading effect --- a large $e_t$ results in a large $\gamma g_t + e_t$, out of which only $k/d$ fraction of the total components are kept by the Top-$k$ strategy.  
This results in a large $e_{t+1}$ (see \S\ref{sec:large_error_topk}). Figure \ref{fig:topk_error_accumulation} shows that this \emph{error-buildup} has severe implications on the total-error. On the other hand, the hard-threshold performs a variable Top-$k$ in each iteration and sends an element as soon as its magnitude is bigger than the threshold. This prohibits the error build-up. 

\smartparagraph{Comparison with existing communication-optimal compression strategies.} 
Since the compression factor, $\Omega$ solely determines the effect of compression in convergence \cite{horvath2021a, stich2019error}, many recent works \cite{gandikota2019vqsgd, safaryan2020uncertainty, wei2020trilemma, alyazeed2020optimal} propose communication-optimal compression strategies by optimizing for $\Omega$ under a {\em communication budget} for each vector, i.e., they propose to solve
\begin{equation}\label{eq:complexity_model_fixed}
    \min_{\cC \in \fC} \max_{x \in \R^d}  \frac{\mbE_{\cC}\norm{x-\cC(x)}^2}{\norm{x}^2} \quad \text{subject\;to\; \texttt{Bits}}(\cC(x))\leq B,
\end{equation}
where \texttt{Bits}$(\cC(x))$ denotes the number of bits needed to encode $\cC(x)$.
We stress that while the compression affected term in Theorem \ref{thm:ef-gen} has the sum of compression errors over the iterations, the above communication complexity model only captures the compression factor.

\vspace{-0.5em}
\section{Absolute compressors and their convergence}\label{sec:absolute_compressors}
Motivated by the previous section, we formally define \emph{absolute compressors}---compressors that have an absolute bound on the error.

\begin{definition}\textbf{Absolute Compressor.} An operator, $\cC:\R^d\to\R^d$ is an {\em absolute compressor} if there exists a $\kappa>0$ such that for all vector, $x\in\R^d$ it satisfies
\useshortskip\begin{eqnarray}\label{eq:absolute_compression}
 \textstyle \mathbb{E}_{\cC}\|x-\cC(x)\|^2\le \kappa^2.
\end{eqnarray}
\end{definition}
In contrast to the relative compressors in \eqref{eq:quantization}, the compression error (or variance) of absolute compressors is bounded by a constant, independent of $x$. 
Based on the above definition, hard-threshold is an absolute sparsifier with $\kappa^2=d\lambda^2$. The stochastic rounding schemes with bounded rounding error in \cite{Gupta_2015} are absolute compressors~(used for model quantization).~Precisely, any rounding scheme, with rounding error bounded by $\epsilon$, is an absolute compressor with $\kappa^2=d\epsilon^2$.~Similarly, the scaled integer rounding scheme in \cite{switchml} is an absolute compressor. While this class of compressors existed in the literature, we are the first to provide their convergence result with an EF. 

\subsection{Convergence results}

Inspired by \cite{stich2019error}, we establish convergence of EF-SGD (Algorithm \ref{algo:ef-sgd}) with absolute compressors. Convergence analysis for the momentum case \cite{zheng2019communication} can be extended similarly. However, we do not include it for brevity, and the existing analyses do not show any benefit over vanilla SGD.~Similarly, analysis for error-reset \cite{cser2020} and local updates \cite{Qsparse-local-SGD, cser2020} can also be extended. We provide convergence results for the convex and non-convex cases and compare them to $\delta$-contraction operators. We start with a bound on the error for absolute compressors. 
\begin{remark}\label{rem:bounded_memory}
\textbf{(Error bound)}
For all $i \in [n], t \in \{0, \hdots T-1\}$, we have
\begin{equation*}
  \textstyle  \mbE_{\cC} [\norm{e_{i,t+1}}^2 \mid p_{i,t}] = \mbE_{\cC} \norm{p_{i,t}-\gamma_t\cC(\frac{p_{i,t}}{\gamma_t})}^2 = \gamma_t^2\mbE_{\cC} \norm{\frac{p_{i,t}}{\gamma_t}-\cC(\frac{p_{i,t}}{\gamma_t})}^2 \leq \gamma_t^2 \kappa^2.
\end{equation*}
\end{remark}
A similar absolute bound for $\delta$-contraction operators requires the bounded gradient assumption \cite{ef-sgd,Qsparse-local-SGD}, but absolute compressors achieve this by design. 

\subsubsection{Convex convergence}

Throughout this section, we denote $P_{T}\eqdef\mbE[f(\bar{x}_T)]-f^{\star}$ and $R_T\eqdef\norm{x_T-x^{\star}}^2$ for $T\ge0$, where $\bar{x}_T=\frac{1}{W_T}\sum_{t=0}^Tw_tx_t$, and $W_T=\sum_{t=0}^Tw_t$ with $w_t\ge 0$. Further denote  $D\eqdef\frac{1}{n}\sum_{i=1}^n\norm{\nabla f_i(x^{\star})}^2$. With these notations, we quote the strongly convex ($\mu>0$) and convex ($\mu=0$) convergence results for absolute compressors, and compare them with the $\delta$-contraction operators from \cite{beznosikov2020biased} for distributed case ($n\ge1$). For exact choices of step-sizes and weights, see \S\ref{sec:abs-quasi}.

\begin{theorem}\label{thm:sc_abs}
Let $\mu>0$ and Assumptions \ref{ass:smoothness}, \ref{ass:minimum}, \ref{ass:m-sigma2-bounded_noise}, and \ref{ass:mu-strong-convex} hold. Then the iterates, $\{x_t\}_{t\geq0}$ of Algorithm \ref{algo:ef-sgd} with an absolute compressor, $\cC_{\upsilon}$, a constant setp-size, $\gamma(T)$ with $\gamma(T)\leq\frac{1}{4L(1+{2M}/n)}$ follow \footnote{The $\tilde{\cO}$ notation hides constants and factors polylogarithmic in the problem parameters.}
\useshortskip
\begin{equation*}
\textstyle P_T = \tilde{\cO}\left(L(1+M/n)R_0\exp\left[-\frac{\mu T}{8L(1+2M/n)}\right]+\frac{\sigma^2+MD}{\mu nT}+ \frac{L\kappa^2}{\mu^2T^2}\right).  \end{equation*}
\end{theorem}

\begin{remark}\label{remark:delta strong}
Under the same setting as in Theorem \ref{thm:sc_abs}, iterates of Algorithm \ref{algo:ef-sgd} with $\delta$-contraction operators \cite{beznosikov2020biased} follow:
\begin{equation*}
\textstyle P_T = \tilde{\cO}\left(L(1/\delta+M)R_0\exp\left[-\frac{\mu T}{28L(2/\delta+M)}\right]+\frac{\sigma^2+D(\frac{1}{\delta}+\frac{M}{n})}{\mu T}\right).
\end{equation*}
\end{remark}
Remark \ref{remark:delta strong} indicates some key convergence behaviors of $\delta$-contraction operators for strongly convex case:\myNum{i} The compression slows down the slower decaying $\cO(\frac{1}{T})$ term, as $\frac{1}{\delta}$ occurs in it. \myNum{ii} They have slower exponential convergence with compression affecting the exponential term. 
\myNum{iii} They do not exhibit the linear speedup in the slower decaying $\cO(\frac{1}{T})$ term.  
The absence of a linear speedup was considered as a disadvantage of EF over unbiased relative compression \cite{horvath2021a}. As our result for absolute compressors with EF in Theorem \ref{thm:sc_abs} has the linear speedup property, we conclude the absence of a linear speedup is a drawback of $\delta$-contraction operators and not EF in general. See further comparison against unbiased compressors \cite{horvath2021a} in \S\ref{sec:unbiased-comp}.
\begin{theorem}\label{thm:weakly-convex}
Let $\mu=0$ and Assumptions \ref{ass:smoothness}, \ref{ass:minimum}, \ref{ass:m-sigma2-bounded_noise}, and \ref{ass:mu-strong-convex} hold. Then the iterates, $\{x_t\}_{t\geq0}$ of Algorithm \ref{algo:ef-sgd} with an absolute compressor, $\cC_{\upsilon}$, a constant setp-size, $\gamma(T)$ with $\gamma(T)\leq\frac{1}{4L(1+{2M}/n)}$ follow
\begin{equation*}
\textstyle P_T = \cO\left(\frac{\sqrt{(\sigma^2+MD)R_0}}{\sqrt{nT}}+\frac{\left(\frac{nL\kappa^2}{\sigma^2+MD}+L(1+M/n)\right)R_0}{T}\right).
\end{equation*}
\end{theorem} 
\begin{remark}
Theorem \ref{thm:weakly-convex} holds when both $\sigma^2$ and $D$ are not simultaneously zero. Typically, we encounter heterogeneous data settings where $D\neq0$, and Theorem \ref{thm:weakly-convex} holds. In case both $\sigma^2$ and $D$ are zero, we get $\cO(\frac{\left(L\kappa R_0\right)^{\frac{2}{3}}}{T^{\frac{2}{3}}})$ convergence. 
\end{remark}
\begin{remark}\label{remark:delta convex}
Under the same setting as in Theorem \ref{thm:weakly-convex}, iterates of Algorithm \ref{algo:ef-sgd} with $\delta$-contraction operators follow \cite{beznosikov2020biased}
\begin{equation*}
\textstyle P_T = \cO\left(\frac{\sqrt{(\sigma^2+D(\frac{1}{\delta}+\frac{M}{n}))R_0}}{\sqrt{T}}+\frac{L(1/\delta+M)R_0}{T}\right).
\end{equation*}
\end{remark}
Remark \ref{remark:delta convex} indicates some key convergence behaviors of $\delta$-contraction operators for convex case: \myNum{i} The compression slows down the slower decaying $\cO(\frac{1}{\sqrt{T}})$ term as $\frac{1}{\delta}$ occurs in it; \myNum{ii} $\delta$-contraction operators do not have linear speedup in the slower decaying  $\cO(\frac{1}{\sqrt{T}})$ term. 

Designing a variance-reduced algorithm by using absolute compressors with EF is orthogonal to this work.

\vspace{-0.5em}
\subsubsection{Non-convex convergence}\label{sec:non_convex_convergence}
\begin{theorem}\label{thm:non-convex-wo-bg}
\textbf{(Non-convex convergence of absolute compressors)} Let Assumptions \ref{ass:smoothness}, \ref{ass:minimum}, \ref{ass:m-sigma2-bounded_noise}, and \ref{ass:bounded_similarity} hold.~Then the iterates, $\{x_t\}_{t\geq0}$ of Algorithm \ref{algo:ef-sgd} with an absolute compressor and a constant setp-size $\textstyle{\gamma \leq \frac{n}{2L(M(C+1)+n)}}$ follow \useshortskip
\begin{eqnarray*}
 \textstyle \frac{1}{T}\sum_{t=0}^{T-1} \mbE \norm{\nabla f(x_t)}^2 \leq \frac{4(f(x_0)-f^{\star})}{\gamma T} +\frac{2\gamma L(M\zeta^2+\sigma^2)}{n} + 2\gamma^2 L^2\kappa^2.
\end{eqnarray*}
 \end{theorem}

Alongside, we compare with the non-convex convergence for $\delta$-contraction operators in a distributed setting. The existing analyses tackle this by using a stronger {\em uniform bounded gradient} assumption \cite{zheng2019communication, ef-sgd, layer-wise}. We use weaker Assumption \ref{ass:m-sigma2-bounded_noise} and \ref{ass:bounded_similarity}, to establish the convergence analysis. 

\begin{theorem}\label{thm:non-convex-wo-bg-rel}
\textbf{(Non-convex convergence of $\delta$-contraction operators)} Let Assumptions \ref{ass:smoothness}, \ref{ass:minimum}, \ref{ass:m-sigma2-bounded_noise}, and \ref{ass:bounded_similarity} hold.~Then the iterates, $\{x_t\}_{t\geq0}$ of Algorithm \ref{algo:ef-sgd} with a $\delta$-compressor and a constant step-size $\textstyle{\gamma \leq \min\{\frac{n}{2L(M(C+1)+n)}, \frac{1}{2L(2/\delta+M)\sqrt{C+1}}\}}$ follow \useshortskip
\begin{eqnarray*}
 \textstyle \frac{1}{T}\sum_{t=0}^{T-1} \mbE \norm{\nabla f(x_t)}^2 \leq \frac{8(f(x_0)-f^{\star})}{\gamma T} + \frac{4 \gamma L(M\zeta^2+\sigma^2)}{n} +\frac{8\gamma^2L^2}{\delta}\left(\left(\frac{2}{\delta}+M\right)\zeta^2+\sigma^2\right).
\end{eqnarray*}
\end{theorem}

The above result implies, in distributed settings with heterogeneous data ($\zeta\neq0$), $\delta$-contraction operators have an $1/\delta^2$ dependence on $\delta$, as compared to an $1/\delta$ dependence in the homogeneous case ($\zeta=0$). In contrast, absolute compressors have the same $\kappa^2$ dependence on $\kappa$ in homogeneous and heterogeneous cases. Therefore, we conjecture
it is beneficial to use absolute compressors in settings such as federated learning \cite{mcmahan2017communication}, where data heterogeneity is widely encountered. 

\begin{remark}\label{rem:non-convex-big-oh}
With appropriate choices of step-size, both absolute compressors and $\delta$-contraction operators with EF-SGD achieve the same $\cO(1/\sqrt{nT})$ asymptotic rate of SGD. See Corollary \ref{cor:big_oh} in \S\ref{apdx:non_convex_convergence} for the full result.
\end{remark}

\begin{figure*}
\centering
\begin{subfigure}{0.32\textwidth}		
		\includegraphics[width=\textwidth]{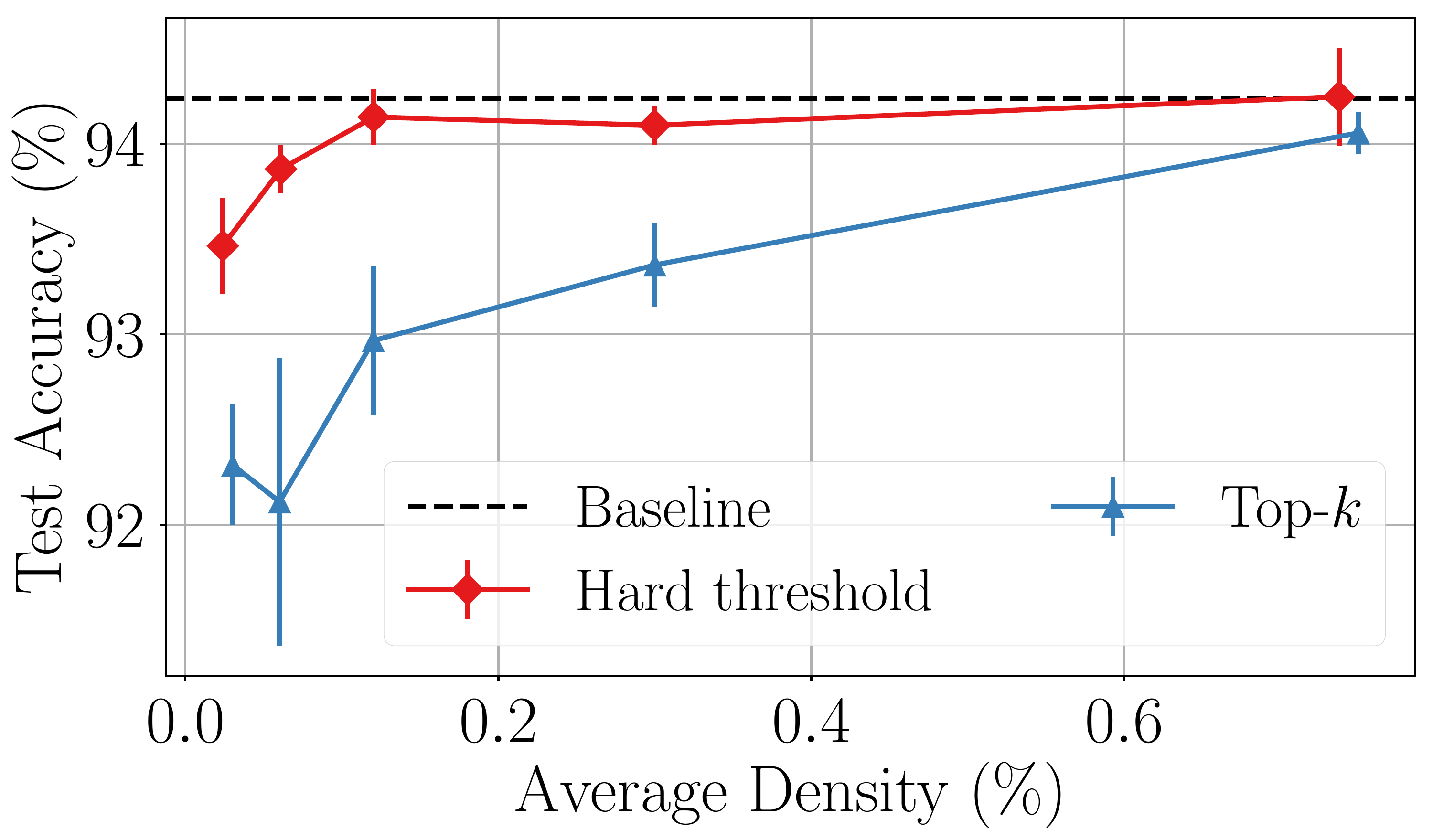}
		\caption{ResNet-18 on CIFAR-10}\label{fig:acc_vs_data_resnet}
	\end{subfigure}
\begin{subfigure}{0.32\textwidth}
		\includegraphics[width=\textwidth]{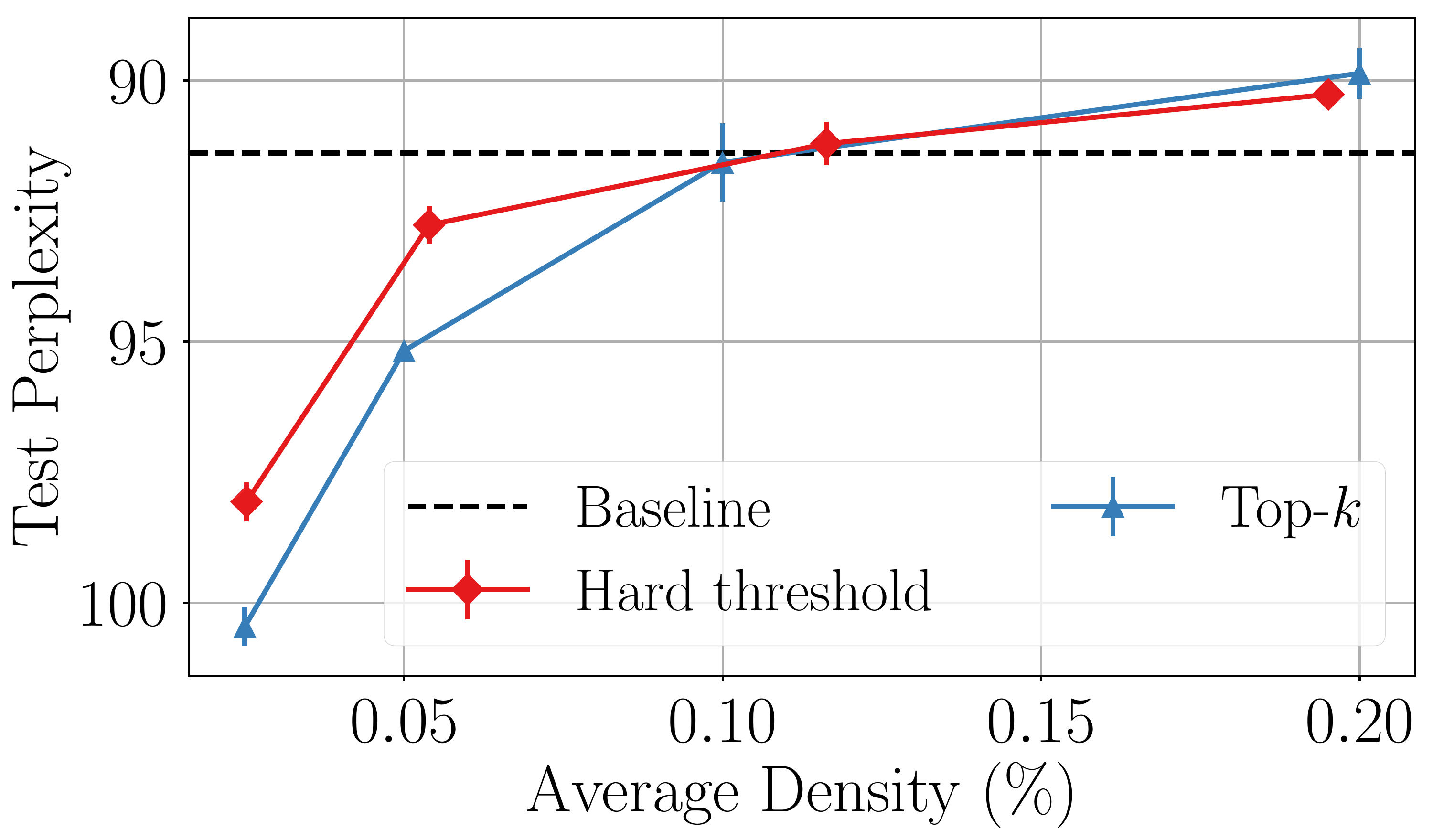}
		\caption{LSTM on Wikitext}
	\end{subfigure}
\begin{subfigure}{0.32\textwidth}
		\includegraphics[width=\textwidth]{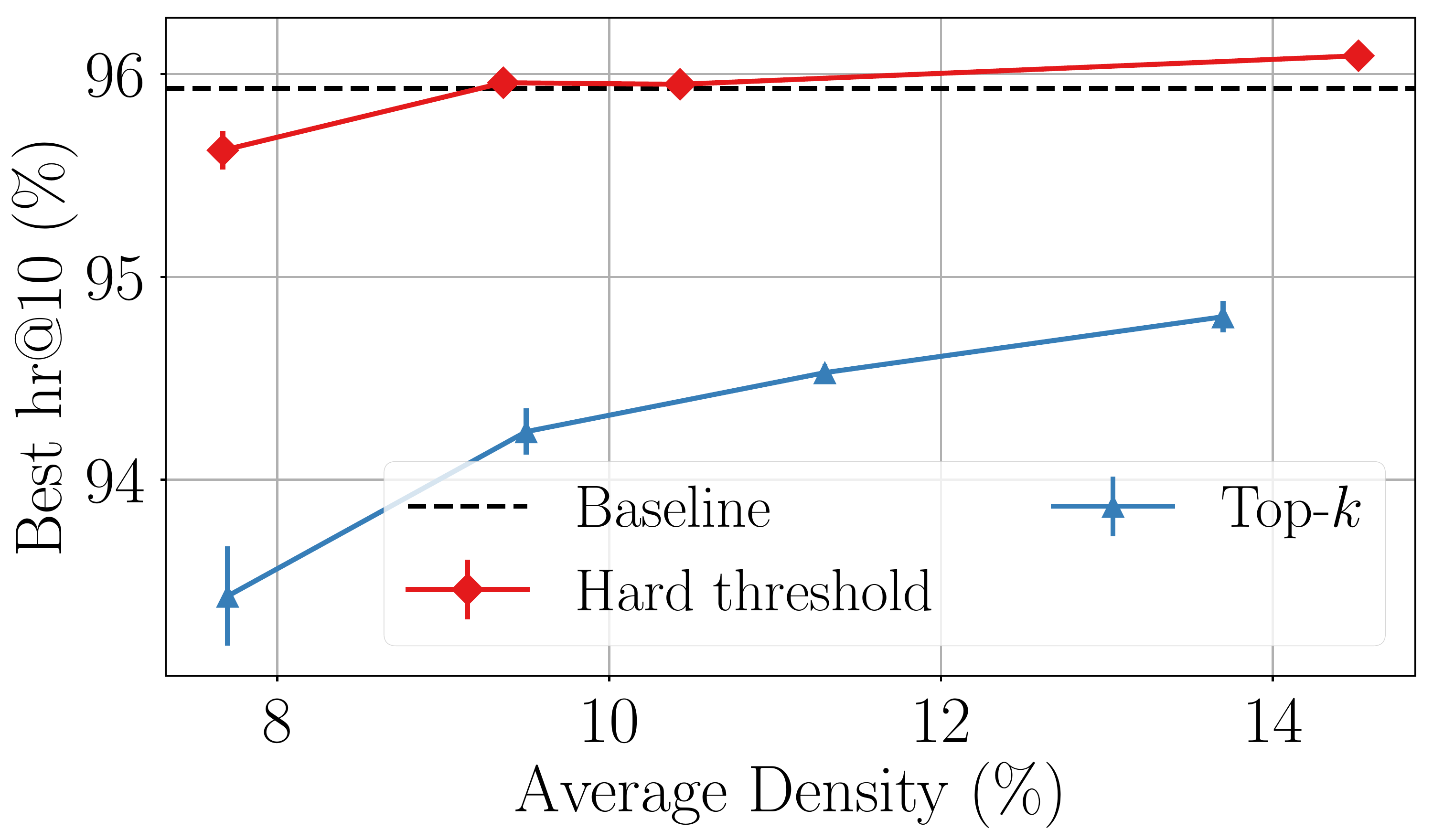}
		\caption{NCF on Movielens-20M}
	\end{subfigure}
\vspace{-1pt}
\caption{\small{\textbf{Test metric vs. Data Volume.} For 3 benchmarks, average test quality with std. dev. over 3 runs. The dashed black line denotes the no compression baseline.}}\label{fig:accuracy vs data-volume}
\end{figure*}

\begin{figure*}[!t]
\vspace{-1em}
\centering
\begin{subfigure}{0.32\linewidth}
		\includegraphics[width=\textwidth]{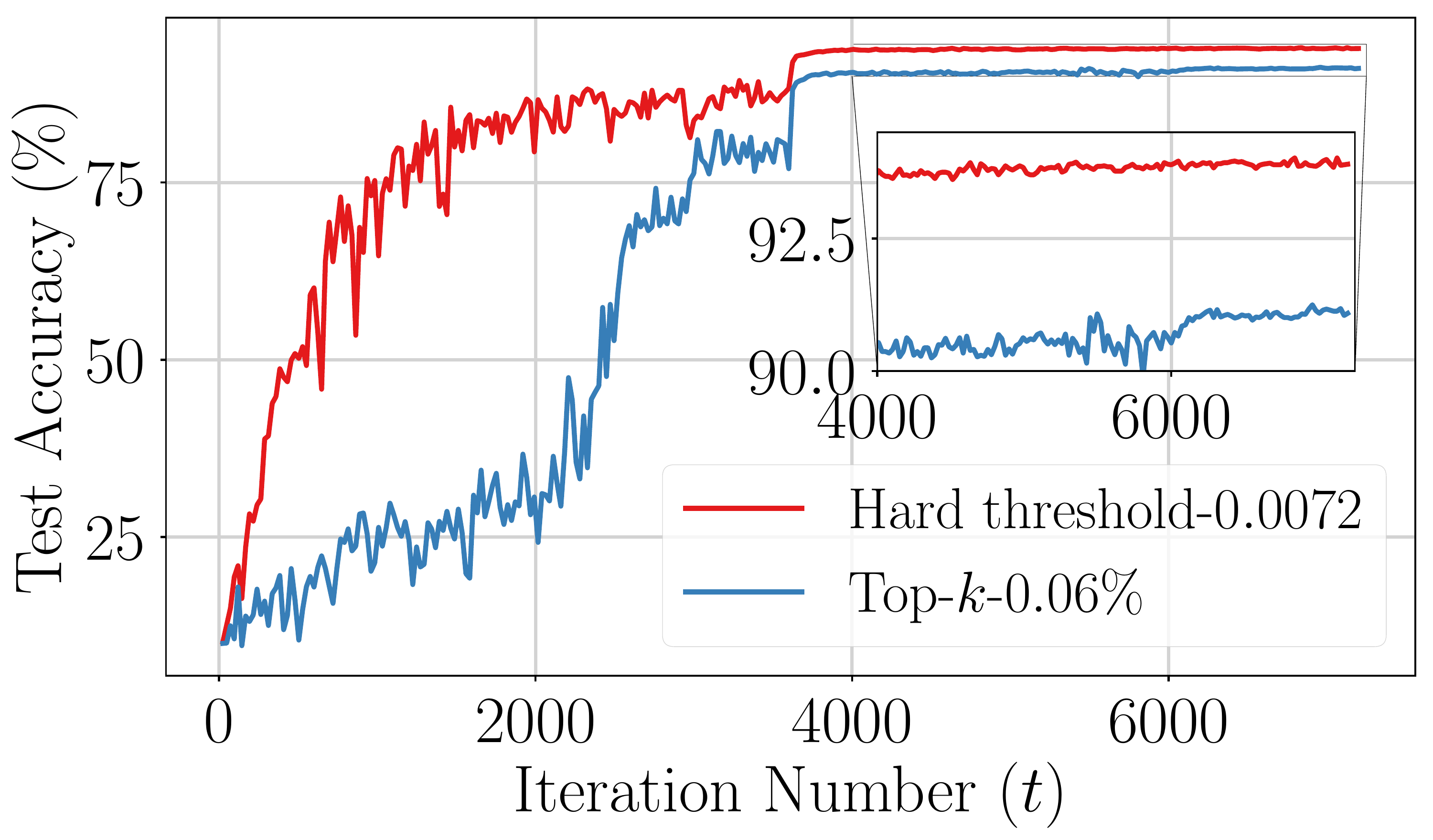}
				\caption{}\label{accuracy}
\end{subfigure}
\begin{subfigure}{0.32\linewidth}
		\includegraphics[width=\textwidth]{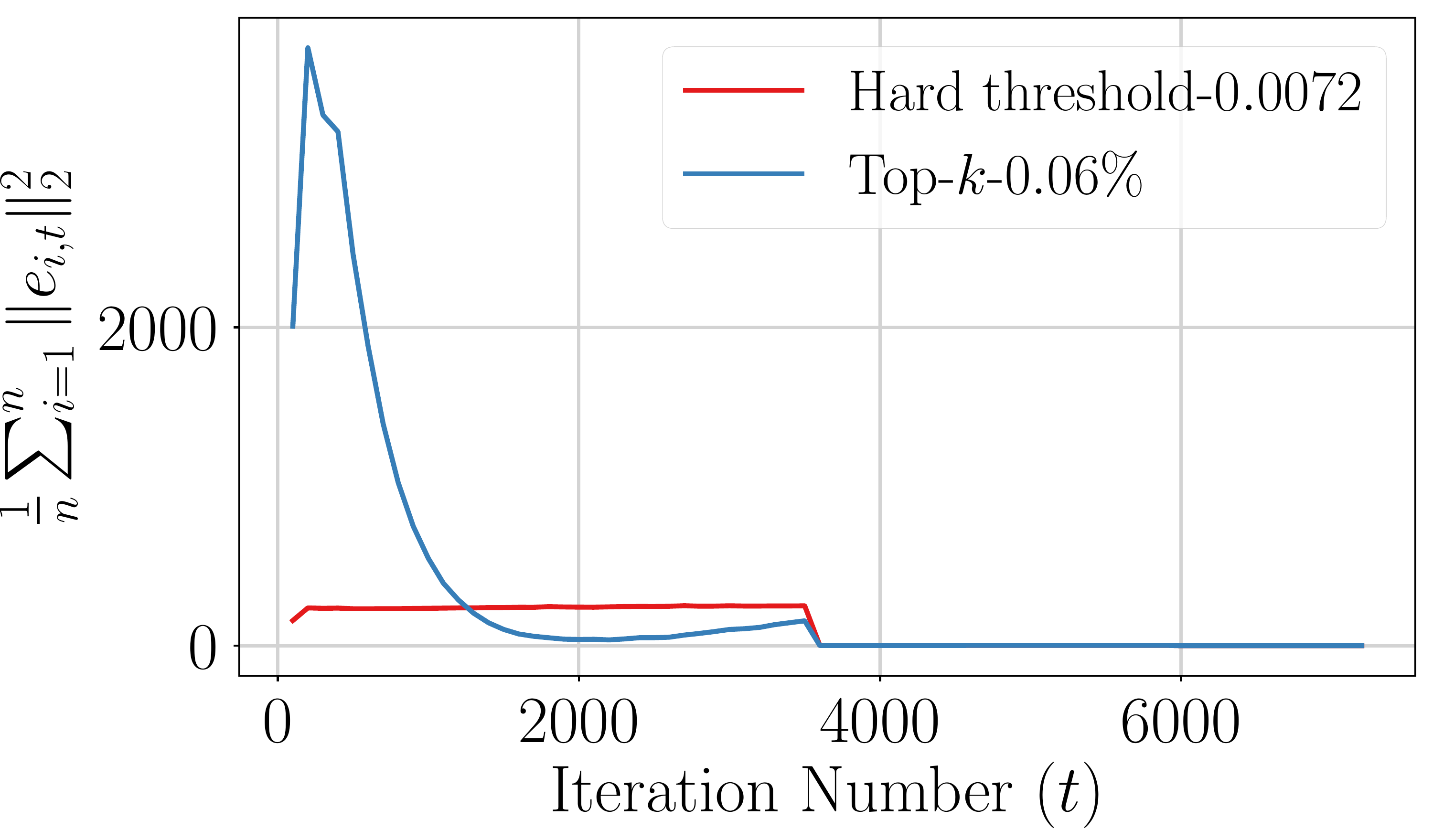}
			\caption{}\label{lr_scaled_mem}
\end{subfigure}
\begin{subfigure}{0.32\linewidth}
		\includegraphics[width=\textwidth]{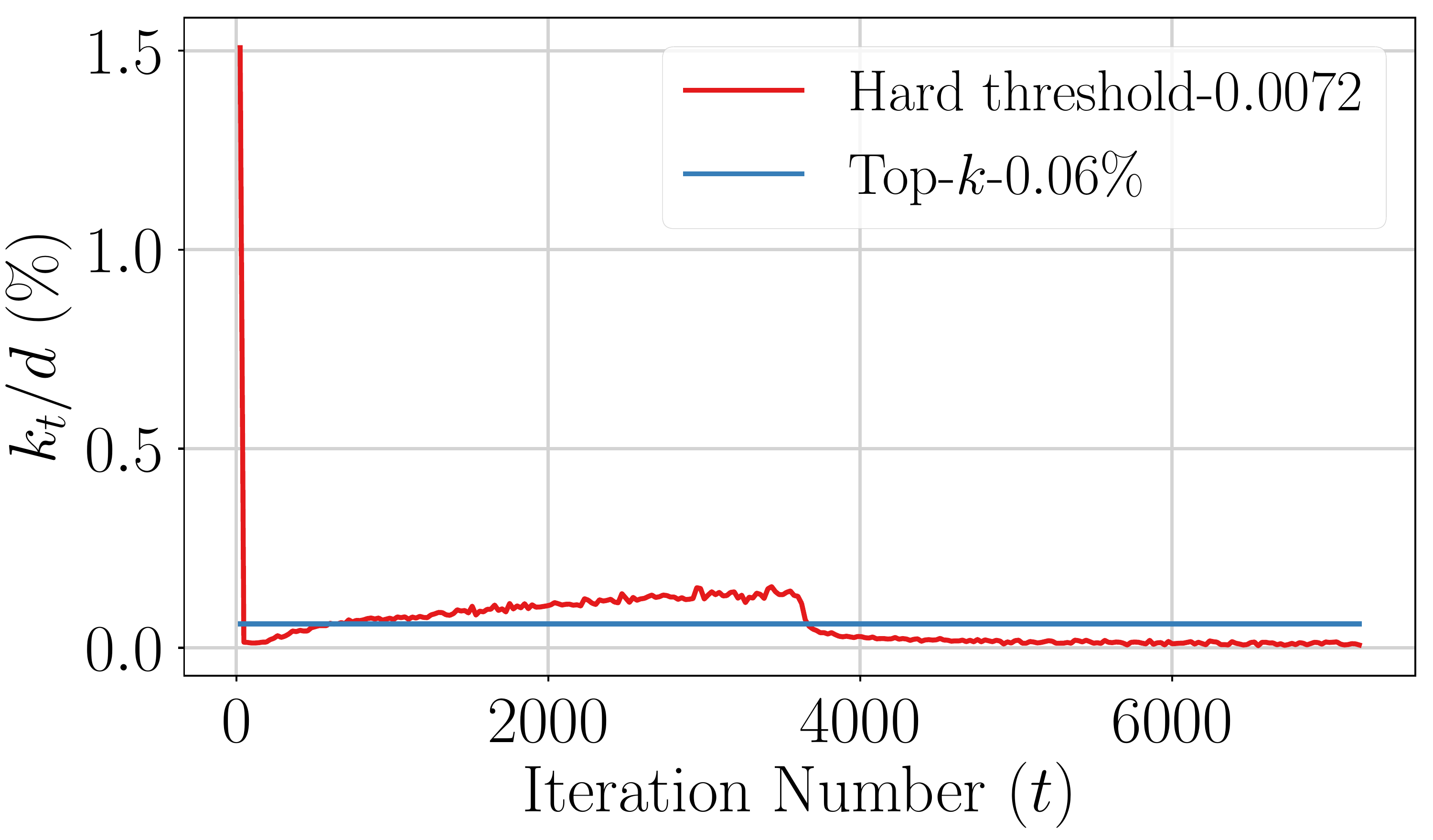}
		\caption{}\label{sparse}
\end{subfigure}
	\vspace{-5pt}
\caption{\small{\textbf{Convergence of Top-$k$ and Hard-threshold for ResNet-18 on CIFAR-10 at $\mathbf{0.06\%}$ average density:} ~(a)~Test-accuracy vs. Iterations,~(b)~Error-norm vs. Iterations, (c) Density ($k_t/d$) vs. Iterations. $k=0.06\%$ of $d$, and $\lambda$ = 0.0072. Hard-threshold has better convergence than Top-$k$ because of a smaller total-error.} }\label{fig:large_error_Topk_resnet}
\vspace{-1em}
\end{figure*}

\section{Experiments}\label{sec:experiments}

\smartparagraph{Experimental setup.}
We compare Top-$k$ and hard-threshold sparsifiers on image classification, language modelling, and recommendation tasks. We use different optimizers: vanilla SGD, SGD with Nesterov momentum, and ADAM \cite{adam}. All experiments were run on an 8-GPU cluster, using {\tt Allgather} as the communication primitive. We perform compression in the standard layer-wise fashion \cite{DBLP:conf/interspeech/SeideFDLY14, lin2018deep, layer-wise} and follow the EF strategy used in \cite{PowerSGD}.~For hyper-parameter configuration, comparison with entire-model compression, discussion on different EF approaches, experiments without EF, and experiments with logistic regression, we refer to Appendix~\ref{apdx:add_exps}.

\smartparagraph{Test metric vs. Data volume.} We tune the sparsification parameters for both sparsifiers such that they send similar total data volumes during training. We use \emph{average density}: $\frac{1}{T}\sum_{t=0}^{T-1}\frac{k_t}{d}$ as a measure of total data volume, where $k_t$ denotes the number of elements transmitted in iteration $t$.
Figure \ref{fig:accuracy vs data-volume} shows the average test quality across three repetitions with different initial random seeds. We observe that fixing the average density, \emph{hard-threshold consistently has better test performance} than Top-$k$. For ResNet-18 on CIFAR-10, we observe that hard-threshold at an average density of $0.12\%$ almost achieves the baseline accuracy and is better than Top-$k$ at $0.75\%$ density ($\sim 6\times$ more total data volume). For LSTM on Wikitext, at an average density of $0.025\%$, hard-threshold has $>2$ better perplexity than Top-$k$. For NCF on Movielens-20M, hard-threshold has $>1\%$ better Hit-Rate$@10$ at all considered average densities.

We now demonstrate that hard-threshold has faster convergence because of a smaller total-error in comparison to Top-$k$. In Figure \ref{fig:large_error_Topk_resnet}, we introspect a run with average density of $0.06\%$ from Figure \ref{fig:acc_vs_data_resnet}. In Figure \ref{fig:large_error_Topk_resnet}a, while hard-threshold converges to an accuracy of $93.9\%$, Top-$k$ achieves $91.1\%$ accuracy. At the same time, in Figure \ref{fig:large_error_Topk_resnet}b, we observe large error-accumulation in the initial $1,200$ iterations for Top-$k$. Consequently, hard-threshold has a significantly lower total-error than Top-$k$, and therefore has better convergence. This observation about large error accumulation for Top-$k$ is consistent across all our benchmarks (see \S\ref{sec:large_error_topk}).

\smartparagraph{Comparison against ACCORDION \cite{agarwal2020accordion}.}
We compare against the state-of-the-art adaptive sparsifier: ACCORDION \cite{agarwal2020accordion}. ACCORDION shifts between two user-defined $k$ values: $k_{\max}$ and $k_{\min}$, by using Top-$k_{\max}$ when the training is in a \emph{critical regime}, else using Top-$k_{\min}$. We compare against ACCORDION with threshold $\lambda=\frac{1}{2\sqrt{k_{\min}}}$. With this hard-threshold, we observe improved and more communication-efficient performance over ACCORDION (e.g., for CIFAR-10 dataset, up to $0.8 \%$ better accuracy with $4 \times$ fewer communicated bits). For these results see \S\ref{sec:accordion}.  

\smartparagraph{How to tune the hard-threshold?} 
Based on our results in \S\ref{sec:non_convex_convergence}, we suggest setting the threshold as $\lambda \sim \frac{1}{2\sqrt{k}}$ to achieve better convergence under similar total-data communication as Top-$k$ with parameter $k$ (see further discussion in \S\ref{sec:tuning_hard_threshold}). But it remains an open question how to tune the hard-threshold such that it achieves no-compression baseline performance with the least total-data transmission. We note that, as of now, this question remains unanswered for Top-$k$ as well.

\section{Conclusion}

We proposed a total-error perspective to compressed communication that captures the effect of compression during the entire training process. Under this, we showed that the hard-threshold sparsifier is more communication-efficient than the state-of-the-art Top-$k$ sparsifier. Our convex convergence result for absolute compressors, the class of compressors in which hard-threshold belongs, is the first to achieve the linear speed-up property of SGD by using the EF framework. As the EF framework is also applicable to Local SGD, we hope that this inspires more communication-efficient versions of Local SGD that adaptively determine when to communicate, 
rather than naively communicating in fixed intervals. 
Furthermore, similar to hard-threshold, we believe adaptive absolute compressor counterparts of quantization schemes and low-rank methods can also be developed.

\bibliographystyle{plain}
\bibliography{references.bib}

\clearpage

\newpage

\appendix

\tableofcontents

\section{Notations} In this paper, by $[d]$ we denote the set of $d$ natural numbers $\{1,2,\cdots, d\}$. We denote the $\ell_2$ norm of a vector $x\in\R^d$ by $\|x\|$, and the $\ell_1$ and $\ell_\infty$-norms are denoted by $\|x\|_1$ and $\|x\|_\infty$, respectively. By $\mathbf{0}$ we denote a vector of all 0s in $\R^d$. In the proofs, we use the notation $\mbE_t[\cdot]$ to denote expectation conditioned on the iterate, $x_t$, that is, $\mbE[\cdot|x_t]$.

\section{Convergence analysis}\label{sec:analysis}
In this section, we provide the proofs of convex and non-convex convergence results of the absolute compressors with EF, and compare them with that of the $\delta$-contraction operators, and vanilla SGD.
\subsection{Overview of results}
In \S\ref{sec:convergence_technical_results}, we provide the technical lemmas and inequalities necessary for the analyses. In \S\ref{apdx:non_convex_convergence} we provide the non-convex convergence results, and \S\ref{sec:abs-quasi} contains the convex convergence results.
\subsection{Technical results}\label{sec:convergence_technical_results}

\begin{lemma}
If $a,b\in\mathbb{R}^d$ then the Young's inequality is: For all $\rho>0$, we have
\begin{equation}\label{eq:youngs}
    \norm{a+b}^2 \leq (1+\rho)\norm{a}^2 + (1+\rho^{-1})\norm{b}^2.
\end{equation}
Alternatively,
\begin{equation}\label{eq:youngs2}
    2\dotprod{a,b} \leq \rho\norm{a}^2 + \rho^{-1}\norm{b}^2.
\end{equation}
\end{lemma}
\begin{lemma}\label{lem:akds}
For $a_i\in\mathbb{R}^d$ we have:
\begin{equation}\label{eq:y}
    \norm{\frac{1}{n}\sum_{i=1}^n a_i}^2 \leq \frac{1}{n} \sum_{i=1}^n \norm{a_i}^2.
\end{equation}
\end{lemma}
\begin{lemma}\label{lem:order_bounds} 
\cite{stich2019error} Let $r_0\geq0$, $c\geq0$, $d>0$, $T>0$ and $0<\gamma\leq\frac{1}{d}$. Then, it holds
\begin{equation*}
    \frac{r_0}{\gamma T} + c\gamma \leq \frac{dr_0}{T} + \frac{2\sqrt{cr_0}}{\sqrt{T}}
\end{equation*}
\end{lemma}
\begin{proof}
We consider two cases. If $\frac{r_0}{cT}\leq\frac{1}{d^2}$, then choosing the step-size $\gamma=\left(\frac{r_0}{cT}\right)^{1/2}$, we get
\begin{equation*}
    \frac{r_0}{\gamma T} + c\gamma \leq \frac{2\sqrt{cr_0}}{\sqrt{T}}.
\end{equation*}
Else, if $\frac{r_0}{cT}>\frac{1}{d^2}$, then choosing $\gamma=\frac{1}{d}$, we get
\begin{equation*}
     \frac{r_0}{\gamma T} + c\gamma \leq \frac{dr_0}{T} + \frac{c}{d}\leq\frac{dr_0}{T}+\frac{\sqrt{cr_0}}{\sqrt{T}}.
\end{equation*}
Combining both bounds, we get the result.
\end{proof}

\begin{lemma}\label{lem:order_bounds_cubic} 
Let $r_0\geq0$, $c\geq0$, $b\geq0$, $d>0$, $T>0$ and $0<\gamma\leq\frac{1}{d}$. Then, it holds
\begin{equation*}
    \frac{r_0}{\gamma T} + c\gamma + b\gamma^2 \leq \frac{dr_0}{T}+\frac{2\sqrt{cr_0}}{\sqrt{T}}+\frac{br_0}{cT}.
\end{equation*}
\end{lemma}

\begin{proof}
The proof follows similar to Lemma \ref{lem:order_bounds}. We consider two cases. If $\frac{r_0}{cT}\leq\frac{1}{d^2}$, then choosing the step-size $\gamma=\left(\frac{r_0}{cT}\right)^{1/2}$, we get
\begin{equation*}
    \frac{r_0}{\gamma T} + c\gamma + b\gamma^2\leq \frac{2\sqrt{cr_0}}{\sqrt{T}}+\frac{br_0}{cT}.
\end{equation*}
Else, if $\frac{r_0}{cT}>\frac{1}{d^2}$, then choosing $\gamma=\frac{1}{d}$, we get
\begin{equation*}
     \frac{r_0}{\gamma T} + c\gamma +b\gamma^2\leq \frac{dr_0}{T} + \frac{c}{d}+ \frac{b}{d^2}\leq\frac{dr_0}{T}+\frac{\sqrt{cr_0}}{\sqrt{T}}+\frac{br_0}{cT}.
\end{equation*}
Combining both bounds, we get the result.
\end{proof}

\begin{lemma}\label{lem:sc_order_bounds_dec_lr}
For every non-negative sequence $\{r_t\}_{t\geq0}$ and parameters $a>0, c\geq0, b\geq0$, $T\geq2$,$\phi\geq1$, decreasing setp-sizes $\{\gamma_t \eqdef \frac{2}{a(\phi + t)}\}_{t\geq0}$, and weights $\{w_t\eqdef(\phi+t)\}_{t\geq0}$, satisfy
\begin{eqnarray*}
\Psi_T \eqdef \frac{1}{W_T}\sum_{t=0}^{T}\left(\frac{w_t}{\gamma_t}(1-a\gamma_t)r_t - \frac{w_t}{\gamma_t}r_{t+1} + c\gamma_t w_t + b\gamma_t^2 w_t \right) \leq  \frac{4c}{a T}+\frac{a\phi^2r_0}{T^2}+\frac{16b\ln(T)}{a^2T^2},
\end{eqnarray*}

\end{lemma}
where $W_T\eqdef\sum_{t=0}^Tw_t$.

\begin{proof}
This proof is motivated from Lemma 11 in \cite{stich2019error}. We observe
\begin{equation}\label{eq:laiss}
    \frac{w_t}{\gamma_t}(1-a\gamma_t)r_t = \frac{a}{2}(\phi+t)(\phi+t-2)r_t = \frac{a}{2}((\phi+t-1)^2-1)r_t \leq \frac{a}{2}(\phi+t-1)^2r_t.
\end{equation}
By plugging in the definition of $\gamma_t$ and $w_t$ in $\Psi_t$, we find
\begin{eqnarray*}
    \Psi_T &\overset{\eqref{eq:laiss}}{\leq}& \frac{1}{W_T}\sum_{t=0}^T  \left(\frac{a}{2}(\phi+t-1)^2r_t - \frac{a}{2}(\phi+t)^2r_{t+1}\right)+ \sum_{t=0}^T \frac{2c}{aW_T} + \sum_{t=0}^T\frac{4b}{a^2(\phi+t)W_T}\\
    &\leq& \frac{a(\phi-1)^2r_0}{2W_T}+\frac{2c(T+1)}{aW_T}+ \frac{4b}{a^2W_T}\sum_{t=0}^T\frac{1}{\phi+t}.
\end{eqnarray*}
By using $(\phi-1)^2\leq\phi^2$, $W_T=\sum_{t=0}^T(\phi+t)\geq\frac{(2\phi+T)(T+1)}{2}\geq\frac{(T+1)(T+2)}{2}$, and $\sum_{t=0}^T\frac{1}{\phi+t} \leq \sum_{t=0}^T\frac{1}{1+t}\leq \ln(T+1)+1$, we have
\begin{eqnarray*}
\Psi_T \leq \frac{a\phi^2r_0}{(T+1)(T+2)}+\frac{4c}{a(T+2)}+ \frac{8b(\ln(T+1)+1)}{a^2(T+1)(T+2)}.
\end{eqnarray*}
For $T\geq2$, we have $\frac{(\ln(T+1)+1)}{(T+1)(T+2)}\leq\frac{2\ln(T)}{T^2}$.~By using this, we get
\begin{eqnarray*}
\Psi_T \leq \frac{a\phi^2r_0}{T^2}+\frac{4c}{aT}+ \frac{16b\ln(T)}{a^2T^2}.
\end{eqnarray*}
Hence the result. 
\end{proof}

\begin{lemma}\label{lem:sc_order_bounds_const_lr}
For every non-negative sequence $\{r_t\}_{t\geq0}$ and parameters $d\geq a > 0$, $c\geq0$, $b\geq0$, $T\geq0$, with a bound on the setp-size $\gamma_t\leq\frac{1}{d}$, there exists a constant setp-size, $\{\gamma_t=\gamma\}_{t\geq0}$  and weights, $w_t \eqdef (1-a\gamma)^{-(t+1)}$, such that
\begin{eqnarray*}
\textstyle\Psi_T \eqdef \frac{1}{W_T}\sum_{t=0}^{T}\left(\frac{w_t}{\gamma_t}(1-a\gamma_t)r_t - \frac{w_t}{\gamma_t}r_{t+1} + c\gamma_t w_t + b \gamma_t^2 w_t\right) = \tilde{\cO}\left(dr_0\exp\left[-\frac{aT}{d}\right]+\frac{c}{aT}+\frac{b}{a^2T^2}\right).
\end{eqnarray*}
\end{lemma}

\begin{proof}
This proof is motivated from Lemma 12 in \cite{stich2019error}.
Substituting the values for $\gamma_t$ and $w_t$, we get
\begin{eqnarray}\label{eq:kjdi}
    \Psi_T &=& \frac{1}{\gamma W_T}\sum_{t=0}^T (w_{t-1}r_t-w_tr_{t+1})+ \frac{c\gamma}{W_T}\sum_{t=0}^Tw_t + \frac{b\gamma^2}{W_T}\sum_{t=0}^Tw_t\nonumber\\
    &\leq&\frac{r_0}{\gamma W_T}+c\gamma + b\gamma^2\nonumber\\
    &\leq&\frac{r_0}{\gamma}\exp[-a\gamma T] + c\gamma + b\gamma^2,
\end{eqnarray}
where we use $W_T\geq w_T \geq (1-a\gamma)^{-T}\geq \exp[a\gamma T]$ in the last inequality. To tune $\gamma$, we consider following two cases:\\
\begin{inparaitem}
    \item If $\frac{1}{d}\geq\frac{\ln(\max\{2,a^2r_0T^2/c\})}{aT},$ then we choose $\gamma=\frac{\ln(\max\{2,a^2r_0T^2/c\})}{aT}$ and \eqref{eq:kjdi} becomes $\tilde{\cO}(\frac{c}{aT}+\frac{b}{a^2 T^2}),$
    as 
    $ar_0T\leq \frac{2c}{aT}$.\\
    \item If $\frac{1}{d}<\frac{\ln(\max\{2,a^2r_0T^2/c\})}{aT},$ then we choose $\gamma=\frac{1}{d}$ and \eqref{eq:kjdi} is $\tilde{\cO}(\frac{c}{aT}+\frac{b}{a^2 T^2}).$
\end{inparaitem}
Combining both bounds, we get the result.
\end{proof}

The recurrence relation in the next lemma is instrumental for perturbed iterate analysis of Algorithm \ref{algo:ef-sgd} used in both convex and non-convex cases. 
\begin{lemma}\label{lem:virtual_iterate}
Let $\bar{e}_t = \frac{1}{n}\sum_{i=1}^n e_{i,t}$, $\bar{g}_{t}=\frac{1}{n}\sum_{i=1}^n g_{i,t}$, and $\bar{p}_t=\frac{1}{n}\sum_{i=1}^n p_{i,t}$.
~Define the sequence of iterates $\{\Tilde{x}_{t}\}_{t \geq 0}$ as $\Tilde{x}_{ t}=x_{t}-\bar{e}_{t}$, with $\Tilde{x}_{0}=x_0$. Then $\{\Tilde{x}_{t}\}_{t \geq 0}$ satisfy the recurrence: $\Tilde{x}_{t+1}=\Tilde{x}_t - \gamma_t \bar{g}_t$. 
\end{lemma}
\begin{proof}
We have\begin{eqnarray*}\label{eq:virtual_iterate}
    \Tilde{x}_{t+1}=x_{t+1} - \bar{e}_{t+1}
    = x_t - (\bar{e}_{t}+\gamma_t \bar{g}_t)
    =\Tilde{x}_t - \gamma_t \bar{g}_t.
\end{eqnarray*}
Hence the result. 
\end{proof}
\subsection{Non-convex convergence analysis}\label{apdx:non_convex_convergence}
In this section, we provide the non-convex convergence analyses. Lemma \ref{lem:descent_lemma} provides a one-step descent recurrence which leads to Theorem \ref{thm:ef-gen} and a key result for proving convergence. Based on this, in  \S\ref{sec:non_convex_abs_comp}, \S\ref{sec:non-convex-delta}, \S\ref{sec:non-convex-uncomp-sgd} we discuss the convergence of absolute compressors, $\delta$-contraction operators, and uncompressed SGD, respectively. In  \S\ref{sec:big_oh} we provide the convergence result for absolute compressors and $\delta$-contraction operators for an appropriate choice of step-size. The following lemma bounds the quantity $\mbE_t \norm{\frac{1}{n}\sum_{i=1}^n g_{i,t}}^2.$
\begin{lemma}
We have
\begin{equation}\label{eq:acbk}
    \mbE_t \norm{\frac{1}{n}\sum_{i=1}^n g_{i,t}}^2 \leq (1+\frac{M(C+1)}{n}) \norm{\nabla f(x_t)}^2 + \frac{M\zeta^2+\sigma^2}{n}.
\end{equation}
\end{lemma}

\begin{proof}
We have
\begin{eqnarray*}
 \mbE_t \norm{\frac{1}{n}\sum_{i=1}^n g_{i,t}}^2 &=& \mbE_t\norm{\frac{1}{n}\sum_{i=1}^n (\nabla f_i(x_t)+ \xi_{i,t})}^2\\
&\overset{\mbE[\xi_{i,t}|x_{t}]=0}{=}& \norm{\nabla f(x_t)}^2 + \mbE_t\norm{ \frac{1}{n}\sum_{i=1}^n\xi_{i,t}}^2\\
&\overset{\mbE[\xi_{i,t}|x_{t}]=0}{=}& \norm{\nabla f(x_t)}^2 + \frac{1}{n^2}\sum_{i=1}^n\mbE_t\norm{ \xi_{i,t}}^2\\
&\overset{{\rm By\;Assumption\;} \ref{ass:m-sigma2-bounded_noise}}{\leq}& \norm{\nabla f(x_t)}^2 + \frac{1}{n^2}\sum_{i=1}^n (M\norm{\nabla f_i(x_{t})}^2+\sigma^2)\\
&=& \norm{\nabla f(x_t)}^2 + \frac{M}{n^2}\sum_{i=1}^n \norm{\nabla f_i(x_{t})-\nabla f(x_t)}^2 + \frac{M\norm{\nabla f(x_t)}^2}{n} +\frac{\sigma^2}{n}\\
&\overset{{\rm By\;Assumption\;} \ref{ass:bounded_similarity}}{\leq}& (1+\frac{M}{n}) \norm{\nabla f(x_t)}^2 + \frac{M}{n}(C\norm{\nabla f(x_t)}^2 + \zeta^2) + \frac{\sigma^2}{n}.
\end{eqnarray*}
By rearranging the terms we get the result. 
\end{proof}
The following non-convex descent lemma is the key result used to establish convergence of both absolute compressors and $\delta$-contraction operators. 
\begin{lemma}\label{lem:descent_lemma}
\textbf{(Non-convex descent lemma)} Let Assumptions \ref{ass:smoothness}, \ref{ass:m-sigma2-bounded_noise}, and \ref{ass:bounded_similarity} hold.~If $\{x_t\}_{t\geq0}$ denote the iterates of Algorithm \ref{algo:ef-sgd} for a constant setp-size, $\gamma \leq \cfrac{n}{2L(M(C+1)+n)}$ then
\begin{equation}\label{eq:descent_lemma}
\mbE[f(\tilde{x}_{t+1})]] \leq \mbE[f(\tilde{x}_t)] - \frac{\gamma}{4}\mbE\norm{\nabla f(x_t)}^2 + \frac{\gamma^2L(M\zeta^2+\sigma^2)}{2n} + \frac{\gamma L^2}{2n}\sum_{i=1}^n\mbE\norm{e_{i,t}}^2.
\end{equation}
\end{lemma}
\begin{proof}
By using the $L$-smoothness of $f$ and taking expectation we have
\begin{eqnarray*}
    \mbE_t[f(\tilde{x}_{t+1})] &\leq& f(\tilde{x}_t) - \dotprod{\nabla f(\tilde{x}_t), \mbE_t[\tilde{x}_{t+1}-\tilde{x}_t]} + \frac{L}{2}\mbE_t\norm{\tilde{x}_{t+1}-\tilde{x}_t}^2\\
     &=& f(\tilde{x}_t) - \gamma\dotprod{\nabla f(\tilde{x}_t), \nabla f(x_t)} + \frac{\gamma^2L}{2} \mbE_t\norm{\frac{1}{n}\sum_{i=1}^n g_{i,t}}^2\\
    &\overset{\eqref{eq:acbk}}{\leq}& f(\tilde{x}_t) - \gamma\dotprod{\nabla f(\tilde{x}_t), \nabla f(x_t)}\\&& + \frac{\gamma^2L}{2} \left((1+\frac{M(C+1)}{n}) \norm{\nabla f(x_t)}^2 + \frac{M\zeta^2}{n} + \frac{\sigma^2}{n}\right)\\
    &\leq& f(\tilde{x}_t) - \gamma \norm{\nabla f(x_t)}^2 + \gamma \dotprod{\nabla f(x_t)-\nabla f(\tilde{x}_t), \nabla f(x_t)}\\&\quad&+\frac{\gamma^2L(M(C+1)+n)}{2n}\norm{\nabla f(x_t)}^2+ \frac{\gamma^2L(M\zeta^2+\sigma^2)}{2n}\\
    &\overset{\eqref{eq:youngs2}}{\leq}& f(\tilde{x}_t) - (\gamma-\frac{\gamma}{2}-\frac{\gamma^2L(M(C+1)+n)}{2n})\norm{\nabla f(x_t)}^2 + \\&&\frac{\gamma \norm{\nabla f(x_t)-\nabla f(\tilde{x}_t)}^2}{2} +\frac{\gamma^2L(M\zeta^2+\sigma^2)}{2n}\\
    &\overset{\substack{{\rm By\;}L-{\rm smoothness}\\ {\rm and\;}\gamma \leq \frac{n}{2L(M(C+1)+n)}}}{\leq}& f(\tilde{x}_t) - \frac{\gamma\norm{\nabla f(x_t)}^2}{4} + \frac{\gamma L^2\norm{x_t-\tilde{x_t}}^2}{2} +\frac{\gamma^2L(M\zeta^2+\sigma^2)}{2n}\\ 
     &=& f(\tilde{x}_t) - \frac{\gamma\norm{\nabla f(x_t)}^2}{4} + \frac{\gamma L^2\norm{\bar{e}_t}^2}{2} +\frac{\gamma^2L(M\zeta^2+\sigma^2)}{2n}\\
     &\overset{\eqref{eq:y}}{\leq}& f(\tilde{x}_t) - \frac{\gamma\norm{\nabla f(x_t)}^2}{4} + \frac{\gamma L^2\frac{1}{n}\sum_{i=1}^n\norm{e_{i,t}}^2}{2} +\frac{\gamma^2L(M\zeta^2+\sigma^2)}{2n}.
\end{eqnarray*}
Taking total expectation yields the lemma.
\end{proof}

\begin{remark}
Rearranging the terms in Lemma \ref{lem:descent_lemma}, performing telescopic sum, and noting that $\zeta=0$ for $n=1$, we get the result in Theorem \ref{thm:ef-gen}.
\end{remark}
\subsubsection{Absolute compressors}\label{sec:non_convex_abs_comp}
\begin{theorem*}\ref{thm:non-convex-wo-bg}
\textbf{(Non-convex convergence of absolute compressors)} Let Assumptions \ref{ass:smoothness}, \ref{ass:minimum}, \ref{ass:m-sigma2-bounded_noise}, and \ref{ass:bounded_similarity} hold.~Then the iterates, $\{x_t\}_{t\geq0}$ of Algorithm \ref{algo:ef-sgd} with an absolute compressor, $\cC$ and a constant step-size, $\gamma \leq \frac{n}{2L(M(C+1)+n)}$, follow
\begin{equation*}
   \frac{1}{T}\sum_{t=0}^{T-1} \mbE \norm{\nabla f(x_t)}^2 \leq \frac{4(f(x_0)-f^{\star})}{\gamma T} +\frac{2\gamma L(M\zeta^2+\sigma^2)}{n} + 2\gamma^2 L^2\kappa^2.
\end{equation*}
\end{theorem*}

\begin{proof}
By using Lemma \ref{lem:descent_lemma}, we have
\begin{eqnarray*}
    \mbE[f(\tilde{x}_{t+1})]
     &\leq& \mbE [f(\tilde{x}_t)] - \frac{\gamma\mbE\norm{\nabla f(x_t)}^2}{4} + \frac{\gamma L^2\frac{1}{n}\sum_{i=1}^n\mbE\norm{e_{i,t}}^2}{2} +\frac{\gamma^2L(M\zeta^2+\sigma^2)}{2n}\\
    &\overset{\text{Remark}\; \ref{rem:bounded_memory}}{\leq}& \mbE [f(\tilde{x}_t)] - \frac{\gamma\mbE\norm{\nabla f(x_t)}^2}{4} + \frac{\gamma^3 L^2\kappa^2}{2} +\frac{\gamma^2L(M\zeta^2+\sigma^2)}{2n}.
\end{eqnarray*}
By taking summation over the iterates,  we get
\begin{eqnarray*}
   \frac{1}{T}\sum_{t=0}^{T-1} \mbE \norm{\nabla f(x_t)}^2 &\leq& \frac{4\sum_{t=0}^{T-1} (\mbE [f(\Tilde{x}_t)]-\mbE [f(\Tilde{x}_{t+1})])}{\gamma T} +\frac{2\gamma L(M\zeta^2+\sigma^2)}{n} + 2\gamma^2 L^2\kappa^2 \\
    &\leq& \frac{4(f(x_0)-f^{\star})}{\gamma T} +\frac{2\gamma L(M\zeta^2+\sigma^2)}{n} + 2\gamma^2 L^2\kappa^2.
\end{eqnarray*}
Hence the result. 
\end{proof}
\subsubsection{$\delta$-contraction operators}\label{sec:non-convex-delta}
We now provide an error-bound for $\delta$-contraction operators, which is an extension of the single node case in \cite{stich2019error}.

\begin{lemma}
Define $e_{i,t}$ as in Algorithm \ref{algo:ef-sgd}. Then by using a $\delta$-compressor, $\cC$ with a constant step-size,  $\gamma\leq\frac{1}{2L(2/\delta+M)\sqrt{C+1}}$ we have
\begin{equation}\label{eq:error_bound}
     \sum_{t=0}^{T} \left[\frac{1}{n} \sum_{i=1}^n \mbE \norm{e_{i,t}}^2 \right] \leq \frac{1}{4L^2}\sum_{t=0}^T\mbE \norm{\nabla f(x_t)}^2 + \frac{2\gamma^2(T+1)}{\delta}\left(\left(\frac{2}{\delta}+M\right)\zeta^2+\sigma^2\right).
\end{equation}
\end{lemma}
\begin{proof}
We have
\begin{eqnarray*}
\frac{1}{n} \sum_{i=1}^n \mbE_{t} \norm{e_{i,t+1}}^2 &=& \frac{1}{n} \sum_{i=1}^n \mbE_{\xi_{i,t}}\left[\mbE_{\cC}\norm{e_{i,t}+\gamma g_{i,t}-\gamma\cC(\frac{e_{i,t}}{\gamma}+ g_{i,t})}^2|x_{t}\right]\\
&\overset{{\rm By\;}\eqref{eq:quantization}}{\leq}& \frac{1}{n} \sum_{i=1}^n\gamma^2(1-\delta)\mbE_{\xi_{i,t}}\left[\norm{\frac{e_{i,t}}{\gamma}+ g_{i,t}}^2|x_{t}\right]\\
&\overset{\mbE[\xi_{i,t}|x_{t}]=0}{=}& \frac{(1-\delta)}{n} \sum_{i=1}^n\norm{e_{i,t}+\gamma\nabla f_i(x_{t})}^2+\frac{(1-\delta)}{n} \sum_{i=1}^n\gamma^2\mbE_{\xi_{i,t}}\left[\norm{\xi_{i,t}}^2|x_{t}\right]\\
&\overset{\text{Assumption\;}\ref{ass:m-sigma2-bounded_noise}}{\leq}& \frac{(1-\delta)}{n} \sum_{i=1}^n\norm{e_{i,t}+\gamma\nabla f_i(x_{t})}^2 + \frac{(1-\delta)\gamma^2}{n}\sum_{i=1}^n \left(M\norm{\nabla f_i(x_{t})}^2+\sigma^2\right)\\
&\overset{\eqref{eq:youngs}}{\leq}& \frac{(1-\delta)(1+\rho)}{n} \sum_{i=1}^n\norm{e_{i,t}}^2 + \frac{(1-\delta)(1+\rho^{-1}+M)\gamma^2}{n}\sum_{i=1}^n\norm{\nabla f_i(x_{t})}^2 \\ &&+ (1-\delta)\gamma^2\sigma^2\\
&\overset{\text{Assumption\;}\ref{ass:bounded_similarity}}{\leq}& \frac{(1-\delta)(1+\rho)}{n} \sum_{i=1}^n\norm{e_{i,t}}^2 + \left((1-\delta)(1+\rho^{-1}+M)\gamma^2(C+1)\right)\norm{\nabla f(x_{t})}^2 \\
&& + \left((1-\delta)(1+\rho^{-1}+M)\gamma^2\zeta^2\right)+(1-\delta)\gamma^2\sigma^2\\
&\leq& \frac{(1-\delta)(1+\rho)}{n} \sum_{i=1}^n\norm{e_{i,t}}^2 \\
&&+ \gamma^2 \left((1+\rho^{-1}+M)(C+1)\norm{\nabla f(x_{t})}^2+(1+\rho^{-1}+M)\zeta^2 + \sigma^2\right).
\end{eqnarray*}
By unrolling the recurrence, taking total expectation, and picking $\rho=\frac{\delta}{2(1-\delta)}$, such that $(1+\rho^{-1})=\frac{2-\delta}{\delta}\leq \frac{2}{\delta}$ and $(1-\delta)(1+\rho)\leq(1-\frac{\delta}{2})$ we find
\begin{eqnarray*}
\frac{1}{n} \sum_{i=1}^n \mbE \norm{e_{i,t+1}}^2 &\leq& \gamma^2\sum_{i=0}^{t} [(1-\delta)(1+\rho)]^{t-i}\left(\left(\frac{2}{\delta}+M\right)(C+1)\mbE \norm{\nabla f(x_i)}^2 + \left(\frac{2}{\delta}+M\right)\zeta^2+\sigma^2\right)\\
&\leq& \gamma^2 \sum_{i=0}^{t} (1-\frac{\delta}{2})^{t-i}\left(\left(\frac{2}{\delta}+M\right)(C+1)\mbE \norm{\nabla f(x_i)}^2 + \left(\frac{2}{\delta}+M\right)\zeta^2+\sigma^2\right).
\end{eqnarray*}
Hence, we have
\begin{eqnarray*}
\sum_{t=0}^{T} \left[\frac{1}{n} \sum_{i=1}^n \mbE \norm{e_{i,t}}^2 \right] &=& \gamma^2\sum_{t=0}^{T}\sum_{i=0}^{t-1} (1-\frac{\delta}{2})^{t-1-i}\left(\left(\frac{2}{\delta}+M\right)(C+1)\mbE \norm{\nabla f(x_i)}^2 + \left(\frac{2}{\delta}+M\right)\zeta^2+\sigma^2\right)\\
&\leq& \gamma^2\sum_{t=0}^{T-1}\sum_{j=0}^{T-t-1} (1-\frac{\delta}{2})^{j}\left(\left(\frac{2}{\delta}+M\right)(C+1)\mbE \norm{\nabla f(x_t)}^2 + \left(\frac{2}{\delta}+M\right)\zeta^2+\sigma^2\right)\\
&\leq& \gamma^2 \sum_{t=0}^{T-1}\left(\left(\frac{2}{\delta}+M\right)(C+1)\mbE \norm{\nabla f(x_t)}^2 + \left(\frac{2}{\delta}+M\right)\zeta^2+\sigma^2\right)  \sum_{j=0}^{\infty}(1-\frac{\delta}{2})^{j}\\
&=& \gamma^2\sum_{t=0}^{T-1}\left(\frac{2}{\delta}\right)\left(\left(\frac{2}{\delta}+M\right)(C+1)\mbE \norm{\nabla f(x_t)}^2 + \left(\frac{2}{\delta}+M\right)\zeta^2+\sigma^2\right) \\
&\leq& \gamma^2\sum_{t=0}^{T}\left(\frac{2}{\delta}\right)\left(\left(\frac{2}{\delta}+M\right)(C+1)\mbE \norm{\nabla f(x_t)}^2 + \left(\frac{2}{\delta}+M\right)\zeta^2+\sigma^2\right) \\
&=& \sum_{t=0}^{T}\left(\gamma^2\left(\frac{2}{\delta}\right)\left(\frac{2}{\delta}+M\right)(C+1)\mbE \norm{\nabla f(x_t)}^2\right) + \sum_{t=0}^{T}\frac{2\gamma^2}{\delta}\left(\left(\frac{2}{\delta}+M\right)\zeta^2+\sigma^2\right).
\end{eqnarray*}
Choosing $\gamma \leq \frac{1}{2L(2/\delta+M)\sqrt{C+1}}$, we get $\gamma^2\left(\frac{2}{\delta}\right)\left(\frac{2}{\delta}+M\right)\leq \frac{1}{4L^2(C+1)}$.~Combining all together we get
\begin{eqnarray*}
    \sum_{t=0}^{T} \left[\frac{1}{n} \sum_{i=1}^n \mbE \norm{e_{i,t}}^2 \right] \leq \frac{1}{4L^2}\sum_{t=0}^T\mbE \norm{\nabla f(x_t)}^2 + \frac{2\gamma^2(T+1)}{\delta}\left(\left(\frac{2}{\delta}+M\right)\zeta^2+\sigma^2\right).
\end{eqnarray*}
Hence the result. 
\end{proof}

By using the previous bound, we now provide the non-convex convergence result for $\delta$-contraction operators.
\begin{theorem*}\ref{thm:non-convex-wo-bg-rel}
\textbf{(Non-convex convergence of $\delta$-contraction operators)} Let Assumptions \ref{ass:smoothness}, \ref{ass:minimum}, \ref{ass:m-sigma2-bounded_noise}, and \ref{ass:bounded_similarity} hold.~Then the iterates, $\{x_t\}_{t\geq0}$ of Algorithm \ref{algo:ef-sgd} with a $\delta$-compressor and a constant step-size $\textstyle{\gamma \leq \min\{\frac{n}{2L(M(C+1)+n)}, \frac{1}{2L(2/\delta+M)\sqrt{C+1}}\}}$ follow \useshortskip
\begin{eqnarray*}
\frac{1}{T}\sum_{t=0}^{T-1} \mbE \norm{\nabla f(x_t)}^2 \leq  \frac{8(f(x_0)-f^{\star})}{\gamma T} + \frac{4 \gamma L(M\zeta^2+\sigma^2)}{n}+\frac{8\gamma^2L^2}{\delta}\left(\left(\frac{2}{\delta}+M\right)\zeta^2+\sigma^2\right).
\end{eqnarray*}
\end{theorem*}
\begin{proof}
Summing over the iterates $t=0$ to $t=T-1$ in  \eqref{eq:descent_lemma} of Lemma \ref{lem:descent_lemma}, we have
\begin{eqnarray*}
\mbE[f(\tilde{x}_{T})]
&\leq& f(x_0) - \frac{\sum_{t=0}^{T-1}\gamma\mbE\norm{\nabla f(x_t)}^2}{4} + \frac{\gamma L^2\sum_{t=0}^{T-1}\frac{1}{n}\sum_{i=1}^n\mbE\norm{e_{i,t}}^2}{2} +\sum_{t=0}^{T-1}\frac{\gamma^2L(M\zeta^2+\sigma^2)}{2n}\\
&\overset{\eqref{eq:error_bound}}{\leq}& f(x_0) - (\frac{\gamma}{4}-\frac{\gamma}{8})\sum_{t=0}^{T-1}\mbE \norm{\nabla f(x_t)}^2 +\frac{\gamma^3L^2T}{\delta}\left(\left(\frac{2}{\delta}+M\right)\zeta^2+\sigma^2\right)+\frac{\gamma^2TL(M\zeta^2+\sigma^2)}{2n}.
\end{eqnarray*}
Rearranging, we get
\begin{eqnarray*}
\frac{1}{T}\sum_{t=0}^{T-1} \mbE \norm{\nabla f(x_t)}^2 &\leq& \frac{8(f(x_0)-\mbE [f(\tilde{x}_t)])}{\gamma T} + \frac{4 \gamma L(M\zeta^2+\sigma^2)}{n}+\frac{8\gamma^2L^2}{\delta}\left(\left(\frac{2}{\delta}+M\right)\zeta^2+\sigma^2\right)\\
&\leq& \frac{8(f(x_0)-f^{\star})}{\gamma T} + \frac{4 \gamma L(M\zeta^2+\sigma^2)}{n}+\frac{8\gamma^2L^2}{\delta}\left(\left(\frac{2}{\delta}+M\right)\zeta^2+\sigma^2\right).
\end{eqnarray*}
Hence the result. 
\end{proof}

\subsubsection{Uncompressed SGD}\label{sec:non-convex-uncomp-sgd}
We provide the convergence result of no-compression SGD (Algorithm \ref{algo:ef-sgd} with an identity compressor, i.e., $\cC(x)=x$ for all $x\in\R^d$). 
\begin{theorem}\label{thm:non-convex-wo-bg-sgd}
\textbf{(Non-convex convergence of SGD)} Let Assumptions \ref{ass:smoothness}, \ref{ass:minimum}, \ref{ass:m-sigma2-bounded_noise}, and \ref{ass:bounded_similarity} hold.~Then the iterates, $\{x_t\}_{t\geq0}$ of Algorithm \ref{algo:ef-sgd} by using an identity compressor ($\cC(x)=x$, for all $x\in\R^d$) with a constant step-size, $\gamma \leq \frac{n}{L(M(C+1)+n)}$ follow
\begin{equation*}
   \frac{1}{T}\sum_{t=0}^{T-1} \mbE \norm{\nabla f(x_t)}^2 \leq \frac{2(f(x_0)-f^{\star})}{\gamma T} +\frac{\gamma L(M\zeta^2+\sigma^2)}{n}.
\end{equation*}
\end{theorem}
\begin{proof}
We use the $L$-smoothness of $f$ to find
\begin{eqnarray*}
    \mbE_t[f(x_{t+1})] &\leq& f(x_t) - \dotprod{\nabla f(x_t), \mbE_t[x_{t+1}-x_t]} + \frac{L}{2}\mbE_t\norm{x_{t+1}-x_t}^2\\
    &=& f(x_t) - \gamma\dotprod{\nabla f(x_t), \mbE_t[\bar{g}_t])} + \frac{\gamma^2L}{2}\mbE_t\norm{\bar{g}_t}^2\\
     &=& f(x_t) - \gamma\norm{ \nabla f(x_t)}^2 + \frac{\gamma^2L}{2} \mbE_t\norm{\frac{1}{n}\sum_{i=1}^n g_{i,t}}^2\\
    &\overset{\eqref{eq:acbk}}{\leq}& f(x_t) - \gamma\norm{ \nabla f(x_t)}^2 + \frac{\gamma^2L}{2} \left((1+\frac{M(C+1)}{n}) \norm{\nabla f(x_t)}^2 + \frac{M\zeta^2}{n} + \frac{\sigma^2}{n}\right)\\
    &=& f(x_t) - \gamma\left(1-\frac{\gamma L(M(C+1)+n)}{2n}\right)\norm{ \nabla f(x_t)}^2  +\frac{\gamma^2L(M\zeta^2+\sigma^2)}{2n}\\
    &\overset{\gamma \leq \frac{n}{L(M(C+1)+n)}}{\leq}& f(x_t) - \frac{\gamma}{2}\norm{ \nabla f(x_t)}^2  +\frac{\gamma^2L(M\zeta^2+\sigma^2)}{2n}. 
\end{eqnarray*}
By summing over the iterates and taking total expectation, we get
\begin{equation*}
   \frac{1}{T}\sum_{t=0}^{T-1} \mbE \norm{\nabla f(x_t)}^2 \leq \frac{2(f(x_0)-f^{\star})}{\gamma T} +\frac{\gamma L(M\zeta^2+\sigma^2)}{n}.
\end{equation*}
Hence the result. 
\end{proof}

\subsubsection{Final convergence result}\label{sec:big_oh}
From Remark \ref{rem:non-convex-big-oh}, the following corollary describes the $\cO(1/\sqrt{nT})$ convergence with an appropriate step-size for absolute compressors and $\delta$-contraction operators. 
\begin{corollary}\label{cor:big_oh}
Let Assumptions \ref{ass:smoothness}, \ref{ass:minimum}, \ref{ass:m-sigma2-bounded_noise}, and \ref{ass:bounded_similarity} hold and let $\{x_t\}_{t\geq0}$ denote the iterates of algorithm \ref{algo:ef-sgd}. Then, if \\
\begin{inparaitem}
    \item $\cC$ is an absolute compressor, we have
    \begin{eqnarray*}
     \frac{1}{T}\sum_{t=0}^{T-1} \mbE \norm{\nabla f(x_t)}^2 = \cO\left(\frac{\sqrt{L(M\zeta^2+\sigma^2)(f(x_0)-f^{\star})}}{\sqrt{nT}}+\frac{L((\frac{M}{n}(C+1)+1)+\frac{n\kappa^2}{M\zeta^2+\sigma^2})(f(x_0)-f^{\star})}{T}\right).
    \end{eqnarray*}
    \item $\cC$ is a $\delta$-contraction operator, we have
    \begin{eqnarray*}
       \textstyle \frac{1}{T}\sum_{t=0}^{T-1} \mbE \norm{\nabla f(x_t)}^2= \cO\left(\frac{\sqrt{\left(L(M\zeta^2+\sigma^2)\right)(f(x_0)-f^{\star})}}{\sqrt{nT}}+\frac{L\left(\max\{\frac{M}{n}(C+1)+1), (\frac{1}{\delta}+M)\sqrt{C+1}\}+\frac{n\left((1+M\delta)\zeta^2+\delta\sigma^2\right)}{\delta^2(M\zeta^2+\sigma^2)}\right)(f(x_0)-f^{\star})}{T}\right).
    \end{eqnarray*}
    \item $\cC$ is the identity compressor, we have
    \begin{equation*}
        \frac{1}{T}\sum_{t=0}^{T-1} \mbE \norm{\nabla f(x_t)}^2 = \cO\left(\frac{\sqrt{L(M\zeta^2+\sigma^2)(f(x_0)-f^{\star})}}{\sqrt{nT}}+\frac{L(\frac{M}{n}(C+1)+1)(f(x_0)-f^{\star})}{T}\right).
    \end{equation*}
\end{inparaitem}
\end{corollary}
\begin{proof}
Invoking Lemma \ref{lem:order_bounds_cubic} in Theorem \ref{thm:non-convex-wo-bg} and Theorem \ref{thm:non-convex-wo-bg-rel}, and Lemma \ref{lem:order_bounds} in Theorem \ref{thm:non-convex-wo-bg-sgd} we get the results.
\end{proof}

While compression does not affect the slower decaying $\cO(1/\sqrt{nT})$ term for both absolute compressors and $\delta$-contraction operators, we observe $\delta$-contraction operators have $1/\delta^2$ dependence in the $\cO(1/T)$ term when $\zeta\neq0$ (heterogeneous data). Therefore, in this setting, the Top-$k$ sparsifier has $d^2/k^2$ in the numerator of $\cO(1/T)$ term. On the other hard, hard-threshold has $d\lambda^2$ in the numerator of $\cO(1/T)$ term even when $\zeta\neq0$, and thus has a significantly better dependence on $d$. 

\subsection{Convex convergence analysis}\label{sec:abs-quasi}
In this Section, we provide convergence results for {\em distributed compressed SGD} with {\em absolute compressors} and an {\em EF} where the loss function on each worker $f_i$ is $\mu$-strongly convex with $\mu\geq0$ (see Assumption \ref{ass:mu-strong-convex}). Our analysis is inspired by the proof techniques in \cite{stich2019error} which analyzes an EF SGD with $\delta$-contraction operators in the single node ($n=1$) case. \cite{beznosikov2020biased} extended this analysis to the distributed ($n>1$) case for $\delta$-contraction operators.

The main highlight of our result is---while gradient compression affects the leading slower-decaying term in the distributed convergence of $\delta$-contraction operators, in case of absolute-compressors, under the same set-up, gradient compression only affects the faster-decaying terms. We start with the following key result by Nesterov \cite{nesterov2018lectures} for convex and smooth functions.
\begin{lemma}
Let $f_i$ follow Assumptions \ref{ass:smoothness} and Assumption \ref{ass:mu-strong-convex} with $\mu\geq0$, then
\begin{equation}\label{eq:L_bregman_div}
    \norm{\nabla f_i(y)-\nabla f_i(x)}^2 \leq {2L}(f_i(y)-f_i(x)-\dotprod{\nabla f_i(x),y-x}), \qquad \forall x,y \in \R^d.
\end{equation}
\end{lemma}

We start with the strongly-convex decent lemma from \cite{beznosikov2020biased}.
\begin{lemma}\label{lem:scd} \textbf{(Strongly convex descent lemma)} (Lemma 21 in \cite{beznosikov2020biased}) 
Let Assumptions \ref{ass:smoothness}, \ref{ass:minimum}, \ref{ass:m-sigma2-bounded_noise}, and \ref{ass:mu-strong-convex} hold.~Denote $D\eqdef\frac{1}{n}\sum_{i=1}^n\norm{\nabla f_i(x^{\star})}^2$. If $\gamma_t\leq\frac{1}{4L(1+{2M}/n)}$, for all $t \geq 0$, then the iterates, $\{\tilde{x}_t\}_{t\geq0}$ of Algorithm \ref{algo:ef-sgd} follow
\begin{eqnarray*}
\mbE_t\norm{\tilde{x}_{t+1}-x^{\star}}^2 \leq (1-\frac{\mu\gamma_t}{2})\norm{\tilde{x}_t-x^{\star}}^2 - \frac{\gamma_t}{2}[f(x_t)-f^{\star}] + 3L\gamma_t  \norm{x_t-\tilde{x}_t}^2 + (\gamma_t^2)\frac{\sigma^2+2MD}{n}.
\end{eqnarray*}
\end{lemma}
\begin{proof}
We have
\begin{eqnarray*}
\norm{\tilde{x}_{t+1}-x^{\star}}^2 &\overset{\text{Lemma \ref{lem:virtual_iterate}}}{=}& \norm{\tilde{x}_{t}-x^{\star}}^2 - 2 \gamma_t\dotprod{\bar{g}_t, \tilde{x}_t-x^{\star}} + \gamma_t^2 \norm{\bar{g}_t}^2\\
&=& \norm{\tilde{x}_t-x^{\star}}^2 - 2 \gamma_t\dotprod{\bar{g}_t, x_t-x^{\star}}+\gamma_t^2 \norm{\bar{g}_t}^2 + 2\gamma_t\dotprod{\bar{g}_t, x_t-\tilde{x}_t}.
\end{eqnarray*}
Therefore,
\begin{eqnarray}
\notag
\mbE_t \norm{\tilde{x}_{t+1}-x^{\star}}^2  &=& \norm{\tilde{x}_{t}-x^{\star}}^2 - 2 \gamma_t \dotprod{\mbE_t[ \bar{g}_t], x_t - x_\star} + \gamma_t^2 \mbE_t\norm{\bar{g}_t}^2 + 2\gamma_t\dotprod{\mbE_t[\bar{g}_t], x_t-\tilde{x}_t}\\
\label{eq:dlone}
&=& \norm{\tilde{x}_{t}-x^{\star}}^2 - 2 \gamma_t \dotprod{\nabla f(x_t), x_t - x_\star} + \gamma_t^2 \mbE_t\norm{\bar{g}_t}^2 + 2\gamma_t\dotprod{\nabla f(x_t), x_t-\tilde{x}_t}.\nonumber\\
\end{eqnarray} 
First, we bound $2\dotprod{\nabla f(x_t), x_t - \tilde{x}_t}$. We use Young's inequality \eqref{eq:youngs2} with $\rho=\frac{1}{2L}$ and get
\begin{eqnarray}
\notag 2\dotprod{\nabla f(x_t), x_t - \tilde{x}_t} &\leq& \frac{1}{2L}\norm{\nabla f(x_t)}^2 + 2L\norm{x_t - \tilde{x}_t}^2\\
\label{eq:dlthree} &\overset{\eqref{eq:L_bregman_div}, \nabla f(x^{\star})=0}{\leq}& f(x_t)-f(x^{\star}) + 2L\norm{x_t - \tilde{x}_t}^2.
\end{eqnarray}
Next, we bound $-2\dotprod{\nabla f(x_t), x_t - x^{\star}}$. We use the $\mu$-strong convexity of $f$ to find
\begin{equation}\label{eq:lsas}
   -2\dotprod{\nabla f(x_t), x_t-x^{\star}} \leq 2(f(x^{\star})-f(x_t))-\mu\norm{x_t-x^{\star}}^2.
\end{equation}
However, since we want to work with $\norm{\tilde{x}_t-x^{\star}}^2$ instead of $\norm{x_t-x^{\star}}^2$, we get rid of $\norm{x_t-x^{\star}}^2$ using \eqref{eq:youngs} with $\rho=1$ as
\begin{equation*}
     \norm{x_t-x^{\star}}^2 \geq  \frac{1}{2}\norm{\tilde{x}_t-x^{\star}}^2-\norm{x_t-\tilde{x}_t}^2.
\end{equation*}
Substituting this in Equation \eqref{eq:lsas}, we get 
\begin{equation}\label{eq:dlfour}
    -2\dotprod{\nabla f(x_t), x_t-x^{\star}} \leq 2(f(x^{\star})-f(x_t))-\frac{\mu}{2}\norm{\tilde{x}_t-x^{\star}}^2+ \mu \norm{x_t-\tilde{x}_t}^2.
\end{equation}
Finally, we bound $\mbE_t \norm{\bar{g}_t}^2$ as
\begin{eqnarray}
\notag
\mbE_t \norm{\frac{1}{n}\sum_{i=1}^n g_{i,t}}^2  &=& \mbE\left[\norm{\frac{1}{n}\sum_{i=1}^n (\nabla f_i(x_t)+ \xi_{i,t})}^2| x_t\right]\\
\notag&=& \mbE \left[\norm{ \nabla f(x_t)+ \frac{1}{n}\sum_{i=1}^n\xi_{i,t}}^2| x_t\right]\\
\notag&\overset{\mbE[\xi_{i,t}|x_{t}]=0}{=}& \norm{\nabla f(x_t)}^2 + \mbE\left[\norm{ \frac{1}{n}\sum_{i=1}^n\xi_{i,t}}^2| x_t\right]\\
\notag&\overset{\mbE[\xi_{i,t}|x_{t}]=0}{=}& \norm{\nabla f(x_t)}^2 + \frac{1}{n^2}\sum_{i=1}^n\mbE\left[\norm{ \xi_{i,t}}^2| x_t\right]\\
\notag&\overset{\text{Assumption}\;\ref{ass:m-sigma2-bounded_noise}}{\leq}& \norm{\nabla f(x_t)}^2 + \frac{1}{n^2}\sum_{i=1}^n (M\norm{\nabla f_i(x_{t})}^2+\sigma^2)\\
\notag&=& \norm{\nabla f(x_t)}^2+ \frac{M}{n^2}\sum_{i=1}^n\norm{\nabla f_i(x_{t})-\nabla f_i(x^{\star})+ \nabla f_i(x^{\star})}^2+\frac{\sigma^2}{n}\\
\notag&\leq& \norm{\nabla f(x_t)}^2+ \frac{2M}{n^2}\sum_{i=1}^n\left(\norm{\nabla f_i(x_{t})-\nabla f_i(x^{\star})}^2+ \norm{\nabla f_i(x^{\star})}^2\right)\nonumber\\&&+\frac{\sigma^2}{n}\nonumber\\
\notag&\overset{\eqref{eq:L_bregman_div}, D=\frac{1}{n}\sum_{i=1}^n\norm{\nabla f_i (x^{\star})}^2}{\leq}& \norm{\nabla f(x_t)}^2+ \frac{2M}{n^2}\sum_{i=1}^n2L[f_i(x_t)-f_i(x^{\star})-\dotprod{\nabla f_i(x^{\star}), x_t-x^{\star}}] \\&\quad&+\frac{2MD}{n}+\frac{\sigma^2}{n}\\
\notag&\overset{\nabla f(x^{\star})=0}{=}& \norm{\nabla f(x_t)- \nabla f(x^{\star})}^2+ \frac{4LM}{n}(f(x_t)-f(x^{\star}))+\frac{2MD+\sigma^2}{n}\\
\label{eq:dltwo}&\overset{\eqref{eq:L_bregman_div}, \nabla f(x^{\star})=0}{\leq}& 2L\left(1+\frac{2M}{n}\right)(f(x_t)-f(x^{\star}))+ \frac{2MD+\sigma^2}{n}.
\end{eqnarray}
We now substitute \eqref{eq:dlthree}, \eqref{eq:dlfour}, and \eqref{eq:dltwo} in \eqref{eq:dlone} to get 
\begin{eqnarray*}
\mbE_t \norm{\tilde{x}_{t+1}-x^{\star}}^2 &=& \norm{\tilde{x}_{t}-x^{\star}}^2 - 2 \gamma_t \dotprod{\nabla f(x_t), x_t - x_\star} + \gamma_t^2 \mbE_t\norm{\bar{g}_t}^2 + 2\gamma_t\dotprod{\nabla f(x_t), x_t-\tilde{x}_t} \\
&\leq& \left(1-\frac{\mu\gamma_t}{2}\right)\norm{\tilde{x}_{t}-x^{\star}}^2  - \gamma_t \left(1-\gamma_t\cdot2L\left(1+\frac{2M}{n}\right)\right) (f(x_t)-f(x^{\star})) \\ &&+ \gamma_t(2L+\mu)\norm{x_t-\tilde{x}_t}^2 + \gamma_t^2\frac{2MD+\sigma^2}{n}.
\end{eqnarray*}
Choosing $\gamma_t\leq\frac{1}{4L(1+{2M}/n)}$ gives the desired result.
\end{proof}

Next, we give the convex convergence result of {\em distributed EF SGD} with {\em absolute compressors}.

\subsubsection{Absolute compressors}\label{sec:convex-abs-comp}
The next theorem combines the results of Theorems \ref{thm:sc_abs} and \ref{thm:weakly-convex} from the main paper. We present them as a single theorem (Theorem \ref{thm:sc_abs_whole}) to keep the structure of the proofs simple.
\begin{theorem}\label{thm:sc_abs_whole}
Let Assumptions \ref{ass:smoothness}, \ref{ass:minimum}, \ref{ass:m-sigma2-bounded_noise}, and \ref{ass:mu-strong-convex} hold.~Denote $D\eqdef\frac{1}{n}\sum_{i=1}^n\norm{\nabla f_i(x^{\star})}^2$, and $R_0=\norm{x_0-x^{\star}}^2$.~Then the iterates, $\{x_t\}_{t\geq0}$ of Algorithm \ref{algo:ef-sgd} with an absolute compressor, $\cC_{\upsilon}$ have the following convergence rates if Assumption \ref{ass:mu-strong-convex} is satisfied with the following choices of the parameters:

\begin{inparaenum}[i)]
    \item \textbf{(Theorem \ref{thm:sc_abs})} If $\mu>0$, a constant setp-size $\{\gamma_t=\gamma\}_{t\geq0}$, with $\gamma\leq\frac{1}{4L(1+{2M}/n)}$ is chosen as in Lemma \ref{lem:sc_order_bounds_const_lr} and weights $\{w_t=(1-\mu\gamma/2)^{-(t+1)}\}_{t\geq0}$ then
    \begin{equation*}
        \mbE[f(\bar{x}_T)]-f^{\star} = \tilde{\cO}\left(L(1+M/n)R_0\exp\left[-\frac{\mu T}{8L(1+2M/n)}\right]+\frac{\sigma^2+MD}{\mu nT}+ \frac{L\kappa^2}{\mu^2T^2}\right).
    \end{equation*}

    \item \textbf{(Theorem \ref{thm:weakly-convex})} If $\mu=0$, a constant setp-size $\{\gamma_t=\gamma\}_{t\geq0}$, with $\gamma\leq\frac{1}{4L(1+{2M}/n)}$ is chosen as in Lemma \ref{lem:order_bounds_cubic} and weights $\{w_t=1\}_{t\geq0}$ then
    \begin{equation*}
         \mbE[f(\bar{x}_T)]-f^{\star} = \cO\left(\frac{\sqrt{(\sigma^2+MD)R_0}}{\sqrt{nT}}+\frac{\left(\frac{nL\kappa^2}{\sigma^2+MD}+L(1+M/n)\right)R_0}{T}\right).
    \end{equation*}

    \item If $\mu>0$, setp-sizes $\{\gamma_t=\frac{4}{\mu(\phi+t)}\}_{t\geq0}$, and weights $\{w_t=\phi+t\}_{t\geq0}$, respectively with $\phi=\frac{16L}{\mu}(1+\frac{2M}{n})$ then
    \begin{equation*}
        \mbE[f(\bar{x}_T)]-f^{\star} = \cO\left(\frac{\sigma^2+MD}{\mu nT}+ \frac{L^2(1+M/n)^2R_0+L\mu\kappa^2\ln(T)}{\mu^2T^2}\right).
    \end{equation*}

\end{inparaenum}
In the above, $\bar{x}_T=\frac{1}{W_T}\sum_{t=0}^Tw_tx_t$, and $W_T=\sum_{t=0}^Tw_t$.
\end{theorem}

\begin{proof}
By using Lemma \ref{lem:virtual_iterate} in Lemma \ref{lem:scd}, and taking total-expectation over all the previous iterates, we have 
\begin{eqnarray*}
\mbE\norm{\tilde{x}_{t+1}-x^{\star}}^2 &\leq& (1-\frac{\mu\gamma_t}{2})\mbE\norm{\tilde{x}_t-x^{\star}}^2 - \frac{\gamma_t}{2}\mbE[f(x_t)-f^{\star}] + 3L\gamma_t \mbE \norm{\bar{e}_t}^2 + \gamma_t^2(\frac{\sigma^2+2MD}{n})\\
 &\overset{\eqref{eq:y}}{\leq}& (1-\frac{\mu\gamma_t}{2})\mbE\norm{\tilde{x}_t-x^{\star}}^2 - \frac{\gamma_t}{2}\mbE[f(x_t)-f^{\star}] + 3L\gamma_t \sum_{i=1}^n\frac{1}{n}\mbE \norm{e_{i,t}}^2\\&& + \gamma_t^2(\frac{\sigma^2+2MD}{n})\\
 &\overset{\text{Remark}\;\ref{rem:bounded_memory}}{\leq}& (1-\frac{\mu\gamma_t}{2})\mbE\norm{\tilde{x}_t-x^{\star}}^2 - \frac{\gamma_t}{2}\mbE[f(x_t)-f^{\star}] + 3L\gamma_t^3\kappa^2 + \gamma_t^2(\frac{\sigma^2+2MD}{n}).
\end{eqnarray*}
Rearranging, we get
\begin{equation}\label{eq:scdl}
\mbE[f(x_t)]-f^{\star} \leq \frac{2}{\gamma_t}(1-\frac{\mu\gamma_t}{2})\mbE\norm{\tilde{x}_t-x^{\star}}^2 - \frac{2}{\gamma_t}\mbE\norm{\tilde{x}_{t+1}-x^{\star}}^2+ \gamma_t\frac{2\sigma^2+4MD}{n} + 6L\gamma_t^2\kappa^2.
\end{equation}
With $r_t=2\mbE\norm{\tilde{x}_t-x^{\star}}^2$, $a=\frac{\mu}{2}$, $c=\frac{2\sigma^2+4MD}{n}$, $b=6L\kappa^2$, we can see the RHS as $\frac{1}{\gamma_t}(1-a\gamma_t)r_t - \frac{1}{\gamma_t}r_{t+1} + c\gamma_t + b\gamma_t^2$. Thus, we use Lemma \ref{lem:sc_order_bounds_const_lr} and Lemma \ref{lem:sc_order_bounds_dec_lr} to get the first and the third result respectively. Note that to get the LHS, we use the convexity of $f$ as $\frac{1}{W_T}\sum_{t=0}^Tw_tf(x_t)\geq f(\bar{x}_T)$. Finally, to get the second result, we substitute $\mu=0$ in Equation \eqref{eq:scdl} and perform telescopic sum to get
\begin{equation*}
    \frac{\sum_{t=0}^T\mbE[f(x_t)]}{T+1}-f^{\star} \leq  \frac{2\norm{x_0-x^{\star}}^2}{\gamma(T+1)}+ \frac{2\sigma^2+4MD}{n}\gamma + 6L\kappa^2\gamma^2.
\end{equation*}
We now use Lemma \ref{lem:order_bounds_cubic} and convexity of $f$ to arrive at the desired result.

\end{proof}

\subsection{Comparison against unbiased compressors}\label{sec:unbiased-comp}
Till now, we have discussed the convergence of compressed SGD using EF. However, unbiased relative compressors which satisfy \myNum{i} $\mbE_{\cC}[\cC(x)]=x$; and \myNum{ii} $\mbE_{\cC}\norm{\cC(x)-x}^2 \leq \Omega\norm{x}^2$ do not require EF. We compare the convergence of such unbiased compressors and absolute compressors with EF. With the notations above,
\cite{horvath2021a} provide the following convergence result for unbiased compressors in the strongly convex case:
\begin{equation*}
   \textstyle \mbE[f(\bar{x}_T)]-f^{\star} + \mu \mbE[\norm{x_T-x^*}^2] \leq 64 \Omega_n L (1+M/n)R_0 \exp\left[-\frac{\mu T}{4\Omega_n L(1+M/n)}\right] + 36\frac{(\Omega_n-1)D+\Omega\sigma^2/n}{\mu T},
\end{equation*}
where $\Omega_n=\frac{\Omega-1}{n}+1.$
Comparing with Theorem \ref{thm:sc_abs}, we find unbiased compressors have compression affecting the slower-decaying $\frac{1}{T}$ term. Although, we note that their convergence is in both the iterates and functional values, whereas ours is only in functional values.

\section{Addendum to numerical experiments}\label{apdx:add_exps}

\smartparagraph{Overview.} In this section, we provide: 
\begin{enumerate}[i)]
\item The experimental settings and implementation details of our DNN experiments (\S\ref{sec:experiment_details}).
\item Further discussion on the large error-accumulation of Top-$k$ and its effect on total-error (\S\ref{sec:large_error_topk}).
\item Logistic regression experiments (\S\ref{sec:logistic_regression}).
\item Comparison against the state-of-the-art adaptive sparsifier ACCORDION \cite{agarwal2020accordion}. (\S\ref{sec:accordion})
\item Experiment with Entire-model Top-$k$ (\S\ref{apdx:fullmodel}).
\item Experiments without EF, and discussion on different forms of EF (\S\ref{apdx:memory}).
\end{enumerate}

\subsection{Experimental settings and implementation details}\label{sec:experiment_details}  
We implement the sparsifiers in {\tt PyTorch}. For each method, a gradient reducer class is defined, which invokes the appropriate compression function and then perform the aggregation among the workers. Tables \ref{tab:benchmarks}, \ref{tab:resnet18}, \ref{tab:lstm}, and \ref{tab:ncf} provide the experimental details for each of the tasks. We used the default hyper-parameters provided in the mentioned repositories for each task.

\begin{table*}
\centering
\caption{Summary of the benchmarks used}
\label{tab:benchmarks}
\scalebox{0.9}{
\begin{tabular}{ccccc}
\hline
Model & Task & Dataset & No. of Parameters & Optimizer  \\
\hline
ResNet-18 \cite{resnet} & Image classification & CIFAR-10~\cite{cifar10}     & 11,173,962  & SGD+Nesterov momentum  \\
LSTM \cite{lstm}  & Language modelling   & Wikitext-2 \cite{wikitext2}       & 28,949,319  & Vanilla SGD  \\
NCF \cite{ncf}   & Recommendation       & Movielens-20M  & 31,832,577   & ADAM \cite{adam}\\
 \hline
\end{tabular}
}
\end{table*}

\begin{table*}[!htbp]
\centering
\caption{Image classification task}\label{tab:resnet18}
\vspace{3pt}
\scalebox{1.0}{
\begin{tabular}{cc}
\hline
Dataset & CIFAR-10 \\
Architecture & ResNet-18 \\
Repository &PowerSGD \cite{PowerSGD} \\
&\footnotesize{See \url{https://github.com/epfml/powersgd}} \\
License &MIT \\
\hline
Number of workers & 8 \\
Global Batch-size & 256 $\times$ 8 \\
\hline 
Optimizer & SGD with Nesterov Momentum \\
Momentum & 0.9 \\
Post warmup LR & $0.1$ $\times$ $16$ \\
LR-decay & $/10$ at epoch 150 and 250 \\ 
LR-warmup & Linearly within 5 epochs, starting from $0.1$ \\
Number of Epochs &300 \\
Weight decay &$10^{-4}$ \\
\hline
Repetitions & 3, with different seeds \\
\hline
Hard-threshold: $\lambda$ values & $\{1.2\times10^{-2}, 7.2\times10^{-3}, 5\times10^{-3}, 3\times10^{-3}, 1.8\times10^{-3}\}$ \\
Top-$k$: $k$ values & $\{0.03\%, 0.06\%, 0.12\%, 0.3\%, 0.75\%\}$ \\
\hline
\end{tabular}
}
\end{table*}

\begin{table*}[!htbp]
\centering
\caption{Language modelling task}
\label{tab:lstm}
\vspace{3pt}
\scalebox{1.0}{
\begin{tabular}{cc}
\hline
Dataset & WikiText2 \\
Architecture & LSTM \\ 
Repository &PowerSGD \cite{PowerSGD} \\
&\footnotesize{See \url{https://github.com/epfml/powersgd}} \\
License & MIT \\
\hline
Number of workers & 8 \\
Global Batch-size & 128 $\times$ 8 \\
\hline 
Optimizer & vanilla SGD \\
Post warmup LR & $1.25$ $\times$ $16$ \\
LR-decay & $/10$ at epoch 60 and 80 \\
LR-warmup & Linearly within 5 epochs, starting from $1.25$ \\
Number of Epochs & 90 \\
Weight decay &$0$ \\
\hline
Repetitions & 3, with different seeds \\
\hline
Hard-threshold: $\lambda$ values & $\{4.5\times10^{-3},2.75\times10^{-3},1.6\times10^{-3},1.12\times10^{-3}\}$ \\
Top-$k$: $k$ values & $\{0.025\%, 0.05\%, 0.1\%, 0.2\%\}$ \\
\hline
\end{tabular}
}
\end{table*}

\begin{table*}
\centering
\caption{Recommendation task}\label{tab:ncf}
\vspace{3pt}
\scalebox{1}{
\begin{tabular}{cc}
\hline
Dataset & Movielens-20M \\
Architecture & NCF \\
Repository & NVIDIA Deep Learning Examples \\
&\footnotesize{See \url{https://github.com/NVIDIA/DeepLearningExamples}} \\
\hline
Number of workers & 8 \\
Global Batch-size & $2^{20}$ \\
\hline 
Optimizer & ADAM \\
ADAM $\beta_1$ & 0.25 \\
ADAM $\beta_2$  & 0.5 \\
ADAM LR & $4.5\times10^{-3}$ \\
Number of Epochs & 30 \\
Weight decay &$0$ \\
Dropout &$0.5$ \\
\hline
Repetitions & 3, with different seeds \\
\hline
Hard-threshold: $\lambda$ values & $\{2\times10^{-6},1.3\times10^{-6}, 1\times10^{-6}, 4\times10^{-7}\}$ \\
Top-$k$: $k$ values & $\{7.7\%, 9.5\%, 11.3\%, 13.7\%\}$ \\
License & Open Source \\
\hline
\end{tabular}
}
\vspace{7pt}
\end{table*}

\subsection{Top-$k$ suffers from large error accumulation}\label{sec:large_error_topk}
In Figure \ref{fig:topk_error_accumulation_a}, we show the cascading effect (mentioned in \S\ref{sec:discussion}) for the experiment in Figure \ref{fig:topk_error_accumulation}. We observe that the error norm profile in Figure\ref{fig:topk_error_accumulation_a} c closely follows the error compensated gradient norm profile in Figure\ref{fig:topk_error_accumulation_a} b.  

In Figure \ref{fig:large_error_Topk_wikitext2} and Figure \ref{fig:large_error_Topk_ncf-ml20m}, we show that hard-threshold has a better convergence because of a smaller total-error in LSTM-WikiText2 and NCF-Ml-20m benchmarks. We note that we use the ADAM optimizer on the NCF-Ml-20m benchmark, and therefore our total-error insight is not theoretically justified in this case. Nevertheless, our experiment empirically confirms that the total-error perspective is useful for optimizers beyond vanilla SGD and momentum SGD.

\begin{figure*}[!htbp]
\centering
\begin{subfigure}{0.32\textwidth}
		\centering
		\includegraphics[width=\textwidth]{Figures/experiments/logistic-regression/sgd_gisette_num_of_workers_20_func_vals_bits.pdf}
				\caption{}
	\end{subfigure}
\begin{subfigure}{0.32\textwidth}
		\centering
		\includegraphics[width=\textwidth]{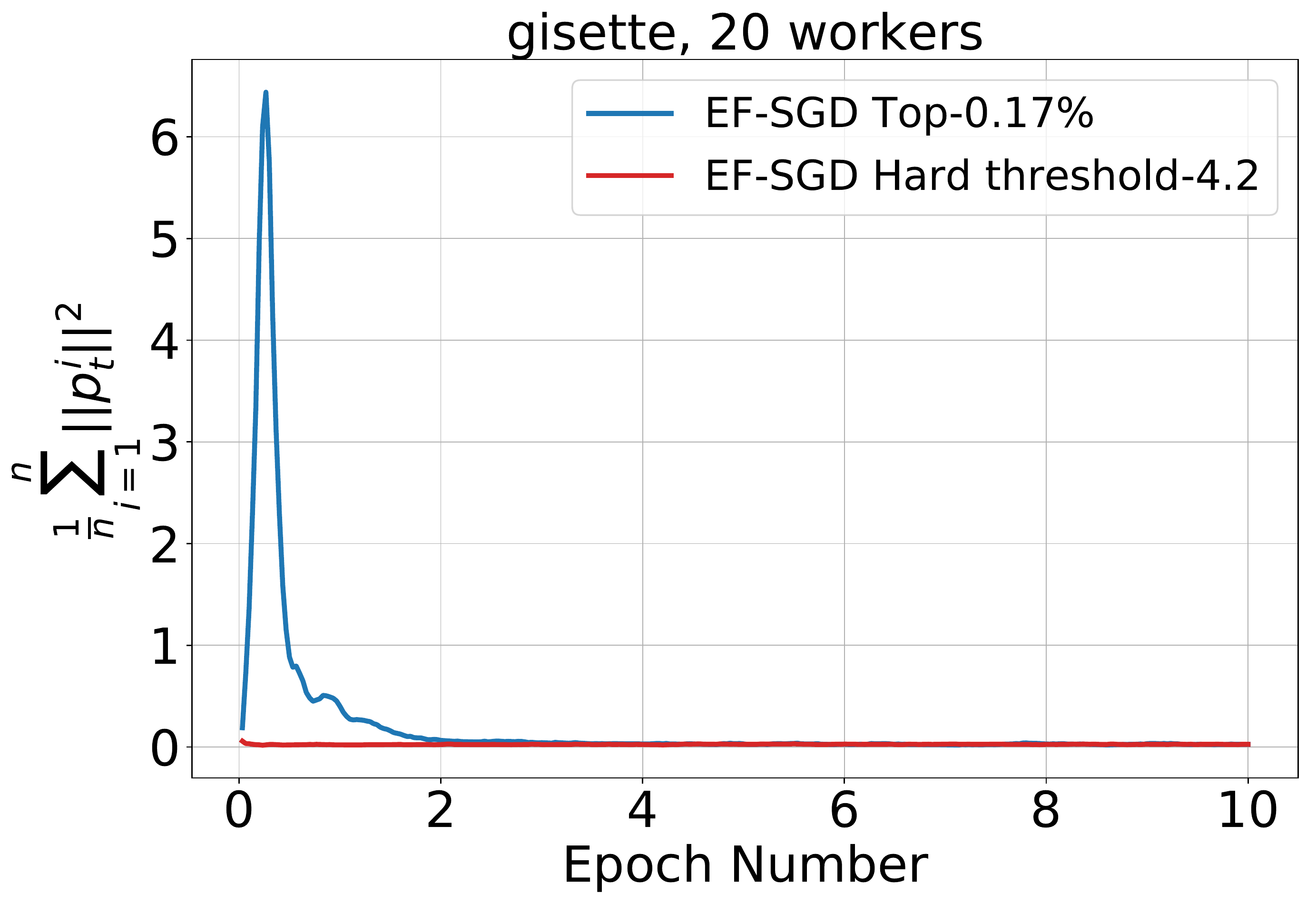}
		\caption{}
	\end{subfigure}
\begin{subfigure}{0.32\textwidth}
        \centering
		\includegraphics[width=\textwidth]{Figures/experiments/logistic-regression/sgd_gisette_num_of_workers_20_avg_error_norms_data_passes.pdf}
		\caption{}
	\end{subfigure}
\caption{
\small{Convergence of Top-$k$ and Hard-threshold for a logistic regression model on \texttt{gisette} LIBSVM dataset with 20 workers: (a) Functional suboptimality vs. bits communicated; (b) Error-compensated gradient norm vs. Epoch; (c) Error-norm vs. iterations. Top-$k$ has large error-accumulation due to the cascading-effect.}}\label{fig:topk_error_accumulation_a}
\end{figure*}

\begin{figure*}[!htbp]
\vspace{-1em}
\centering
\begin{subfigure}{0.32\linewidth}
		\includegraphics[width=\textwidth]{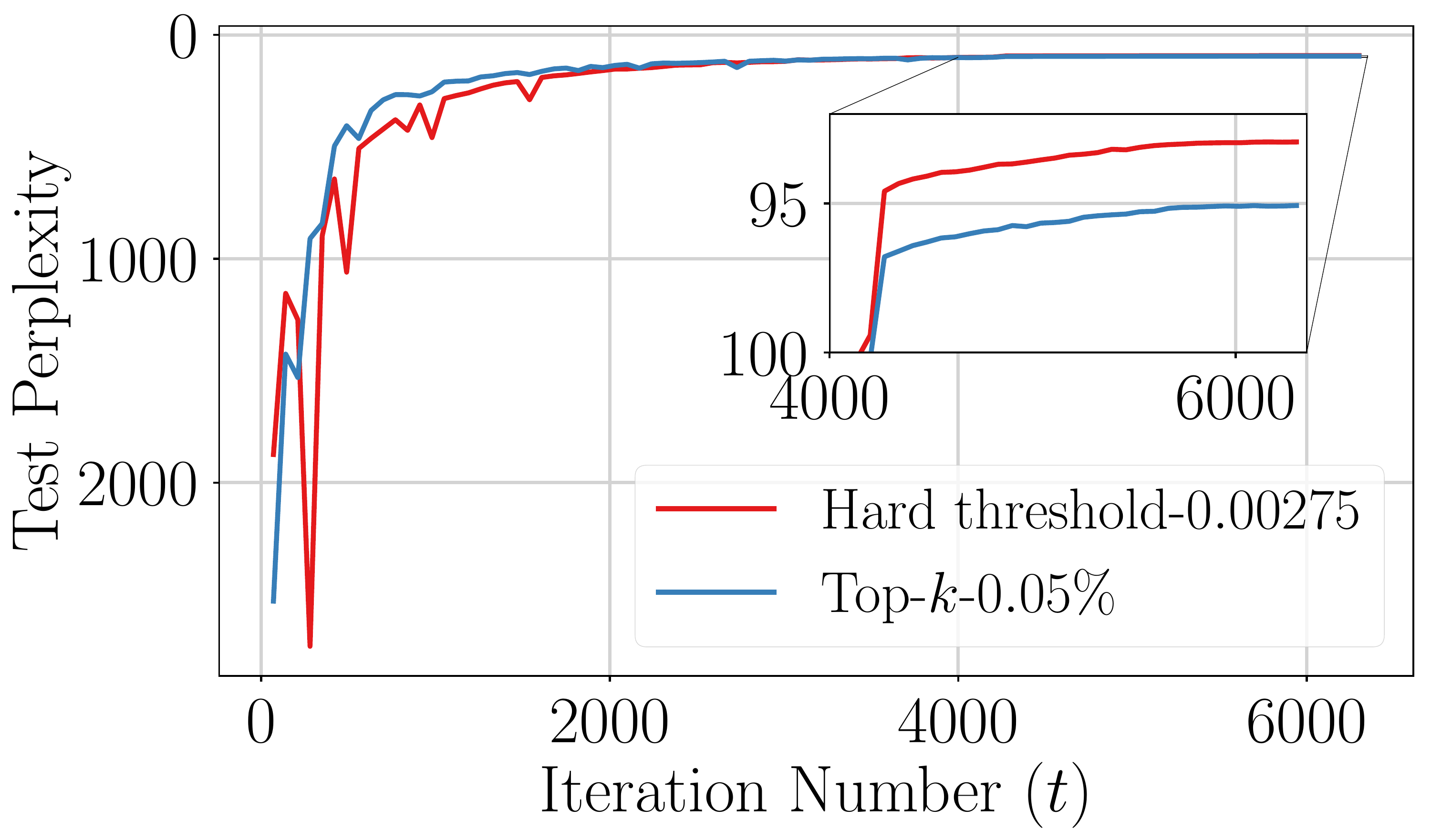}
				\caption{}\label{accuracy_wikitext2}
\end{subfigure}
\begin{subfigure}{0.32\linewidth}
		\includegraphics[width=\textwidth]{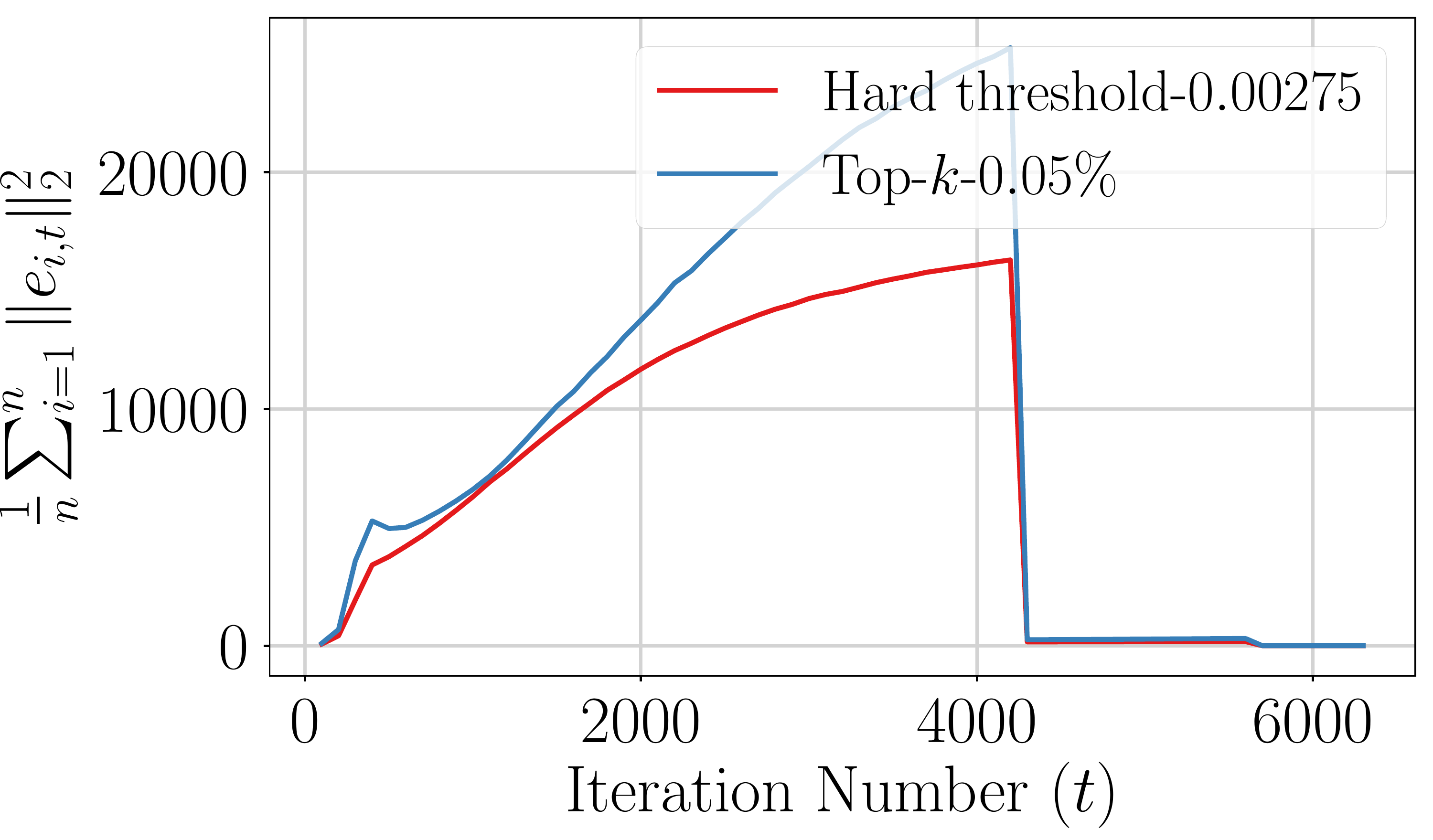}
			\caption{}\label{lr_scaled_mem_wikitext2}
\end{subfigure}
\begin{subfigure}{0.32\linewidth}
		\includegraphics[width=\textwidth]{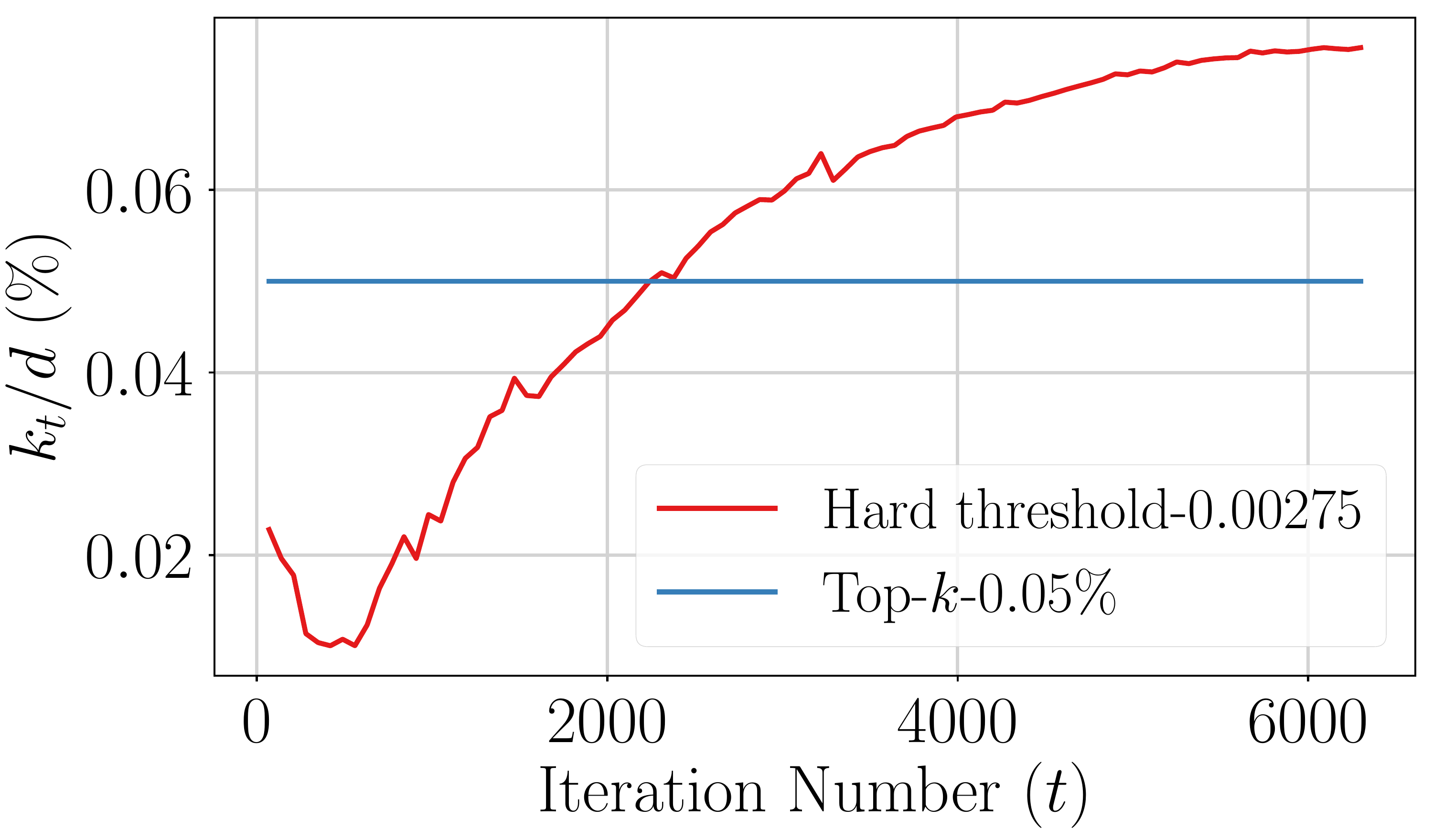}
		\caption{}\label{sparse_wikitext2}
\end{subfigure}
\caption{\small{\textbf{Convergence of Top-$k$ and Hard-threshold for an LSTM on WikiText2 at $\mathbf{0.05\%}$ average density:} ~(a)~Test-perplexity vs. Iterations,~(b)~Error-norm vs. Iterations, (c) Density ($k_t/d$) vs. Iterations. $k=0.05\%$ of $d$, and $\lambda$ = 0.0072. Hard-threshold has better convergence than Top-$k$ because of a smaller total-error.} }\label{fig:large_error_Topk_wikitext2}
\end{figure*}

\begin{figure*}[!htbp]
\vspace{-1em}
\centering
\begin{subfigure}{0.32\linewidth}
		\includegraphics[width=\textwidth]{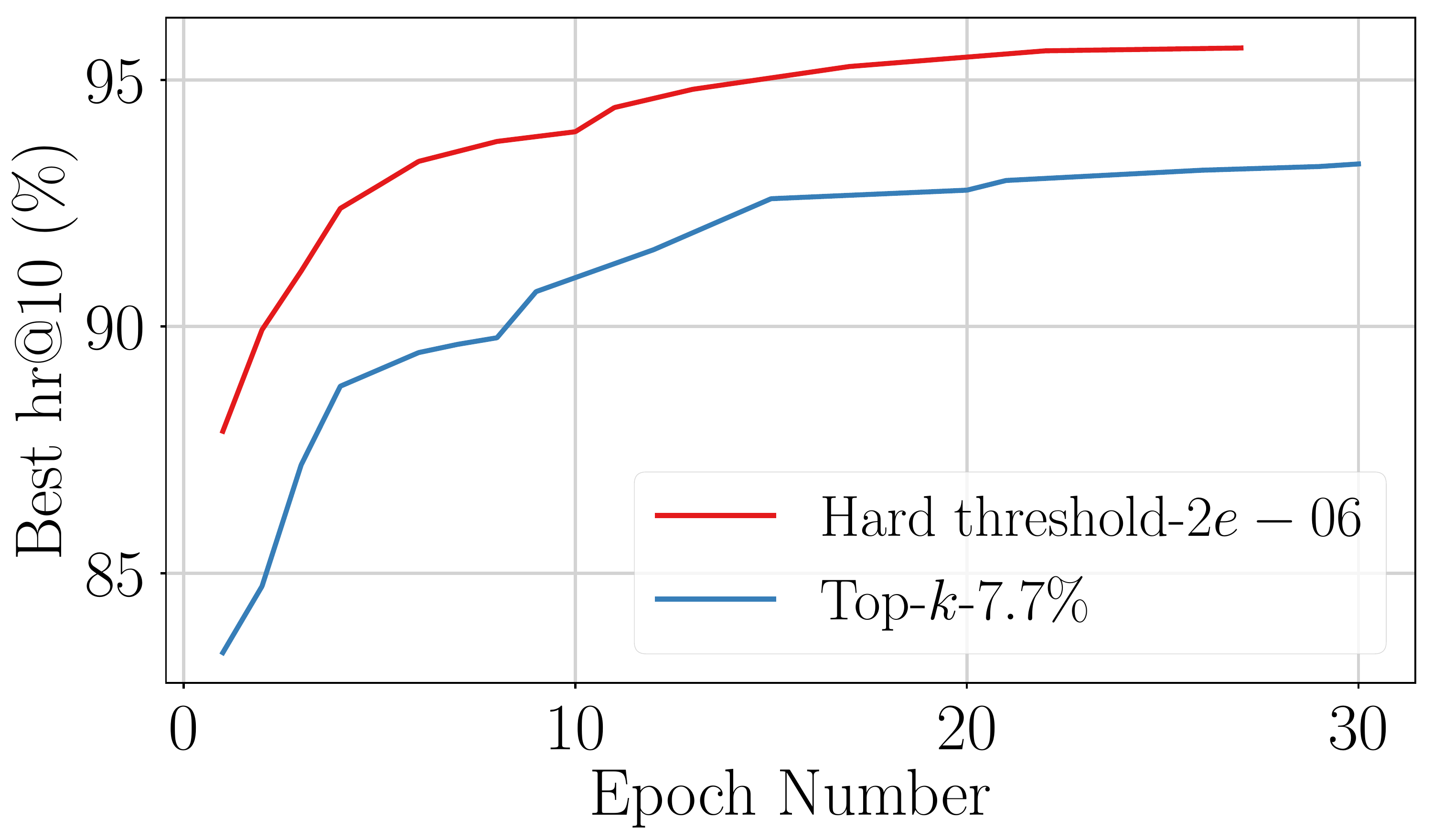}
				\caption{}\label{accuracy_ncf-ml20m}
\end{subfigure}
\begin{subfigure}{0.32\linewidth}
		\includegraphics[width=\textwidth]{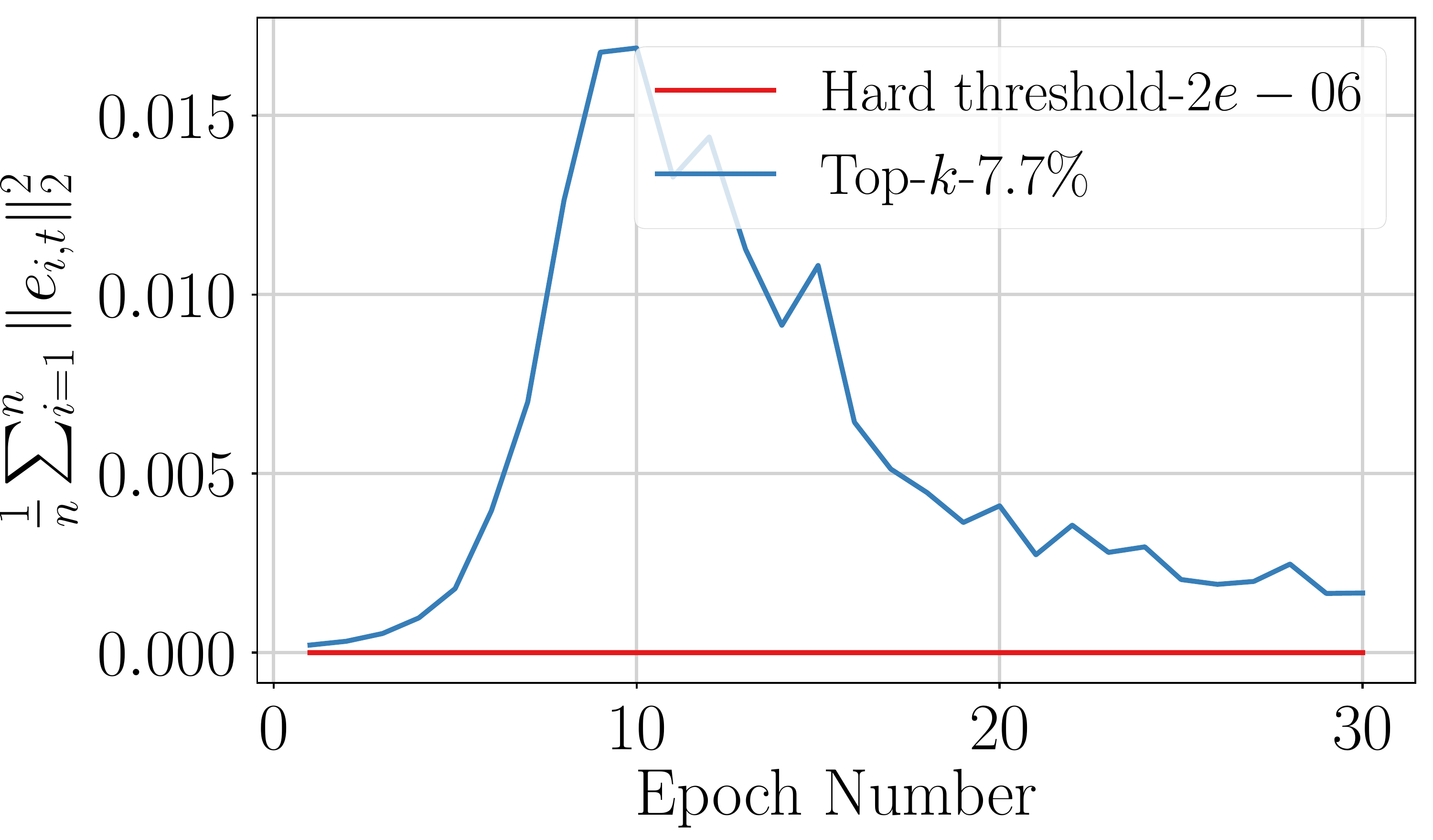}
			\caption{}\label{lr_scaled_mem_ncf-ml20m}
\end{subfigure}
\begin{subfigure}{0.32\linewidth}
		\includegraphics[width=\textwidth]{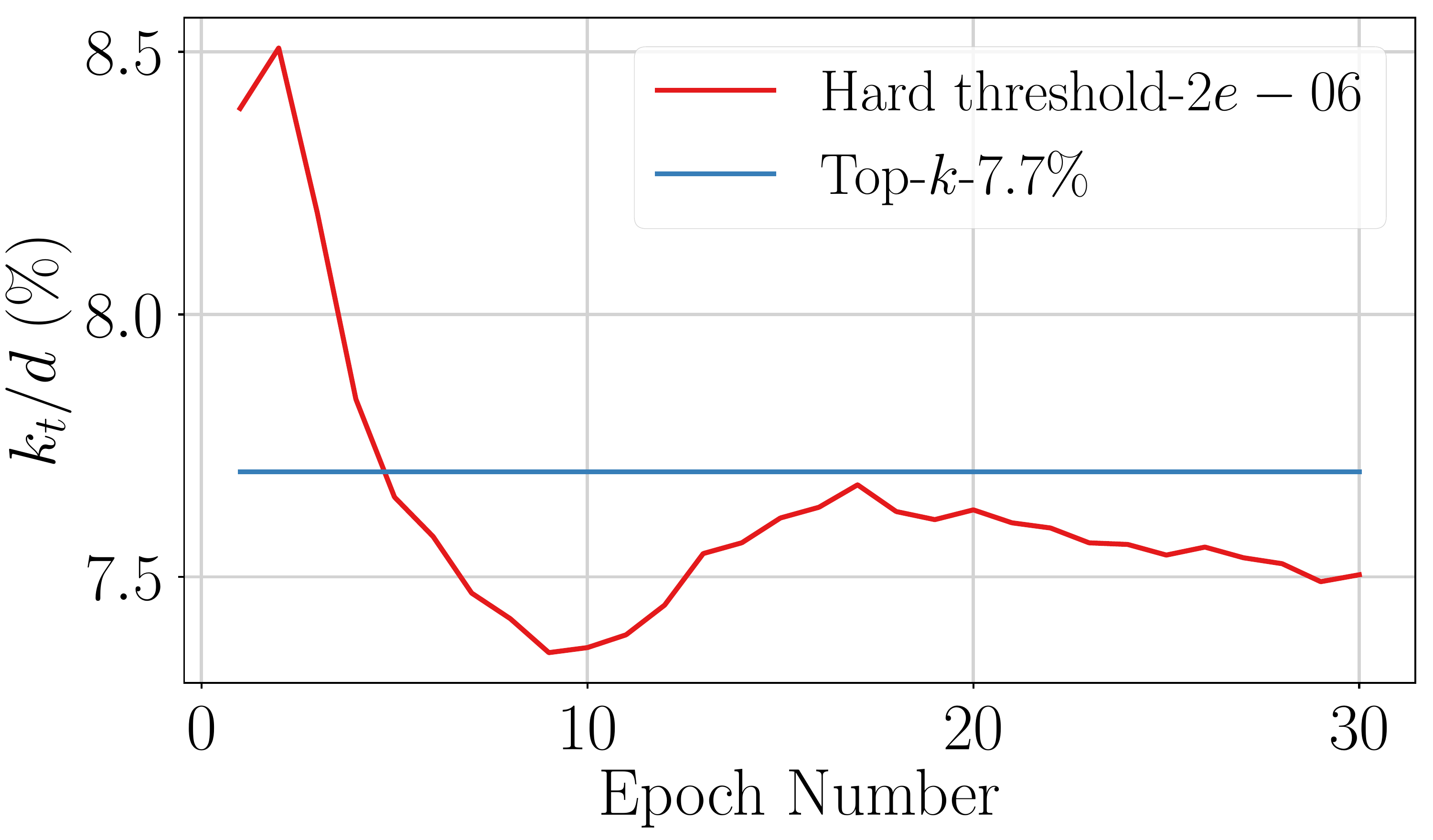}
		\caption{}\label{sparse_ncf-ml20m}
\end{subfigure}
\caption{\small{\textbf{Convergence of Top-$k$ and Hard-threshold for NCF on ML-20m at $\mathbf{7.7\%}$ average density:} ~(a)~Best Hit-rate@10 vs. Epochs,~(b)~Error-norm vs. Epochs, (c) Density ($k_t/d$) vs. Epochs. $k=0.06\%$ of $d$, and $\lambda$ = 0.0072. Hard-threshold has better convergence than Top-$k$ because of a smaller total-error.} }\label{fig:large_error_Topk_ncf-ml20m}
\vspace{-1em}
\end{figure*}

\subsection{Logistic regression experiments}\label{sec:logistic_regression}
For the convex experiments, we consider the following $\ell_2$ regularized logistic regression experiment considered in \cite{gorbunov2020linearly}\footnote{Open source code: \url{https://github.com/eduardgorbunov/ef_sigma_k}}: 

\begin{equation}\label{eq:log_reg}
\min_{x\in\mathbb{R}^d}f(x)=\frac{1}{N}\sum_{i=1}^N\log(1+\exp(-y_i\mathbf{A}[i,:]x))+\frac{\mu}{2}\|x\|^2, \quad \text{where }  \mathbf{A}\in\mathbb{R}^{N\times d}, y\in\mathbb{R}^N.
\end{equation}
The function, $f(x)$ in \eqref{eq:log_reg} is $\mu$-strongly convex and $L$-smooth with $L=\mu+\frac{\lambda_{\max}(\mathbf{A}^T\mathbf{A})}{4N}$. As in \cite{gorbunov2020linearly}, we use the setp-size $\gamma=1/L$, and $\mu=10^{-4}\frac{\lambda_{\max}(\mathbf{A}^T\mathbf{A})}{4N}$. We use standard LIBSVM datasets \cite{libsvm}, and split the dataset into number of worker partitions. For distributed EF-SGD, we use a local batch size of 1 at each node, where the new batch is chosen uniformly at random at each step.

\smartparagraph{Tuning the hard-threshold:} Our goal is to make $f(x_T)-f(x_\star)\leq\epsilon$, for a given precision, $\epsilon>0$. We set $\lambda$ such that $d\gamma^2\lambda^2=\epsilon$, i.e., $\lambda=\frac{\sqrt{\epsilon}}{d\sqrt{\gamma}}$.

\textit{Justification:} Remark \ref{rem:bounded_memory} states that by using a hard-threshold $\lambda>0$, the noise due to compression is $d\gamma^2\lambda^2$. Due to this compression noise, we expect (although we did not prove) that $x_T$ will oscillate in a $d\gamma^2\lambda^2$ neighborhood of the optimum, $x_\star$, i.e. $\norm{x_T-x_\star}^2\leq d\gamma^2\lambda^2$. Furthermore, by $L$-smoothness, we have 
\begin{equation*}
    f(x_T)-f(x_\star) \leq \frac{L}{2}\norm{x_T-x_\star}^2. 
\end{equation*}
Therefore, if we want to converge to a $\epsilon$-close functional-suboptimality value, $f(x_T)-f(x_\star)$, then ensuring $d\gamma^2\lambda^2\leq\epsilon$ guarantees $\norm{x_T-x_\star}^2\leq\epsilon$, and implies, $f(x_T)-f(x_\star)\leq\frac{L}{2}\epsilon$. The above is an upper bound, and we observe in our experiments by using $\lambda=\frac{\sqrt{\epsilon}}{d\sqrt{\gamma}}$, gives $f(x_t)-f(x_\star)\leq\epsilon$.

\subsubsection{Extreme sparsification}
In Figure \ref{fig:madelon}, we perform extreme sparsification to train a logistic regression model on the \texttt{madelon} LIBSVM dataset. We compare Top-$k$ with $k=1$, and hard-threshold with $\lambda=14881$ set via  $d\gamma^2\lambda^2=1.25 \times 10^{-4}$, so that they both communicate \emph{same data volume}. In Figure \ref{fig:madelon}b, we see that Hard-threshold sparsifier does not communicate any elements in many iterations. Despite this, hard-threshold has faster convergence than Top-$k$ in Figure \ref{fig:madelon} a. Figure \ref{fig:madelon} c demonstrates that this is because hard-threshold has a smaller total-error than Top-$k$. 

\begin{figure*}[!htbp]
\centering
\begin{subfigure}{0.32\textwidth}
		\centering
		\includegraphics[width=\textwidth]{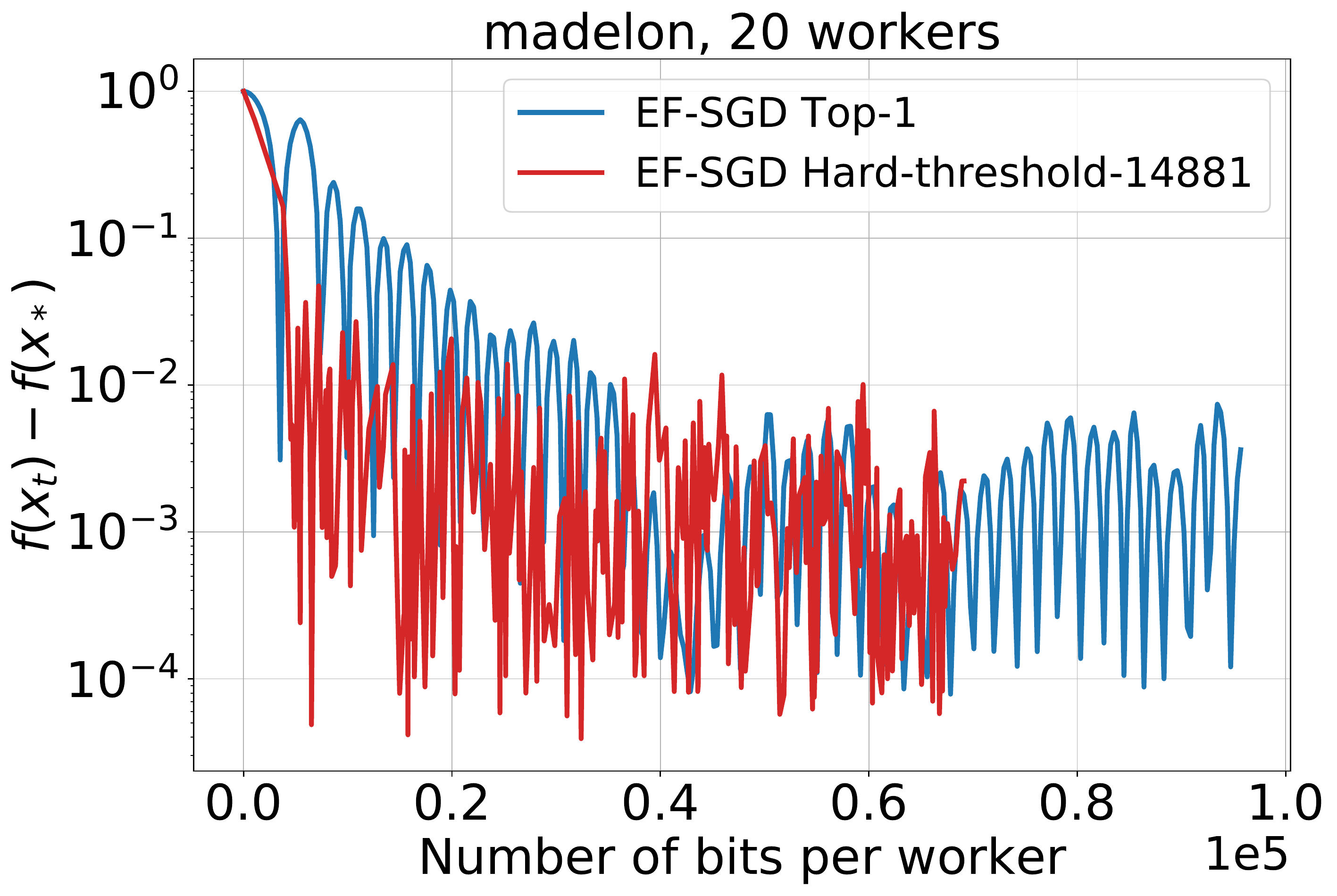}
				\caption{}
	\end{subfigure}
\begin{subfigure}{0.32\textwidth}
		\centering
		\includegraphics[width=\textwidth]{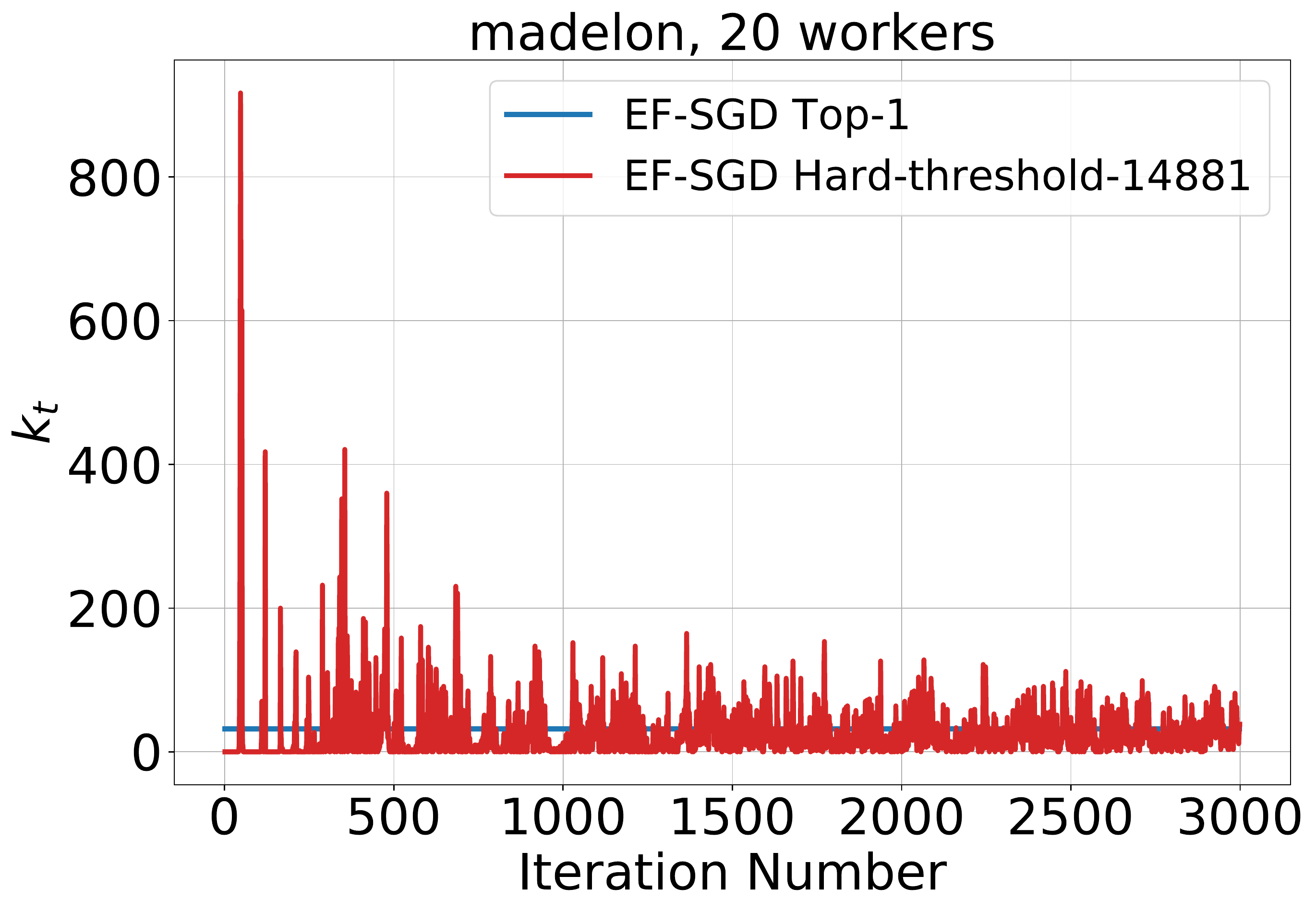}
		\caption{}
	\end{subfigure}
\begin{subfigure}{0.32\textwidth}
        \centering
		\includegraphics[width=\textwidth]{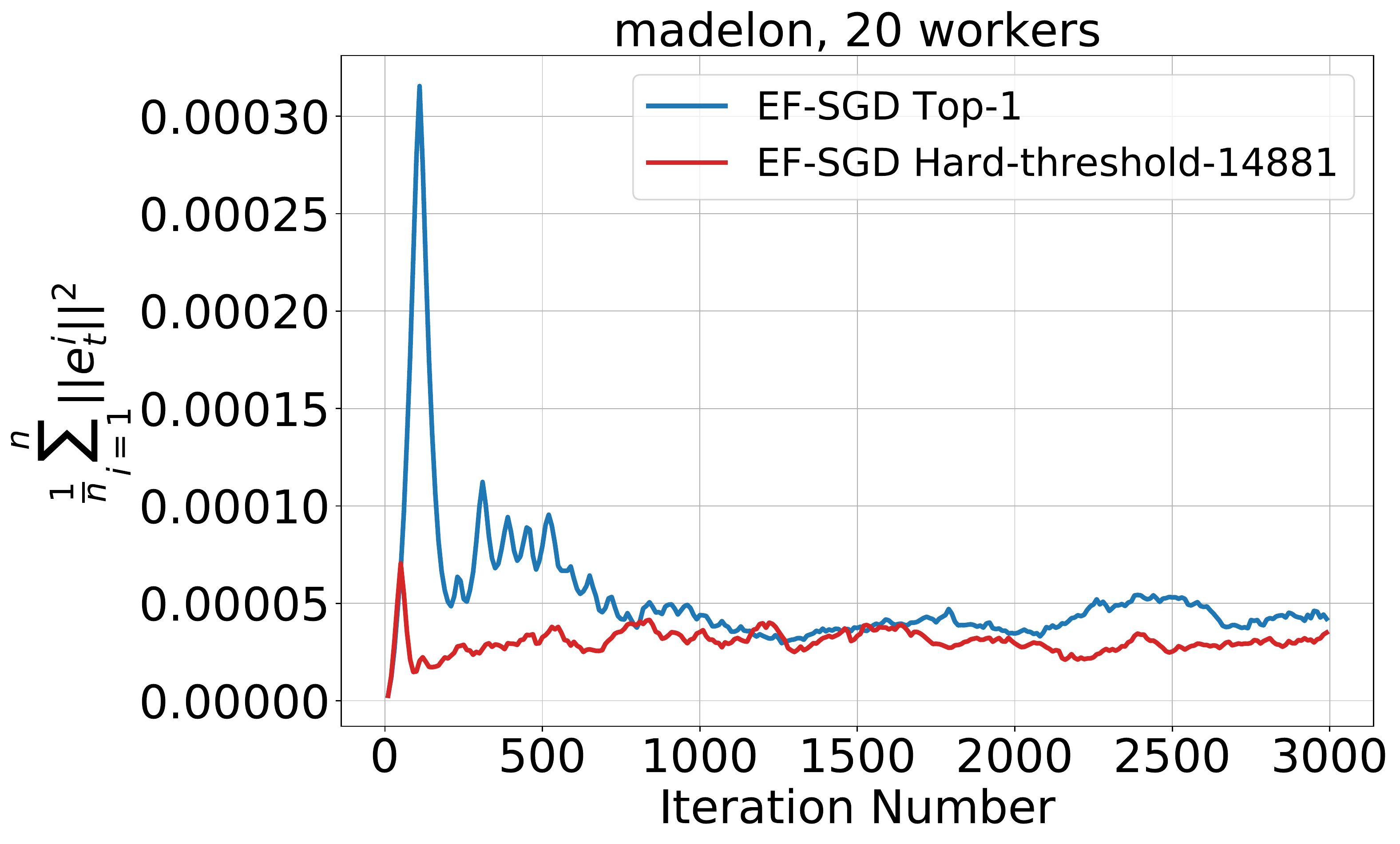}
		\caption{}
	\end{subfigure}
\caption{
\small{Convergence of Top-$k$ and Hard-threshold for a logistic regression model on \texttt{madelon} LIBSVM dataset with 20 workers: (a) Functional suboptimality vs. bits communicated; (b) parameters communicated vs. iterations; (c) error norm vs. iterations. Hard-threshold has a faster convergence than Top-$k$ even when it does not communicate any parameter in some iterations.}}\label{fig:madelon}
\end{figure*}

\subsubsection{Convergence to an arbitrary neighborhood of the optimum}
For the experiments in this section, the uncompressed baseline is distributed gradient descent (GD). Unlike SGD, 
GD has linear convergence to the exact optimum. However, Distributed EF-GD does not converge to the exact optimum due to compression noise. To remedy this, Gorbunov et al. \cite{gorbunov2020linearly} introduced a family of variance-reduced compression algorithms that have linear convergence to the exact optimum. We consider algorithm \texttt{EF-GDstar} from \cite{gorbunov2020linearly}~(known as \texttt{EC-GDstar} in \cite{gorbunov2020linearly}). 

We empirically show that \texttt{EF-GDstar} with hard-threshold compressor, can converge to an arbitrarily small neighborhood around the optimum, for an appropriate choice of hard-threshold. Figure \ref{fig:gd_star_madelon_20_workers} and Figure \ref{fig:gd_star_madelon_100_workers} demonstrate the convergence of \texttt{EF-GDstar} using Hard-threshold and Top-$k$ sparsifiers with 20 workers and 100 workers, respectively. We choose \myNum{i} $k=1$ for $20$ workers and $k=5$ for $100$ workers, respectively; \myNum{ii} $\lambda=2.98$, such that $d\gamma^2\lambda^2=5\times10^{-12}$. By using this $\lambda$, the compression error for hard-threshold is less than $5\times10^{-12}$ in Figures \ref{fig:gd_star_madelon_20_workers} c and \ref{fig:gd_star_madelon_100_workers} c. Moreover, hard-threshold converges to $f(x_T)-f(x_\star)\leq5\times10^{-12}$ in both Figures \ref{fig:gd_star_madelon_20_workers} b and \ref{fig:gd_star_madelon_100_workers} b. Additionally, hard-threshold sends $1.7\times$ and $8\times$ less data than Top-$k$ in Figure \ref{fig:gd_star_madelon_20_workers} a and Figure \ref{fig:gd_star_madelon_100_workers} a, respectively. Furthermore, Figure \ref{fig:gd_star_madelon_20_workers} is an extreme sparsification scenario where hard-threshold communicates $<1$ parameter per iteration per worker.

\begin{figure*}[!htbp]
\centering
\begin{subfigure}{0.32\textwidth}
		\centering
		\includegraphics[width=\textwidth]{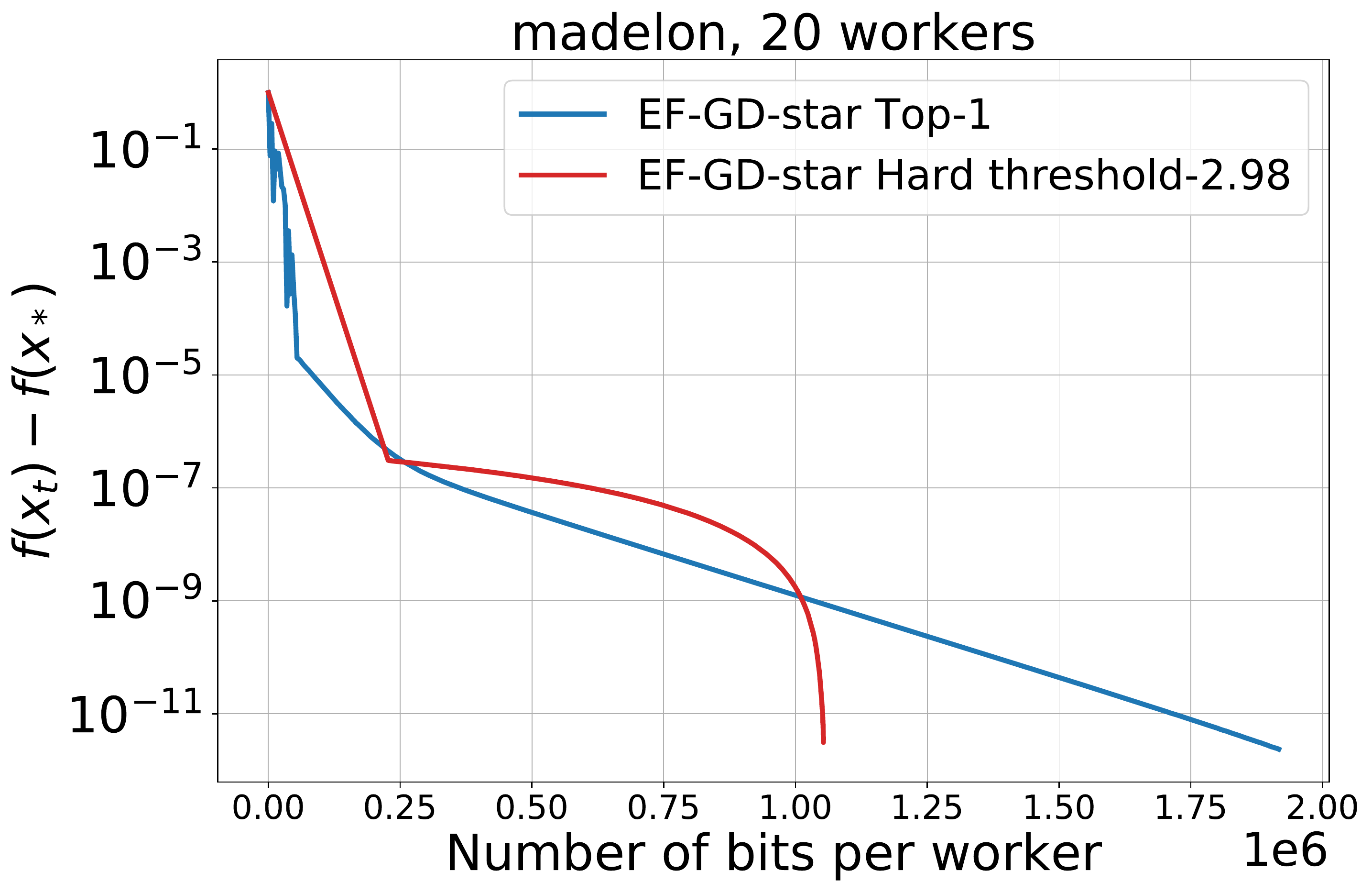}
				\caption{}
	\end{subfigure}
\begin{subfigure}{0.32\textwidth}
		\centering
		\includegraphics[width=\textwidth]{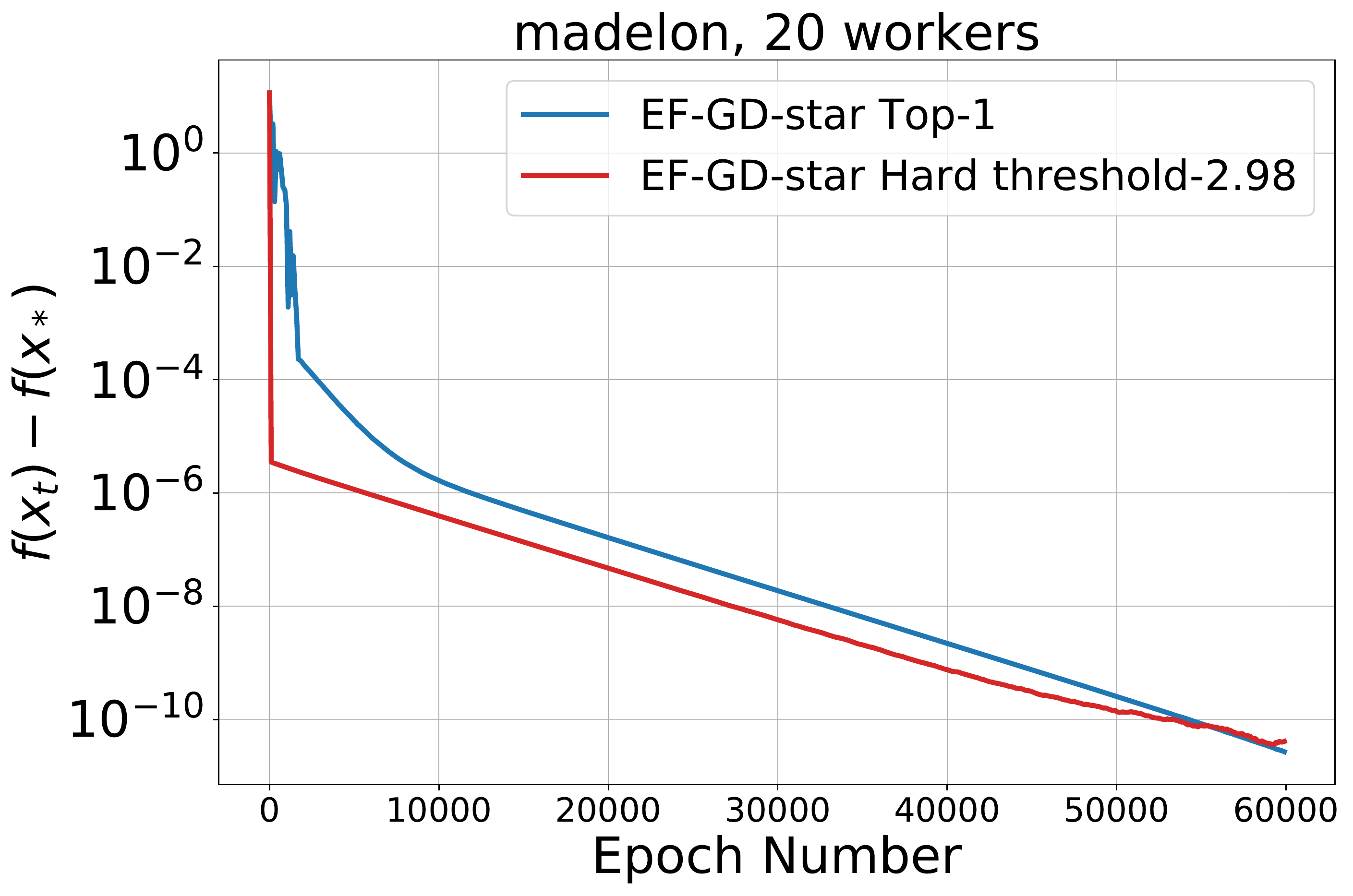}
		\caption{}
	\end{subfigure}
\begin{subfigure}{0.32\textwidth}
        \centering
		\includegraphics[width=\textwidth]{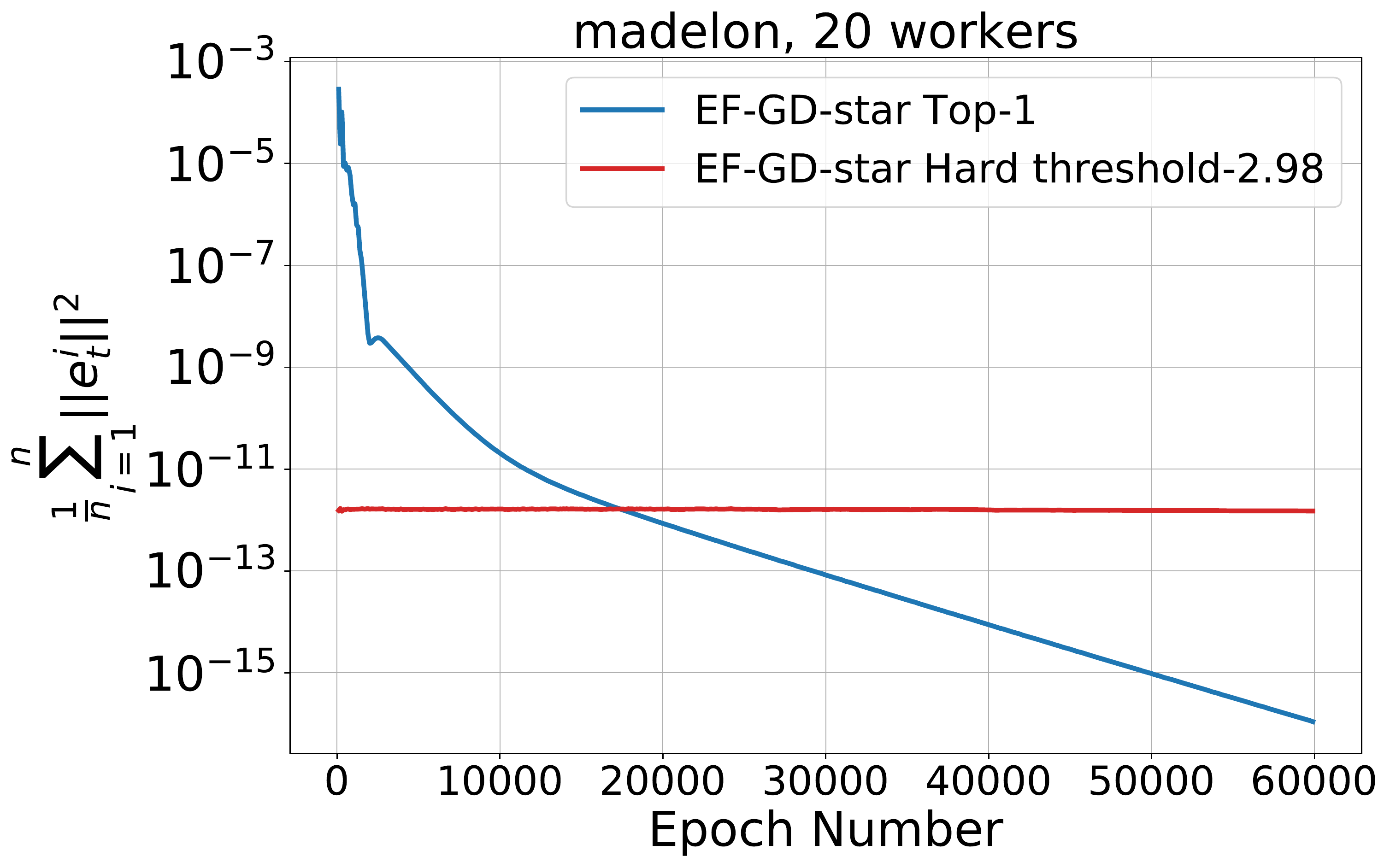}
		\caption{}
	\end{subfigure}
\caption{
\small{Convergence of \texttt{EF-GDstar} using Top-$k$ and Hard-threshold sparsifiers on a logistic regression model on \texttt{madelon} LIBSVM dataset with 20 workers: (a) Functional suboptimality vs. bits communicated; (b) functional suboptimality vs. epochs; (c) error-norm vs. epochs.}}\label{fig:gd_star_madelon_20_workers}
\end{figure*}

\begin{figure*}[!htbp]
\centering
\begin{subfigure}{0.32\textwidth}
		\centering
		\includegraphics[width=\textwidth]{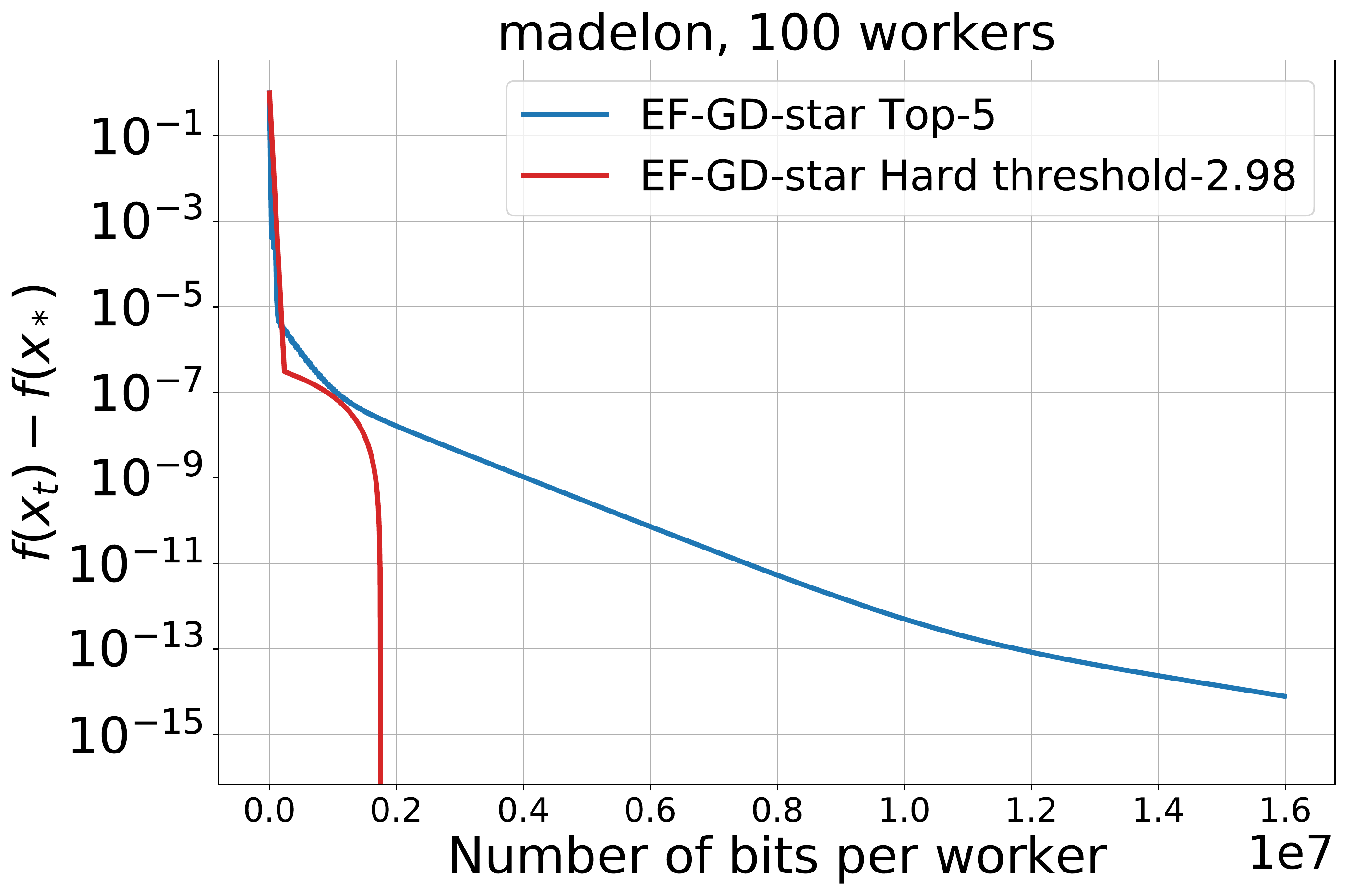}
				\caption{}
	\end{subfigure}
\begin{subfigure}{0.32\textwidth}
		\centering
		\includegraphics[width=\textwidth]{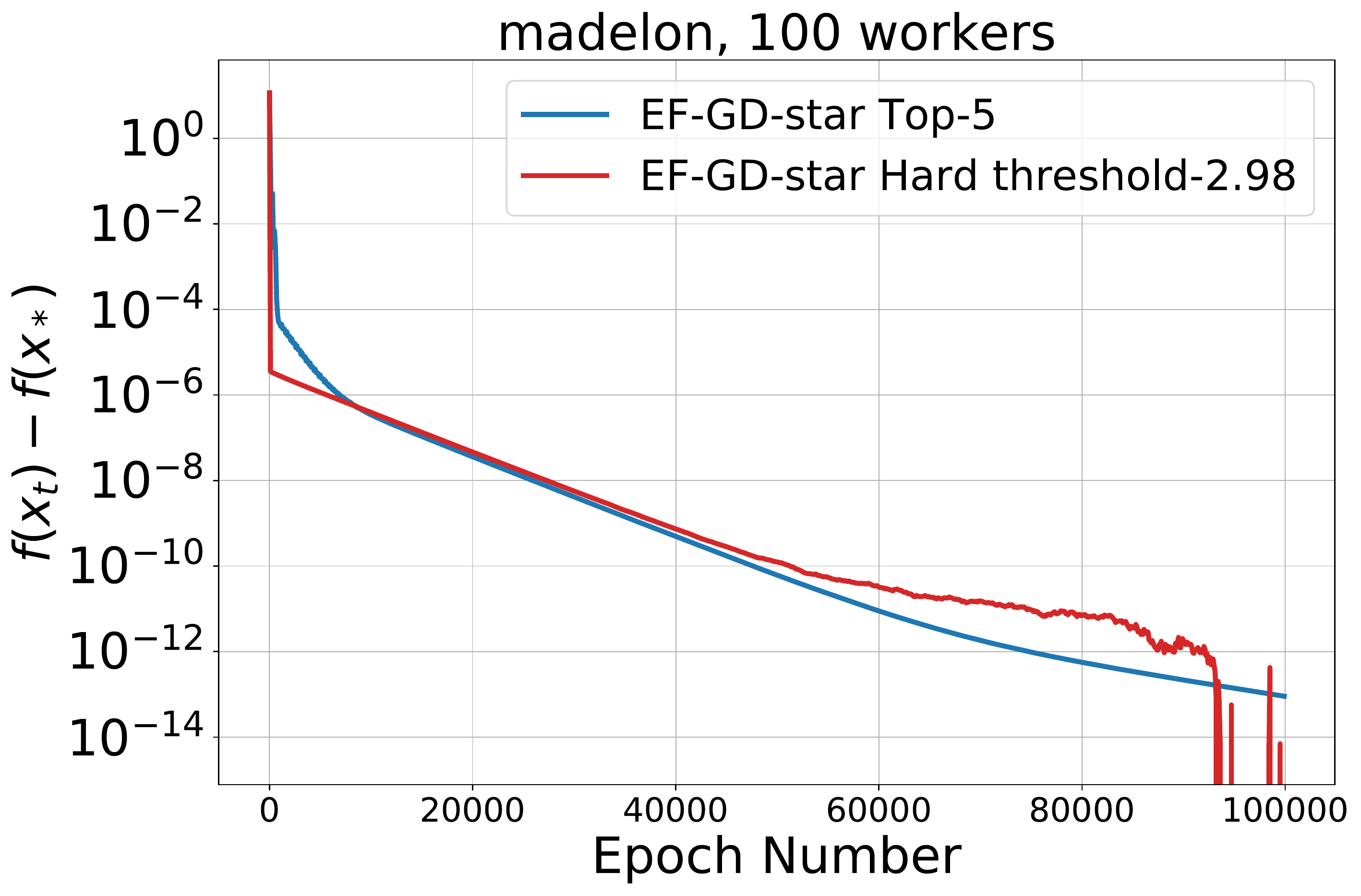}
		\caption{}
	\end{subfigure}
\begin{subfigure}{0.32\textwidth}
        \centering
		\includegraphics[width=\textwidth]{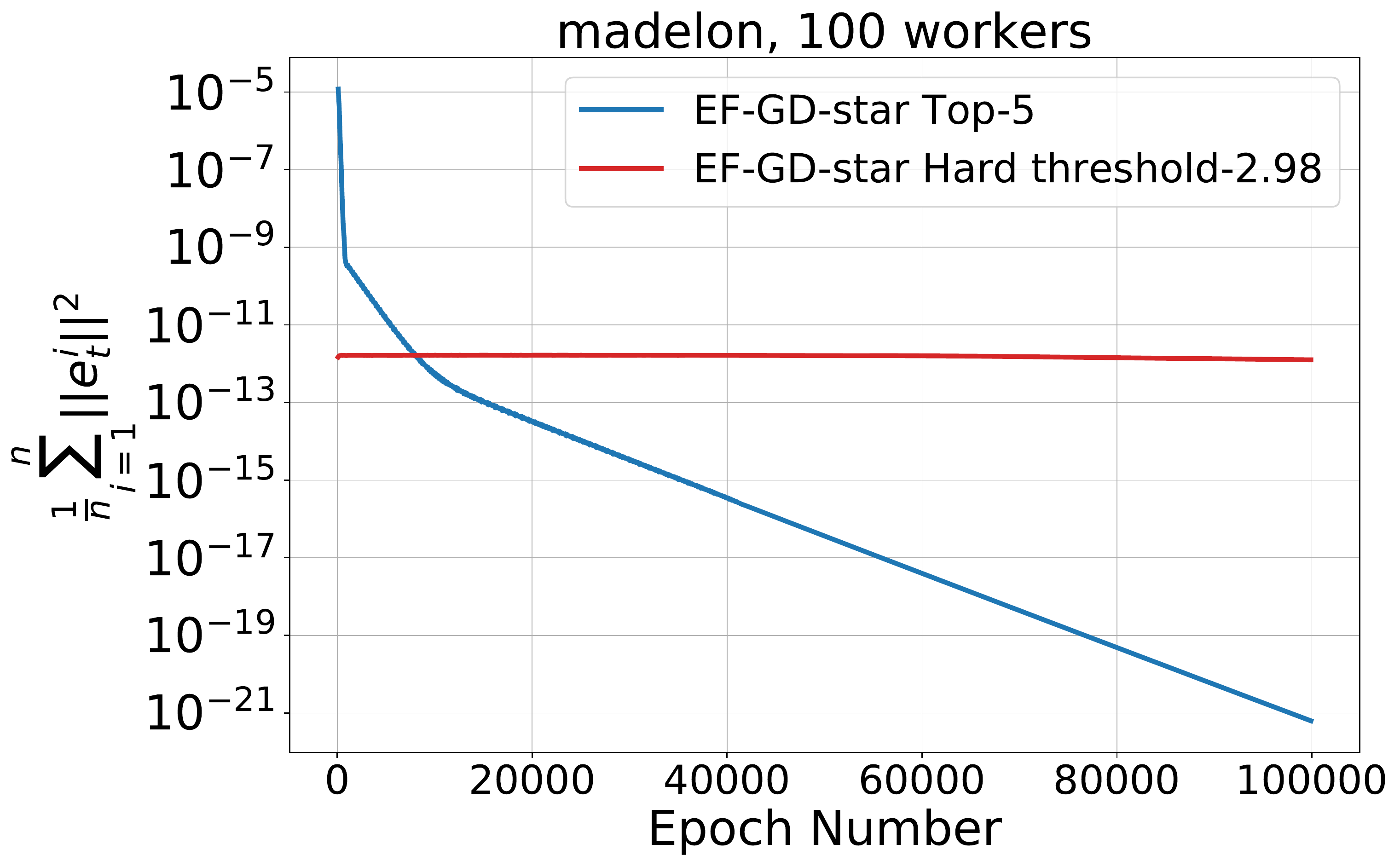}
		\caption{}
	\end{subfigure}
\caption{
\small{Convergence of \texttt{EF-GDstar} using Top-$k$ and Hard-threshold sparsifiers on a logistic regression model on \texttt{madelon} LIBSVM dataset with 100 workers: (a) Functional suboptimality vs. bits communicated; (b) functional suboptimality vs. epochs; (c) error norm vs. epochs.}}\label{fig:gd_star_madelon_100_workers}
\end{figure*}
Our results demonstrate that it is possible to use the hard-threshold compressor to converge to an arbitrarily small neighborhood around the optimum. We leave the convergence analyses, and devising practical variants for future research.

\subsection{Comparison against ACCORDION}\label{sec:accordion}
We compare against ACCORDION \cite{agarwal2020accordion} on CIFAR-10 and CIFAR-100 datasets. ACCORDION shifts between two user-defined $k$ values: $k_{\max}$ and $k_{\min}$, by using Top-$k_{\max}$ when the training is in a \emph{critical regime}, else using Top-$k_{\min}$. We compare against ACCORDION with hard-threshold $\lambda=\frac{1}{2\sqrt{k_{\min}}}$; see Table \ref{tab:accordion-setup} for the complete experiment details.

\begin{table*}[!htbp]
\centering
\caption{ACCORDION experiments}\label{tab:accordion-setup}
\vspace{3pt}
\scalebox{1.0}{
\begin{tabular}{cc}
\hline
Dataset & CIFAR-10 and CIFAR-100\\
Architectures & ResNet-18 \cite{resnet}, SENet18 \cite{senet}, GoogleNet \cite{GoogleNet}\\
Repository &PowerSGD \cite{PowerSGD} \\
&\footnotesize{See \url{https://github.com/epfml/powersgd}} \\
License &MIT \\
\hline
Number of workers & 8 \\
Global Batch-size & 256 $\times$ 8 \\
\hline 
Optimizer & SGD with Nesterov Momentum \\
Momentum & 0.9 \\
Post warmup LR & $0.1$ $\times$ $16$ \\
LR-decay & $/10$ at epoch 150 and 250 \\ 
LR-warmup & Linearly within 5 epochs, starting from $0.1$ \\
Number of Epochs &300 \\
Weight decay &$10^{-4}$ \\
\hline
Repetitions & 6, with different seeds\\
\hline
Accordion: $k_{\min}$ value & $0.1\%$ for both CIFAR-10 and CIFAR-100\\
Accordion: $k_{\max}$ value & $1\%$ for CIFAR-10 and $2\%$ for CIFAR-100\\
Hard-threshold: $\lambda$ values &  \\
(Calculated using $\lambda=\frac{1}{2\sqrt{k_{\min}}}$) & ResNet-18-CIFAR-10: $4.73\times10^{-3}$\\
& ResNet-18-CIFAR-100: $4.72\times10^{-3}$\\
& GoogleNet-CIFAR-10: $6.37\times10^{-3}$\\
& GoogleNet-CIFAR-100: $6.32\times10^{-3}$\\
& SENet18-CIFAR-10: $4.68\times10^{-3}$\\
& SENet18-CIFAR-100: $4.68\times10^{-3}$\\
\hline
\end{tabular}
}
\end{table*}

We report the result in Tables \ref{tab:accordion-cifar10} and \ref{tab:accordion-cifar100}. Each setting is repeated with $6$ different seeds and we report the average. For the CIFAR-10 dataset, we observe that hard-threshold has $0.5\%-0.8\%$ higher test accuracy than ACCORDION and is approximately $3.5\times$ more communication efficient than ACCORDION. For the CIFAR-100 dataset, except the ResNet-18 model, we observe that hard-threshold obtains more than $0.8\%$ higher accuracy than ACCORDION with more than $1.26\times$ communication savings over ACCORDION.

\begin{table*}[]
\centering
\caption{Comparison against ACCORDION \cite{agarwal2020accordion} on CIFAR-10}
\label{tab:accordion-cifar10}
\begin{tabular}{llll}
\hline
\multirow{2}{*}{Network}   & \multirow{2}{*}{Method}              & \multirow{2}{*}{Accuracy ($\%$)} & \multirow{2}{*}{Average Density  ($\%$)}         \\  
                           &                                        &                           &      \\ \hline
\multirow{4}{*}{ResNet-18}  & Top-$1 \%$ ($k_{\max}/d$)             & $94.1$  & $1.00$ ($1 \times$)        \\ \cline{2-4} 
                           & Top-$0.1 \%$ ($k_{\min}/d$)           & $93.2$  & $0.10$ ($10 \times$)       \\ \cline{2-4} 
                           & ACCORDION                             & $93.5$  & $0.53$ ($1.9 \times$)       \\ \cline{2-4} 
                           & Hard-threshold ($\frac{1}{2\sqrt{k_{\min}}}$) &   $\mathbf{94.0}$   &  $\mathbf{0.13}$ $(\mathbf{7.7 \times})$ \\ \hline
\multirow{4}{*}{GoogleNet} & Top-$1 \%$ ($k_{\max}/d$)                      & $94.1$   & $1.00$ ($1 \times$)        \\ \cline{2-4} 
                           & Top-$0.1 \%$ ($k_{\min}/d$)                          & $92.9$  & $0.10$ ($10 \times$)       \\ \cline{2-4} 
                           & ACCORDION                              & $93.4$   & $0.47$ ($2.1 \times$)       \\ \cline{2-4} 
                           & Hard-threshold ($\frac{1}{2\sqrt{k_{\min}}}$) &  $\mathbf{94.2}$       &    $\mathbf{0.13}$ ($\mathbf{7.7 \times}$)      \\ \hline
\multirow{4}{*}{SENet18}     & Top-$1 \%$ ($k_{\max}/d$)                     & $ 94.0$ & $1.00$ ($1 \times$) \\ \cline{2-4} 
                          & Top-$0.1 \%$  ($k_{\min}/d$)                        & $ 92.5$  & $0.10$ ($10 \times$)       \\ \cline{2-4} 
                          & ACCORDION                               & $ 93.5$  & $0.47$ ($2.1 \times$)       \\ \cline{2-4} 
                          & Hard-threshold ($\frac{1}{2\sqrt{k_{\min}}}$) &  $\mathbf{94.2}$  & $\mathbf{0.14}$ ($\mathbf{7.1 \times}$)      \\ \hline
\end{tabular}
\end{table*}

\begin{table*}[]
\centering
\caption{Comparison against ACCORDION \cite{agarwal2020accordion} on CIFAR-100}
\label{tab:accordion-cifar100}
\begin{tabular}{llll}
\hline
\multirow{2}{*}{Network}   & \multirow{2}{*}{Method}              & \multirow{2}{*}{Accuracy ($\%$)} & \multirow{2}{*}{Average Density ($\%$)}           \\  
                           &                                        &                           &      \\ \hline
\multirow{4}{*}{ResNet-18}  & Top-$2\%$  ($k_{\max}/d$)                  & $71.8$   & $2.00$ ($1 \times$)        \\ \cline{2-4} 
                           & Top-$0.1 \%$  ($k_{\min}/d$)                 & $70.6$  & $0.10$ ($20 \times$)       \\ \cline{2-4} 
                           & ACCORDION                             & $\mathbf{71.6}$  & $0.57$ ($3.5 \times$)    \\ \cline{2-4} 
                           & Hard-threshold ($\frac{1}{2\sqrt{k_{\min}}}$) &  $71.4$  & $\mathbf{0.35}$ ($\mathbf{5.7 \times}$)  \\ \hline
\multirow{4}{*}{GoogleNet} & Top-$2 \%$  ($k_{\max}/d$)                           & $75.5$  & $2.00$ ($1 \times$)        \\ \cline{2-4} 
                           & Top-$0.1 \%$  ($k_{\min}/d$)                       & $73.1$ & $0.10$ ($20 \times$)       \\ \cline{2-4} 
                           & ACCORDION                              & $74.2$   & $0.48$ ($4.2 \times$)      \\ \cline{2-4} 
                           & Hard-threshold ($\frac{1}{2\sqrt{k_{\min}}}$) &   $\mathbf{75.0}$  &   $\mathbf{0.38}$ ($\mathbf{5.3 \times}$)         \\ \hline
\multirow{4}{*}{SENet18}     & Top-$2 \%$  ($k_{\max}/d$)  & $ 71.9$  & $2.00$ ($1 \times$)\\\cline{2-4} 
                          & Top-$0.1 \%$   ($k_{\min}/d$)   & $ 70.1$  & $0.10$ ($20 \times$)   \\ \cline{2-4} 
                          & ACCORDION                               & $ 71.0$  & $0.55$ ($3.6 \times$)       \\ \cline{2-4} 
                          & Hard-threshold ($\frac{1}{2\sqrt{k_{\min}}}$) &   $\mathbf{72.1}$    &   $\mathbf{0.36}$   ($\mathbf{5.6 \times}$)   \\ \hline
\end{tabular}
\end{table*}

\subsection{Entire-model sparsification}\label{apdx:fullmodel}
Sparsification can be performed in two ways: layer-wise or entire-model. In layer-wise sparsification, the sparsifier is invoked individually on each tensor resulting from each layer. In contrast, in entire-model sparsification, the sparsifier is applied to a single concatenated tensor resulting from all layers. Since hard-threshold is an element-wise sparsifier, layer-wise and entire-model sparsification result in the same sparsified vector. However, it is expected that layer-wise and entire model vary substantially for Top-$k$. Layer-wise Top-$k$ is used in all practical implementations \cite{DBLP:conf/interspeech/SeideFDLY14, lin2018deep, grace} because performing entire-model Top-$k$ is both compute and memory intensive. 

While we employ layer-wise Top-$k$ in our experiments, we present in Figure \ref{fig:accuracy vs data-volume global} the \emph{test metric vs. data volume} experiment for ResNet-18-CIFAR-10 benchmark (Figure \ref{fig:accuracy vs data-volume}a) using entire-model Top-$k$. We find that hard-threshold is more communication-efficient than entire-model Top-$k$ as well. Notably, at an average density ratio of $0.003\%$, hard-threshold has more than $4\%$ higher accuracy than entire-model Top-$k$.

\begin{figure*}
\centering
\begin{minipage}{0.7\textwidth}		
		\includegraphics[width=\textwidth]{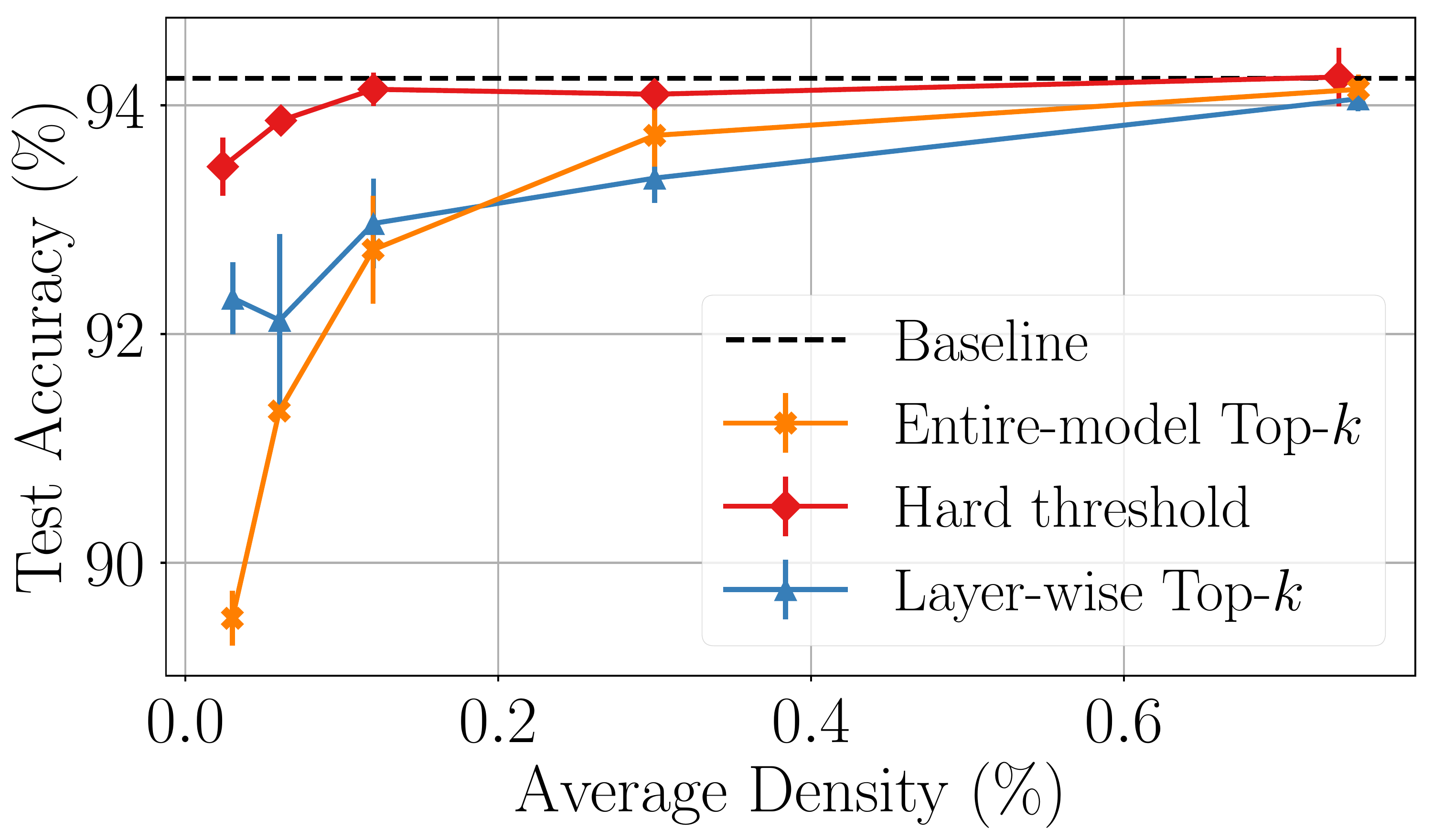}
		\caption{ResNet-18 on CIFAR-10}
\end{minipage}
\caption{\small{\textbf{Test metric vs. Data volume for entire-model compression}. The dashed black line in each plot denotes the no compression baseline. Each setting is repeated with three seeds, and we plot the average with standard deviation.~For description on parameters, see Tables \ref{tab:resnet18}, \ref{tab:lstm}, and \ref{tab:ncf}.}}\label{fig:accuracy vs data-volume global}
\vspace{7pt}
\end{figure*}

\subsection{Error-Feedback~(EF)}\label{apdx:memory}
In this section, we discuss various aspects of EF (or memory). Particularly, in \S\ref{sec:effect_mem} we investigate if hard-threshold is more communication-efficient than Top-$k$ without EF. Then, in Section \ref{grad_comp_vs_update}, we discuss and compare the different ways to perform EF in the literature. 

\subsubsection{Convergence without EF}\label{sec:effect_mem}
To understand how the sparsifiers perform without the EF, 
we conduct experiments without EF for ResNet-18 benchmark. We report this in Figure \ref{fig:no_mem}.  Similar to the with EF case, we find that hard-threshold has better convergence than Top-$k$. We note that with EF, hard-threshold achieved baseline performance at an extreme average density of $0.12\%$. However, without EF, hard-threshold fails to achieve baseline performance ($94.2\%$) even at a significantly higher average density of $5\%$. Hence, EF is a necessary tool to ensure faster convergence.

\begin{figure}
\centering
\begin{subfigure}{0.48\textwidth}
    \centering
		\includegraphics[width=\textwidth]{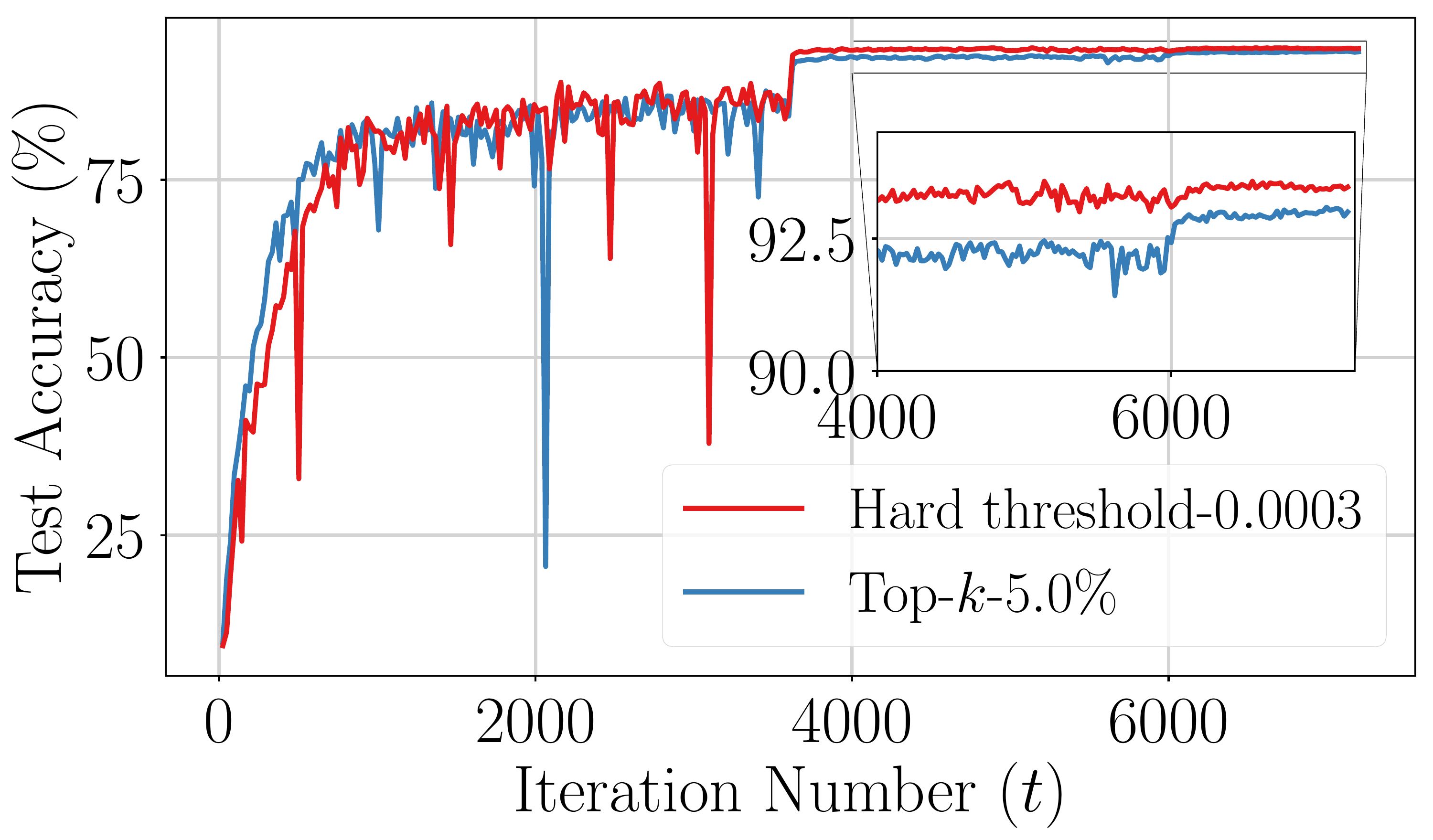}
			\caption{}
\end{subfigure}
\begin{subfigure}{0.48\textwidth}
		\centering
		\includegraphics[width=\textwidth]{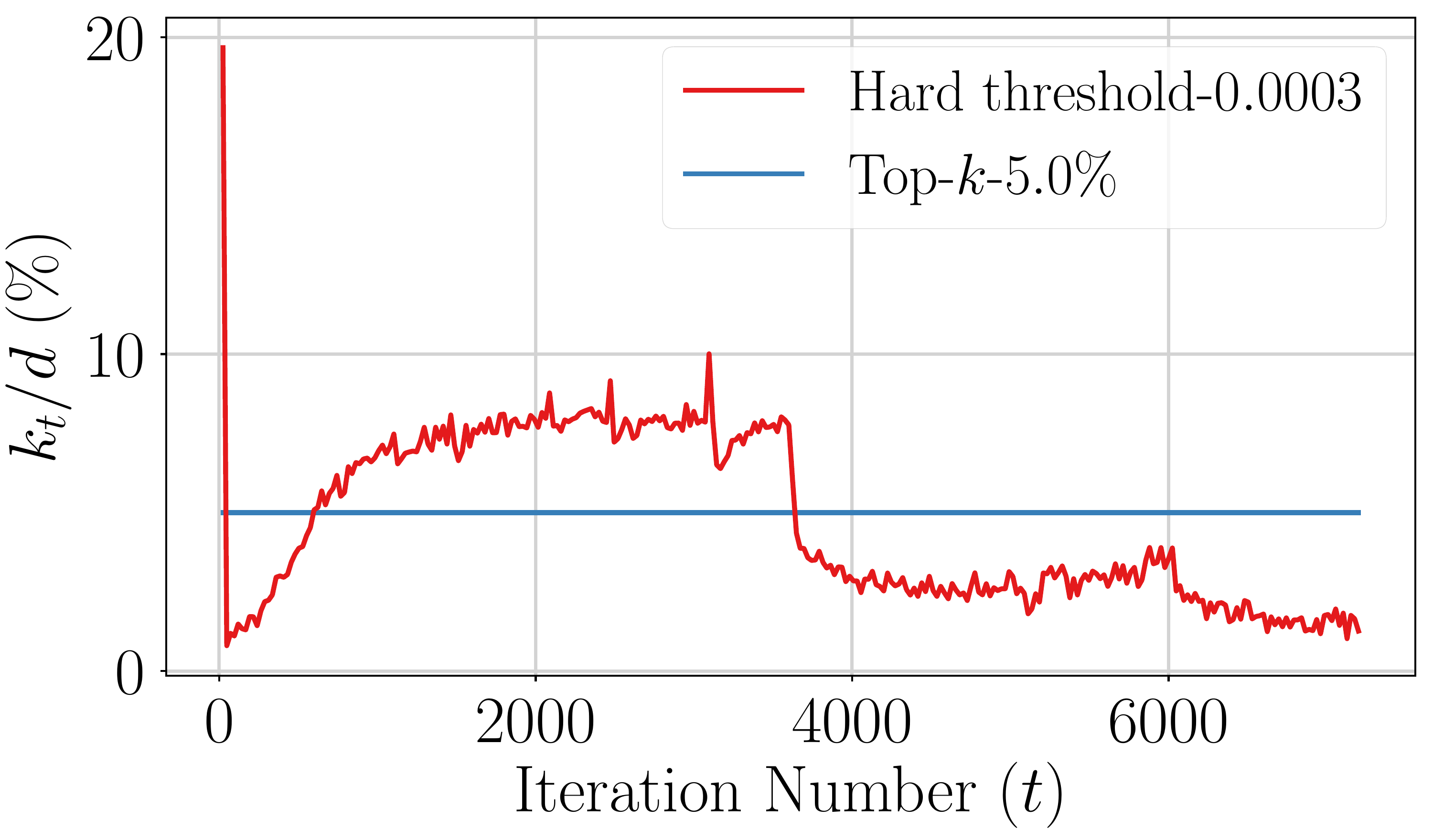}
		\caption{}
\end{subfigure}
\caption{Top-$k$ and hard-threshold without error compensation for ResNet-18 on CIFAR-10:~{(a)~Accuracy vs. Iterations, (b)~density, ($k_t/d$) vs. iterations.} Average density is $5\%$ for Top-$k$ and $4.7\%$ for hard-threshold.}\label{fig:no_mem}
\end{figure}

\begin{figure*}
\begin{minipage}{0.5\textwidth}		
\centering
\begin{algorithm}[H]
\small
    \SetKwFor{ForP}{for}{in parallel do}
	\SetAlgoLined
    \ForP{worker $w=1,..,W$}
    {
    \For{iteration $t = 1, 2, \cdots,$}
	{
	 Compute local stochastic gradient $g_w$\\
	 $\Delta_w \leftarrow g_w + e_w$\\
	$\cC(\Delta_w) \leftarrow {\tt COMPRESS}(\Delta_w)$\\
	$e_w \leftarrow \Delta_w - {\tt DECOMPRESS}(\Delta_w)$\\
	$\cC(\Delta) \leftarrow {\tt AGGREGATE}(\textstyle \cC(\Delta_1),\hdots,\cC(\Delta_W))$\\
	$\Delta^{'} \leftarrow {\tt DECOMPRESS}(\cC(\Delta))$\\
	$m \leftarrow \lambda m + \Delta^{'}$\\
	$x \leftarrow x - \gamma(\Delta^{'} + m)$\\
	}
    }
\caption{Distributed EF SGD with momentum by using gradient compression}\label{alg_1}
\end{algorithm}
\end{minipage}
\hskip .6cm
\begin{minipage}{0.5\textwidth}		
\centering
\begin{algorithm}[H]
 \small
    \SetKwFor{ForP}{for}{in parallel do}
	\SetAlgoLined
    \ForP{worker $w=1,..,W$}
    {
    \For{iteration $t = 1, 2, \cdots,$}
	{
	 Compute local stochastic gradient $g_w$\\
	 $m_w \leftarrow \lambda m_w + g_w$\\
	 $u_w \leftarrow m_w + g_w$\\
	 $\Delta_w \leftarrow u_w + e_w$\\
	 $\cC(\Delta_w) \leftarrow {\tt COMPRESS}(\Delta_w)$\\
	 $e_w \leftarrow \Delta_w- {\tt DECOMPRESS}(\Delta_w)$\\
	 $\cC(\Delta)\leftarrow {\tt AGGREGATE}(\cC(\Delta_1),\ldots, \cC(\Delta_W))$\\
	 $\Delta^{'}\leftarrow {\tt DECOMPRESS}(\cC(\Delta))$\\
	 $x \leftarrow x - \gamma(\Delta^{'})$\\
	}
    }
\caption{Distributed EF SGD with momentum by using update compression}\label{alg_2}
\end{algorithm}
\end{minipage}
\vspace{5pt}
\end{figure*}

\subsubsection{Different types of EF}\label{grad_comp_vs_update}
\begin{figure}
\centering
\begin{minipage}{0.5\linewidth}
    \centering
		\includegraphics[width=\linewidth]{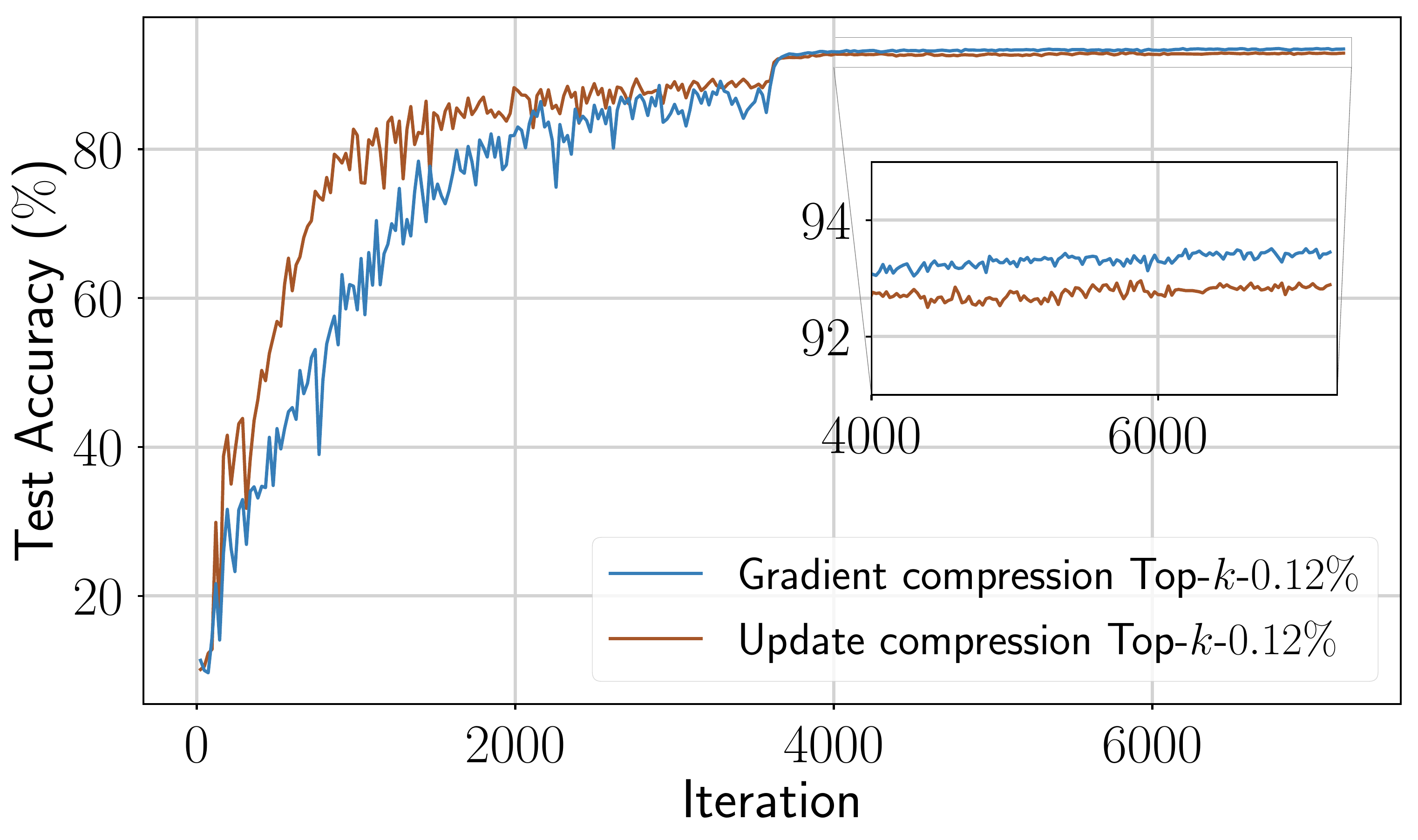}
\end{minipage}
\caption{{Test Accuracy for gradient compression vs. update compression for Top-$k$ on ResNet-18 on CIFAR-10. We experiment with three different seeds, and the plot represents the run with highest final accuracy for each setting. The test accuracy statistics ($\mu \pm \sigma$) are: Gradient compression ($92.96 \pm 0.39 \%$) and update compression ($90.78 \pm 2.03 \%$).}}\label{fig:update_compare}
\vspace{7pt}
\end{figure}

For optimizers other than vanilla SGD, one can compress and aggregate quantities other than stochastic gradients (such as momentum).
Consider an example for SGD with Nesterov momentum, where the compression and aggregation can be performed in the following two ways:
\begin{itemize}
    \item \textbf{Gradient compression.} This was proposed in \cite{PowerSGD} and is depicted in Algorithm \ref{alg_1}.
    In the case of SGD with Nesterov momentum, this update rule ensures that every worker maintains the same momentum state. However, the updates to momentum is sparse, as the momentum is calculated using sparsified gradients.
    
    \item \textbf{Update compression.}
    This was proposed in \cite{lin2018deep}, and is depicted in Algorithm \ref{alg_2}. 
    In the case of SGD with Nesterov momentum, every worker maintains a different momentum state calculated from their local stochastic gradients. Although updates to the momentum state is dense in this case, the momentum state is completely unaware of the compression and does not reflect the actual history of the updates. In order to circumvent this issue, Lin et. al. \cite{lin2018deep} had proposed momentum factor-masking to clear old local momentum states of a parameter once the parameter is updated. However, it is not easy to devise such modifications for optimizers which maintain multiple states derived from complicated calculations, such as RMSProp and ADAM.
\end{itemize}
\smartparagraph{Nomenclature for Algorithm \ref{alg_1} and \ref{alg_2}.} In Algorithm \ref{alg_1} and \ref{alg_2} we show the distributed training loop. We denote the learning rate by $\gamma$, momentum factor by $\lambda$, the model parameters by $x \in \R^d$, the momentum at worker $w$ by $m_w$, and the error at worker $w$ by $e_w$. At the beginning of the training, $m_w$ and $e_w$ are initialized to zero for all workers. By {\tt COMPRESS}, {\tt DECOMPRESS}, and {\tt AGGREGATE} we denote the compression, decompression, and aggregate function, respectively.

We also conduct experiments for Top-$k$ on ResNet-18 benchmark by using aforementioned update rules and find that gradient compression (Algorithm \ref{alg_1}) results in better performance (see Figure \ref{fig:update_compare}). In light of the above discussion and experimental evidence, we stick to gradient compression~(Algorithm \ref{alg_1}) for our main experiments.

\section{How to tune the hard-threshold?}\label{sec:tuning_hard_threshold}
We use our non-convex convergence results from \S\ref{sec:non_convex_convergence} to suggest a hard-threshold value which has better convergence than Top-$k$ with parameter $k$ for non-convex loss functions (including DNNs).
Substituting $\kappa^2=d\lambda^2$ for hard-threshold in Theorem \ref{thm:non-convex-wo-bg} we get
\begin{eqnarray}\label{eq:non-convex-hard-threshold}
 \textstyle \frac{1}{T}\sum_{t=0}^{T-1} \mbE \norm{\nabla f(x_t)}^2 \leq \frac{4(f(x_0)-f^{\star})}{\gamma T} +\frac{2\gamma L(M\zeta^2+\sigma^2)}{n} + 2\gamma^2 L^2d\lambda^2.
\end{eqnarray}
Similarly, substituting $\delta=\frac{k}{d}$ for Top-$k$ in Theorem \ref{thm:non-convex-wo-bg-rel} we get
\begin{eqnarray}\label{eq:non-convex-Top-k}
 \textstyle \frac{1}{T}\sum_{t=0}^{T-1} \mbE \norm{\nabla f(x_t)}^2 \leq \frac{8(f(x_0)-f^{\star})}{\gamma T} + \frac{4 \gamma L(M\zeta^2+\sigma^2)}{n} +\frac{8\gamma^2L^2d}{k}\left(\left(\frac{2d}{k}+M\right)\zeta^2+\sigma^2\right).
\end{eqnarray}
We ignore the first two terms unaffected by compression in \eqref{eq:non-convex-hard-threshold} and \eqref{eq:non-convex-Top-k}, and focus on the last term. To ensure that hard-threshold has better convergence than Top-$k$ we have
\begin{equation*}
\textstyle 2L^2d\lambda^2 \leq \frac{8L^2d}{k}\left(\left(\frac{2d}{k}+M\right)\zeta^2+\sigma^2\right),
\end{equation*}
that is,
\begin{equation}\label{eq:hard-threshold-val}
\textstyle \lambda \leq \frac{2}{\sqrt{k}}\sqrt{\left(\frac{2d}{k}+M\right)\zeta^2+\sigma^2}.
\end{equation}

To simplify further, we assume $\zeta^2=0$ (homogeneous distributed data), and $\sigma\sim\frac{1}{4}$. This leads us to the hard-threshold value
\begin{equation*}
    \lambda \sim \frac{1}{2\sqrt{k}}.
\end{equation*}

We find that $\lambda=\frac{1}{2\sqrt{k_{\min}}}$ has better performance (with similar total-data volume) than Top-$k_{\min}$ in Tables \ref{tab:accordion-cifar10} and \ref{tab:accordion-cifar100}, and therefore the derived formula works well. Finally, we note that more accurate hard-threshold values can be derived by having estimates of $\zeta^2$, $M$, and $\sigma^2$ in \eqref{eq:hard-threshold-val}.

\end{document}